\documentclass{article}

 \usepackage[main, final]{neurips_2026}

\usepackage[utf8]{inputenc} 
\usepackage[T1]{fontenc}    
\usepackage{hyperref}       
\usepackage{url}            
\usepackage{booktabs}       
\usepackage{amsfonts}       
\usepackage{nicefrac}       
\usepackage{microtype}      
\usepackage{xcolor}         

\usepackage{amsmath}
\usepackage{booktabs}
\usepackage{multirow}
\usepackage{array}
\usepackage{graphicx}
\usepackage{subcaption}
\usepackage{caption}
\usepackage[table]{xcolor}
\usepackage{wrapfig}
\usepackage{setspace}
\usepackage{amsthm}
\usepackage{placeins}

\usepackage{newfloat}
\usepackage{listings}
\DeclareCaptionStyle{ruled}{labelfont=normalfont,labelsep=colon,strut=off} 
\newtheorem{proposition}{Proposition}
\definecolor{AgentBlue}{RGB}{39,110,175}
\definecolor{AgentGreen}{RGB}{48,144,32}
\definecolor{AgentYellow}{RGB}{186,183,37}
\definecolor{AgentRed}{RGB}{183,60,58}

\newcommand{\blueagent}[1]{\textcolor{AgentBlue}{#1}}
\newcommand{\greenagent}[1]{\textcolor{AgentGreen}{#1}}
\newcommand{\yellowagent}[1]{\textcolor{AgentYellow}{#1}}
\newcommand{\agentred}[1]{\textcolor{AgentRed}{#1}}

\title{Cross-Modal Navigation with Multi-Agent Reinforcement Learning}

\author{%
  Shuo Liu, Xinzichen Li, Christopher Amato \\
  Khoury College of Computer Sciences\\
  Northeastern University\\
  Boston, MA, 02120 \\
  \texttt{\{liu.shuo2,li.xinzi,c.amato\}@northeastern.edu} \\
}

\begin{document}

\maketitle

\begin{abstract}

Robust embodied navigation relies on complementary sensory cues. However, high-quality and well-aligned multi-modal data is often difficult to obtain in practice. Training a monolithic model is also challenging as rich multi-modal inputs induce complex representations and substantially enlarge the policy space. Cross-modal collaboration among lightweight modality-specialized agents offers a scalable paradigm. It enables flexible deployment and parallel execution, while preserving the strength of each modality. In this paper, we propose \textbf{CRONA}, a Multi-Agent Reinforcement Learning (MARL) framework for \textbf{Cro}ss-Modal \textbf{Na}vigation. CRONA improves collaboration by leveraging control-relevant auxiliary beliefs and a centralized multi-modal critic with global state. Experiments on visual-acoustic navigation tasks show that multi-agent methods significantly improve performance and efficiency over single-agent baselines. We find that homogeneous collaboration with limited modalities is sufficient for short-range navigation under salient cues; heterogeneous collaboration among agents with complementary modalities is generally efficient and effective; and navigation in large, complex environments requires both richer multi-modal perception and increased model capacity.
\end{abstract}

\section{Introduction}

\begin{wrapfigure}{r}{0.5\textwidth}
    \centering
    \captionsetup{font=footnotesize}
    \vspace{-4mm}
    \includegraphics[width=0.43\textwidth]{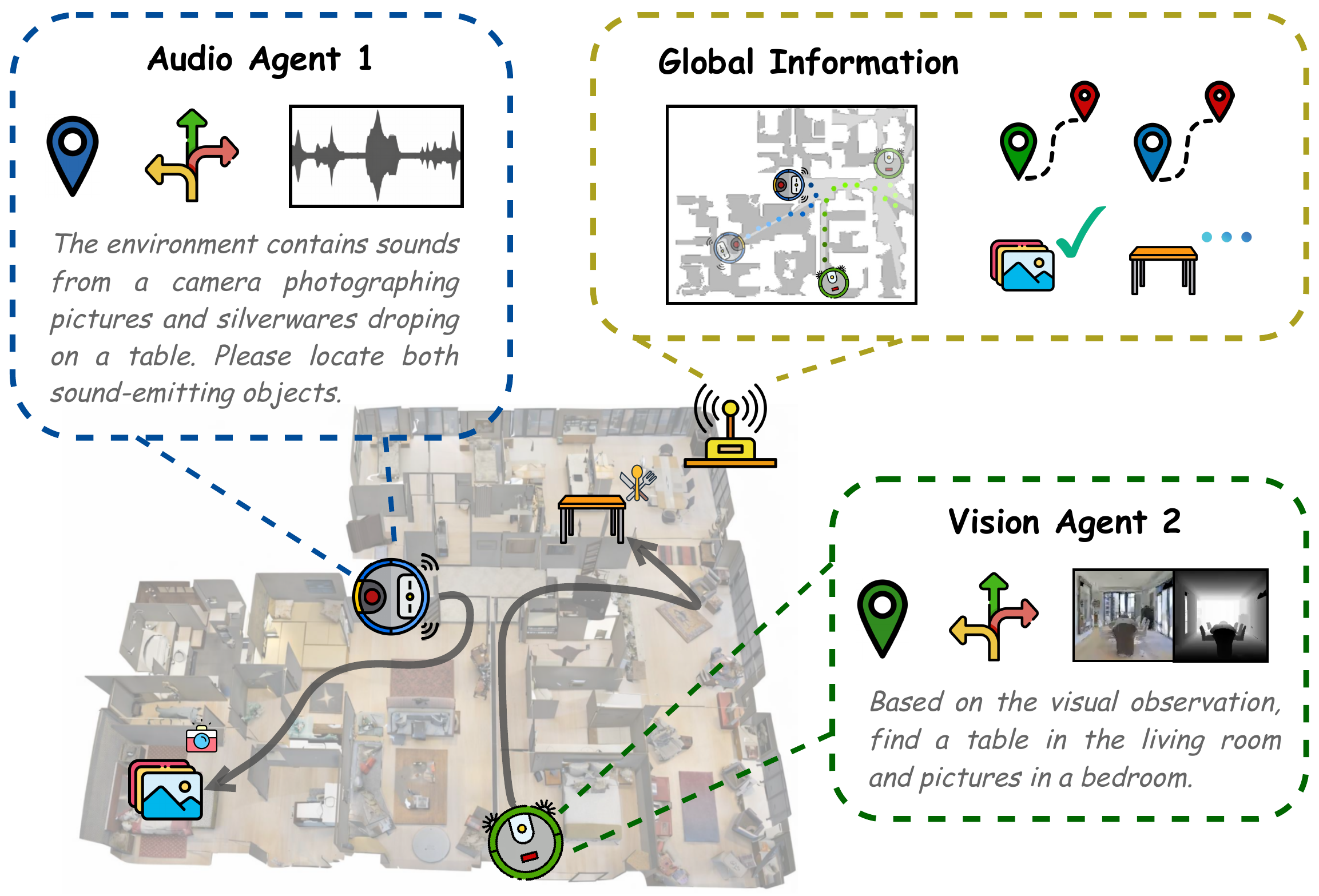}
    \caption{A collaborative navigation task in a \texttt{Ranch} scene from \texttt{Matterport3D}. An audio agent (\blueagent{blue}) collaborates with a vision agent (\greenagent{green}) to locate a table and pictures. Each agent receives only local observations during execution, while global information is captured by a global monitor (\yellowagent{yellow}) and used only during training. Gray curves denote agents' trajectories.}
    \label{fig:problem}
    \vspace{-3mm}
\end{wrapfigure}

In embodied navigation, agents perceive the environment through diverse sensory inputs, e.g., RGB-D images, audio, radar, LiDAR, and language instructions \cite{kolve2017ai2,habitat19iccv,szot2021habitat,anderson2018vision,duan2022survey}. 
These inputs provide rich geometric, semantic, and acoustic cues across modalities, enabling agents to locate target objects and navigate in complex environments, such as autonomous driving, robotic systems, and human-computer interaction \cite{chen2017multi,das2017visual,ku2018joint,qi2018frustum,shridhar2022cliport,brohan2022rt,driess2023palm}.

However, real-world observations are often noisy, incomplete, and asynchronous. 
Low-quality and misaligned training signals can make policy learning unstable and ineffective \cite{tsai2019multimodal, jia2021scaling}.
Although many methods align different modalities within the model, they differ substantially in dimensionality, noise levels, and temporal structure \cite{chen2020soundspaces,shridhar2020alfworld, shridhar2020alfred}. This mismatch often leads to imbalanced joint optimization, where dominant modalities drive most gradient updates while weaker or noisier modalities are underutilized \cite{wang2020makes, wu2022characterizing, huang2022modality, peng2022balanced}. 
Rich-modal models also tend to rely on large architectures to align diverse signals, making them hard to deploy and costly at test time \cite{kim2021vilt, yao2024minicpm, li2023blip}.

Many studies leverage multi-agent collaboration to improve navigation efficiency and robustness \cite{simmons2000coordination,parker2002alliance,burgard2005coordinated, gu2021attention}. However, most existing approaches focus on collaboration under limited sensory configurations \cite{huang2021decentralized,qin2021fully,azzam2023swarm,xiao2020learning,xiao2022asynchronous,xue2023multi,wang2024multi}. Even in more complex embodied settings, collaborative agents are typically homogeneous, with each receiving inputs from the same modalities \cite{wang2021collaborative,zhang2025advancing,wang2026conavbench}. Heterogeneous collaboration with rich sensory modalities are less explored \cite{hao2025conav,liu2025caml,hu2023semantic}. Specifically, it remains unclear which modalities improve collaborative capability, which team configurations support both effective and efficient navigation, and what cooperative behaviors emerge among agents. 

In this paper, we study fully-decentralized collaborative navigation without any inter-agent communication \cite{oliehoek2008optimal, decpomdp}. We construct a multi-modal collaborative navigation benchmark based on diverse \texttt{Matterport3D} scenes, as illustrated in Figure~\ref{fig:problem}. We propose \textbf{CRONA}, a cooperative Multi-Agent Reinforcement Learning (MARL) framework for \textbf{Cro}ss-Modal \textbf{Na}vigation. CRONA employs auxiliary belief predictors to extract control-relevant features from complex multi-modal observations and a centralized critic with state information to facilitate training. Our experiments demonstrate that collaborative navigation consistently outperforms single-agent navigation in both effectiveness and efficiency. Moreover, we identify five modality-dominance patterns across scenarios (i.e., no clear dominance, vision dominance, audio dominance, cross-modal, and multi-modal dominance). We find that homogeneous collaboration with few modalities is sufficient for short-range navigation; cross-modal collaboration among complementary modalities is generally efficient and effective when targets have clean, modality-specific cues; large and complex environments typically require both full-modal inputs and higher-capacity models.

\textbf{Our core contributions are summarized as follows:}
(i) we construct a collaborative navigation benchmark where multi-modal agents collaborate to navigate;
(ii) we propose CRONA, a MARL framework for cross-modal navigation;
(iii) we identify five dominance patterns in our experiments and explain when and why they emerge, respectively.

\section{Related Work}

\paragraph{Multi-Modal Navigation}
Embodied navigation has been studied under a wide range of input modalities.
Most work considers visual observations (RGB-D images), while specifying navigation goals or instructions in language \cite{anderson2018vision,krantz2020beyond,fried2018speaker,hao2020towards,majumdar2020improving,hong2021vlnbert,chen2021hamt}. 
Acoustic and 3D spatial signals can also provide semantic and geometric cues that complement visual observations that are degraded by occlusions or obstacles \cite{cheng2018mobile,chen2020soundspaces,chen2021semantic,paul2022avlen,chen2020learning, yao2023radar,yuksel2026gram}. 
While it has been shown that richer multi-modal context can improve performance in certain settings \cite{qi2021road,huang2023visual,yu2023l3mvn,liu2025caml}, it remains unclear how different modalities contribute under different conditions and how to align modalities with substantially different representations \cite{wang2020makes,wu2022characterizing,huang2022modality,li2023blip, peng2022balanced}.

\paragraph{Collaborative Navigation}
Recent studies have explored multi-agent collaboration for navigation. However, collaborative navigation remains challenging because agents need to coordinate in real time during execution. 
Early methods rely on centralized planning, where a central controller coordinates all agents \cite{bruce2006safe,van2009centralized,janssen2016cloud}. 
Such designs suffer from limited scalability and robustness \cite{velagapudi2010decentralized, luna2011efficient, iqbal2019actor}. 
Recent approaches therefore shift toward decentralized collaboration, where agents take action based on their local observations with limited or even without communication \cite{huang2021decentralized, qin2021fully, azzam2023swarm, wang2024multi, xue2023multi, wang2026conavbench}. However, most decentralized navigation studies still assume homogeneous agents or agents with similar sensory inputs \cite{zhang2025advancing,xiao2020learning,xiao2022asynchronous}. 
How agents with heterogeneous modalities collaborate effectively remains largely underexplored \cite{hao2025conav,liu2025caml, hu2023semantic}.

\paragraph{Cooperative MARL}
Cooperative MARL studies how multiple agents learn to coordinate under a shared objective \cite{zhang2021multi,marl-book,yuan2023survey}. A simple and scalable approach is independent learning, where agents are separately trained \cite{tan1993multi,peshkin2001learning}. 
But as all agents update their policies concurrently, each agent faces a non-stationary learning environment, which often leads to instability and convergence issues \cite{claus1998dynamics,tuyls2003selection,wunder2010classes}. Centralized training with decentralized execution (CTDE) mitigates this issue by exploiting centralized information during training \cite{amato2024introduction}. For example, a centralized critic can estimate joint values from joint histories and global states \cite{MADDPG,MAPPO,COMA,lyu2021contrasting,lyu2023centralized}. Since the critic is discarded at execution time, each agent remains execute in a decentralized manner. CRONA follows this paradigm and incorporates task progress into a multi-modal centralized critic for joint value estimation.

\section{Background}

\subsection{Problem Formulation}

In cooperative navigation (Figure~\ref{fig:problem}), agents need to infer task assignments and learn cooperative policies under partial observations. This setting follows the standard cooperative MARL formulation and can be modeled as a Decentralized Partially Observable Markov Decision Process (Dec-POMDP) \cite{decpomdp}, denoted by
$\langle \mathcal{I}, \mathcal{S}, \{\mathcal{O}_i\}, \{\mathcal{A}_i\}, R, T, \gamma, H \rangle$.

\begin{itemize}
    \item $\mathcal{I}$ is a set of $n$ decentralized agents, where each agent $i$ is controlled by an individual policy $\pi_i$. Each agent is equipped with specialized sensors to perceive the environment.
    \item $\mathcal{S}$ denotes the global state space. At each time step $t$, the global state $s_t \in \mathcal{S}$ includes the scene layout, all agent poses, target object locations and categories, sound-source states, and task-completion status. This state is not directly observed by decentralized agents.
    \item Each agent $i$ receives a local observation $o_{i,t} \in \mathcal{O}_i$. The observation contains the agent pose $o_{i,t}^{\mathrm{pose}}=(x_{i,t}, y_{i,t}, \vartheta_{i,t}, t)$, where $(x_{i,t}, y_{i,t})$ is the agent position and $\vartheta_{i,t}$ is its orientation. It also includes a natural-language description of the navigation target, denoted by $o_{i,t}^{\mathrm{goal}}$. Depending on its sensor configuration, an agent may also receive visual input $o_{i,t}^{\mathrm{vision}} = (o_{i,t}^{\mathrm{rgb}}, o_{i,t}^{\mathrm{depth}}) \in \mathbb{R}^{H_v \times W_v \times 4}$, binaural audio input $o_{i,t}^{\mathrm{audio}} \in \mathbb{R}^{2 \times L}$, where $H_v$ and $W_v$ denote the image height and width, and $L$ denotes the length of the binaural audio segment. $\mathcal{O}_i$ is the local observation space of agent $i$, and $\mathcal{O}=\times_i \mathcal{O}_i$ is the joint observation space.

    \item Agents share a joint reward function $R:\mathcal{S}\times\mathcal{A}\rightarrow\mathbb{R}$, which depends on the global state and their joint action. The reward incentives agents to approach targets and stop in their vicinity. 
    
    \item The environment evolves according to a transition function $T:\mathcal{S}\times\mathcal{A}\rightarrow\Delta(\mathcal{S})$. Given the current state $s_t$ and joint action $\mathbf{a}_t$, the next state is sampled as $s_{t+1}\sim T(\cdot \mid s_t,\mathbf{a}_t)$. 

    \item $\gamma$ is the discount factor and $H$ is the episode horizon.

\end{itemize}

Since the full state is not directly observable, each agent maintains a local observation-action history $h_{i,t}=\{o_{i,0}, a_{i,0}, \cdots, o_{i,t}\}$ to infer information about $s_t$. The history of agents forms a joint history $\mathbf{h}_t = \{h_{1,t}, \cdots, h_{n,t}\}$, and agents' policies forms a joint policy $\boldsymbol{\pi} =\{\pi_1, \cdots, \pi_n\}$. The objective is to find an optimal joint policy, $\boldsymbol{\pi}^* =\{\pi^*_1, \cdots, \pi^*_n\}$, that maximizes the expected cumulative reward over the horizon $H$, $\boldsymbol{\pi}^* =\arg\max_{\boldsymbol{\pi}} \mathbb{E}_{\boldsymbol{\pi}}\left[
\sum_{t=0}^{H-1} \gamma^t r_t\right]$.

\subsection{Acoustic Representation} \label{subsec:audio_intro}

Audio signals provide semantic information for source recognition and spatial cues for source localization. However, raw audio waves are high-dimensional and contain complex temporal dependencies. These make them difficult to model directly. Spectrograms represent audio as structured time-frequency features, making local acoustic patterns more explicit and easier to learn.

Given a binaural waveform $o_{i,t}^{\mathrm{audio}} \in \mathbb{R}^{2 \times L}$, its magnitude spectrogram
$o_{i,t}^{\mathrm{spec}} \in \mathbb{R}^{2 \times K \times F}$ can be computed via short-time Fourier transform (STFT),
\begin{equation}
o_{i,t}^{\mathrm{spec}}(\kappa,\omega,\tau)
=
\left|
\sum_{\ell=0}^{L-1}
o_{i,t}^{\mathrm{audio}}(\kappa,\ell)\,
w(\ell-\tau \delta)\,
e^{-j2\pi \omega \ell / N_{\mathrm{fft}}}
\right|.
\end{equation}
Here, $\kappa \in \{1,2\}$ denotes the left and right audio channels. For each time frame $\tau$, the window $w(\cdot)$ extracts a short segment of the waveform around sample position $\tau\delta$. The Fourier basis then decomposes this segment into frequency components indexed by $\omega$. $\delta$ denotes the hop size between adjacent time frames, and $N_{\mathrm{fft}}$ denotes the FFT size. The resulting spectrogram contains $K$ frequency bins over $F$ time frames.

Sounds emitted by different objects produce distinct patterns in the spectrogram. Targets with salient acoustic cues, such as strong energy and clean, stable patterns, are typically easier to localize, whereas distant or occluded sounds tend to be weak, unclear, and difficult to identify.

\section{Method}

Figure~\ref{fig:framework} gives an overview of CRONA. Each agent processes its sensory observations with the corresponding encoder. Audio-based agents use an auxiliary belief predictor to estimate control-relevant beliefs (target category and location). Each agent combines its observations, beliefs, and previous actions into a local history, where multi-head attention layers capture important features and temporal dependencies. A centralized critic estimates the joint value from the joint history, beliefs, and global state during training, which is used to update decentralized agent policies.

\subsection{Auxiliary Belief Predictor} \label{subsec:aux_predict}

Audio observations are often noisy and stochastic (Section~\ref{subsec:audio_intro}), making it difficult to learn effective policies directly from raw inputs. However, control-relevant beliefs can be inferred from these signals to facilitate training. CRONA uses target location and target category as auxiliary beliefs.

For an agent $i$ with audio sensor, given its spectrogram observation $o_{i,t}^{\mathrm{spec}}$, a convolutional encoder extracts acoustic features $z_{i,t}^{\mathrm{audio}}$. 
A location head predicts an instantaneous sound-source goal $\hat{b}_{i,t}^{\mathrm{goal}}\in \mathbb{R}^2$ in global coordinate based on $z_{i,t}^{\mathrm{audio}}$. Given the current pose $o_{i,t}^{\mathrm{pose}}=(x_{i,t}, y_{i,t}, \vartheta_{i,t}, t)$, the predicted relative location $\hat{b}_{i,t}^{\mathrm{loc}}$ can be calculated by,
\begin{equation}
\hat{b}_{i,t}^{\mathrm{loc}}
=
\mathrm{T}(\vartheta_{i,t})
\left(
\hat{b}_{i,t}^{\mathrm{goal}}
-
\begin{bmatrix}
x_{i,t}\\
y_{i,t}
\end{bmatrix}
\right),
\end{equation}
where {\footnotesize$\mathrm{T}(\vartheta)=
\begin{bmatrix}
\cos\vartheta & \sin\vartheta\\
-\sin\vartheta & \cos\vartheta
\end{bmatrix}$}
is a 2D rotation matrix from the global frame to the agent's frame. 
In addition, a category head with fully connected layers also predicts a belief $\,\hat{b}_{i,t}^{\mathrm{cat}} \in \mathbb{R}^\mathcal{C}$ over all categories $c\in\mathcal{C}$ based on $z_{i,t}^{\mathrm{audio}}$. To reduce prediction variance, we smooth the auxiliary beliefs with an exponential moving average using coefficient $\alpha \in [0,1]$,
\begin{equation}
b_{i,t}^{\mathrm{loc}}=\alpha\,\hat{b}_{i,t}^{\mathrm{loc}}+(1-\alpha)\,b_{i,t-1}^{\mathrm{loc}},
\quad
b_{i,t}^{\mathrm{cat}}=\alpha\,\hat{b}_{i,t}^{\mathrm{cat}}+(1-\alpha)\,b_{i,t-1}^{\mathrm{cat}}.
\end{equation}

The location and category belief jointly form an auxiliary belief $b_{i,t} = (b_{i,t}^{\text{loc}}, b_{i,t}^{\text{cat}})$ for agent $i$, and since they are inferred from local histories, they remain consistent with the information available to decentralized policies. During training, the goal point $b_{i,t}^{\mathrm{goal},*}$ of the closest target to agent $i$ and the multi-hot category label $y_{i,t}^{\mathrm{cat},*}$ over all targets are provided. The belief predictor is optimized as,
\begin{equation}
\mathcal{L}_{\mathrm{belief}}
=
\left\|
\hat{b}_{i,t}^{\mathrm{goal}}
-
b_{i,t}^{\mathrm{goal},*}
\right\|_2^2
-
\sum_{c \in \mathcal{C}}
\left[
y_{i,t}^{\mathrm{cat},*}(c)
\log \hat{b}_{i,t}^{\mathrm{cat}}(c)
+
\left(1-y_{i,t}^{\mathrm{cat},*}(c)\right)
\log
\left(1-\hat{b}_{i,t}^{\mathrm{cat}}(c)\right)
\right].
\end{equation}

\subsection{Attention-Based History Encoder}

In collaborative navigation, each agent selects actions based on its history. However, maintaining all raw images and audio over time is computationally expensive and difficult to optimize. We use a short-term history cache and apply multi-head attention to extract spatial and temporal features.

We use convolutional encoders to capture the local patterns of RGB-D images and spectrograms, i.e., $z_{i,t}^{\mathrm{rgb}}$, $z_{i,t}^{\mathrm{depth}}$, and $z_{i,t}^{\mathrm{audio}}$, respectively. The visual inputs $o^\mathrm{rgb}_{i,t}$ and $o^\mathrm{depth}_{i,t}$ often have higher dimensionalities and exhibit richer spatial structures, whereas $o^\mathrm{spec}_{i,t}$ are computed over short temporal windows, so we use deeper convolutional neural networks as visual encoders (i.e., ResNet-18 \cite{he2016deep}). The encoded features are concatenated with the agent pose and the goal instruction to form a latent observation embedding $z_{i,t}^{o}=z_{i,t}^{\mathrm{rgb}}\oplus z_{i,t}^{\mathrm{depth}} \oplus o_{i,t}^{\mathrm{pose}} \oplus o_{i,t}^{\mathrm{goal}}$ for vision-based agents, and $z_{i,t}^{o}=z_{i,t}^{\mathrm{audio}}\oplus o_{i,t}^{\mathrm{pose}}\oplus o_{i,t}^{\mathrm{goal}}$ for audio-based agents. Each agent stores the current observation embedding, the previous $k$ observation embeddings $\{z_{i,t-k}^{o}, \ldots, z_{i,t-1}^{o}\}$, and the previous $k$ actions $\{a_{i,t-k}, \ldots, a_{i,t-1}\}$ in a fixed-size memory cache.

The cached observation-action sequence is then processed by transformer blocks to produce a history representation $z_{i,t}^{h}$ over $h_{i,t}$. $z_{i,t}^{h}$ can capture geometric cues, sound information, and motion patterns, and provide a compact context for each agent to select its action, i.e., $a_{i,t}\sim \pi_i(\cdot \mid h_{i,t})$.

\begin{figure}
    \centering
    \captionsetup{font=footnotesize}
    \includegraphics[width=0.99\linewidth]{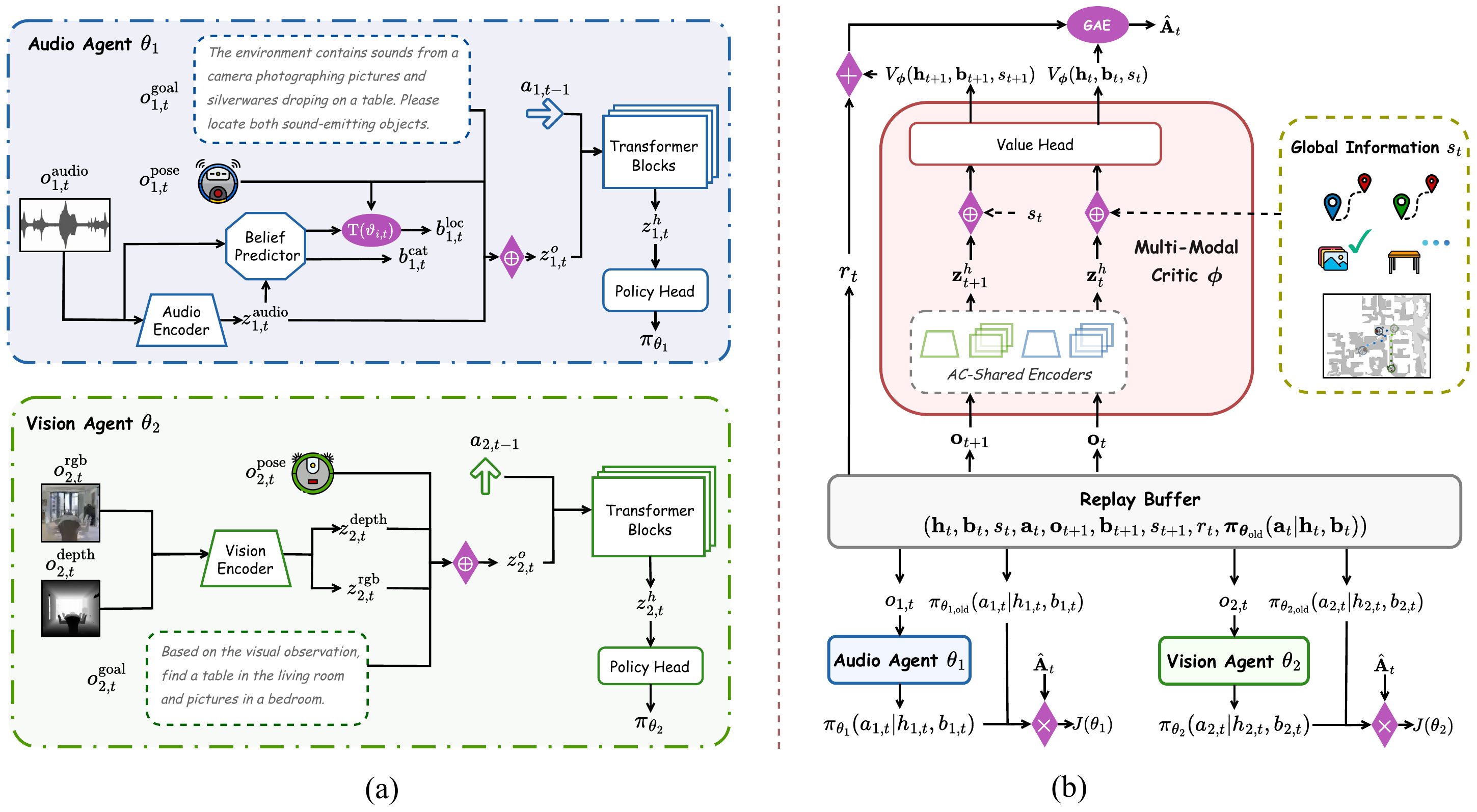}
    \caption{Illustration of \textbf{CRONA} framework. 2 decentralized agents, one with audio inputs (\blueagent{blue}) and another with vision inputs (\greenagent{green}), cooperate to navigate toward a table with silverware-dropping sounds and pictures with camera-shutter sounds. (a) Observation-action history embeddings and auxiliary belief predictors of agents. (b) A multi-modal critic (\agentred{red}) estimates the value with joint history, the auxiliary belief, and the global information, while each agent updates its individual policy.\vspace{-2mm}}
    \label{fig:framework}
\end{figure}

\subsection{Centralized Critic with Global Information}

CRONA employs a centralized critic for joint value estimation during training. Since the reward directly depends on the state, incorporating state information can improve value estimation without introducing bias \cite{lyu2023centralized}. The critic and agent policies also depend on auxiliary belief predictions (Section~\ref{subsec:aux_predict}), which are consistent with the information available in local observations. The centralized critic is not used during execution, where all agents take action under decentralized policies \cite{marl-book,amato2024introduction}.

During training, the centralized critic $\mathbf{V}_{\boldsymbol{\phi}}(\mathbf{z}^{\mathbf{h}}_t, \mathbf{b}_t, s_t)$ estimates the joint value using the joint history embedding $\mathbf{z}^{\mathbf{h}}_t$, the auxiliary beliefs of audio-based agents $\mathbf{b}_t = \{b_{1,t}, \dots, b_{n,t}\}$, and the global state $s_t$ (e.g., target locations, agent positions and orientations, and completion indicators for each target). As proved in Appendix~\ref{app:proof}, augmenting the critic with these history-induced beliefs and the global state does not introduce bias in value estimation. At each time step $t$, the joint history embedding is obtained by concatenating all agents' history embeddings,
$\mathbf{z}^{\mathbf{h}}_t=\bigoplus_{i=1}^{n} z_{i,t}^{h}$. 

To improve representation learning and accelerate training, CRONA shares the modality-specific encoders, auxiliary belief predictor, and history transformer between the decentralized actors and the centralized critic, while using separate heads for policy and value prediction. To stabilize training, we use clipped surrogate objectives for both policy and value updates. The advantage $\hat{\mathbf{A}}_t$ is computed using generalized advantage estimation (GAE),
\begin{equation}
\hat{\mathbf{A}}_t
=
\sum_{l=0}^{T-t-1}
(\gamma \lambda)^l
\left[
r_{t+l}
+
\gamma \mathbf{V}_{\boldsymbol{\phi}_{\mathrm{old}}}(\mathbf{h}_{t+l+1}, \mathbf{b}_{t+l+1}, s_{t+l+1})
-
\mathbf{V}_{\boldsymbol{\phi}_{\mathrm{old}}}(\mathbf{h}_{t+l}, \mathbf{b}_{t+l}, s_{t+l})
\right],
\end{equation}
and the corresponding return target is $\hat{\mathbf{R}}_t=\hat{\mathbf{A}}_t+\mathbf{V}_{\boldsymbol{\phi}_{\mathrm{old}}}(\mathbf{h}_t, \mathbf{b}_t, s_t)$. We train the value head of the centralized critic by minimizing a clipped value surrogate objective against the return target $\hat{\mathbf{R}}_t$,
\begin{equation}
\begin{gathered}
L(\boldsymbol{\phi})
=
\mathbb{E}_t
\left[
\max
\left(
\left(
\mathbf{V}_{\boldsymbol{\phi}}(\mathbf{h}_t, \mathbf{b}_t, s_t)
-
\hat{\mathbf{R}}_t
\right)^2,\,
\left(
\bar{\mathbf{V}}_{\boldsymbol{\phi}}(\mathbf{h}_t, \mathbf{b}_t, s_t)
-
\hat{\mathbf{R}}_t
\right)^2
\right)
\right],
\\
\bar{\mathbf{V}}_{\boldsymbol{\phi}}(\mathbf{h}_t, \mathbf{b}_t, s_t)
=
\operatorname{clip}
\left(
\mathbf{V}_{\boldsymbol{\phi}}(\mathbf{h}_t, \mathbf{b}_t, s_t),
\mathbf{V}_{\boldsymbol{\phi}_{\mathrm{old}}}(\mathbf{h}_t, \mathbf{b}_t, s_t)-\xi,\,
\mathbf{V}_{\boldsymbol{\phi}_{\mathrm{old}}}(\mathbf{h}_t, \mathbf{b}_t, s_t)+\xi
\right),
\end{gathered}
\end{equation}
where $\bar{\mathbf{V}}_{\boldsymbol{\phi}}$ denotes the clipped value prediction and $\xi$ is the value clipping range. Each agent's policy is conditioned only on its local history and auxiliary belief, $a_{i,t}\sim \pi_{\theta_i}(\cdot \mid h_{i,t}, b_{i,t})$, and is updated using the shared advantage estimate $\hat{\mathbf{A}}_t$. Specifically, each agent $i$ maximizes
\begin{equation}
J(\theta_i)
=
\mathbb{E}_t
\left[
\min
\left(
\rho_{i,t}\hat{\mathbf{A}}_t,\,
\operatorname{clip}(\rho_{i,t},1-\epsilon,1+\epsilon)\hat{\mathbf{A}}_t
\right)
+
\beta\,\mathcal{H}\!\left(
\pi_{\theta_i}(\cdot\mid z_{i,t}^{h}, b_{i,t})
\right)
\right],
\end{equation}
where
$\rho_{i,t}=
\frac{
\pi_{\theta_i}(a_{i,t}\mid z_{i,t}^{h}, b_{i,t})
}{
\pi_{\theta_{i,\mathrm{old}}}(a_{i,t}\mid z_{i,t}^{h}, b_{i,t})
}$
is the importance sampling ratio, $\epsilon$ is the policy clipping range, $\beta$ is the entropy regularization coefficient, and $\mathcal{H}(\cdot)$ is the policy entropy to encourage exploration.

The decentralized actors and the multi-modal critic share encoders and transformers. Gradients from both actor and critic objectives are backpropagated through the shared modules, which are optimized by a weighted sum of the policy gradient and averaged temporal difference loss with $\mu\in[0,1]$,
\begin{equation}
\mathcal{L}(\theta_i^{z_i}, \boldsymbol{\phi}^{z_i})
=
-\mu J(\theta_i)
+
\frac{1-\mu}{n}\,L(\boldsymbol{\phi}).
\end{equation}

\section{Experiments}

We evaluate CRONA in \texttt{Matterport3D} scenes \cite{Matterport3D}, where agent observations are simulated via \texttt{Habitat} and \texttt{libsora} \cite{habitat19iccv, szot2021habitat, mcfee2015librosa}. Dataset details, experimental settings, additional results, instruction and reward designs, and compute resources are provided in Appendix \ref{app:expset}, \ref{app:addres}, \ref{app:prompt}, \ref{app:reward-design}, and \ref{app:compute}.

\subsection{Setup}

We construct collaborative navigation datasets with two agents using five representative \texttt{Matterport3D} scenes that span diverse layouts and difficulties. 

\texttt{Studio} (\texttt{GdvgFV5R1Z5}) is a single-room scene with a picture target with a camera-shutter sound. 
\texttt{Corridor} (\texttt{ac26ZMwG7aT}) consists of a passage connecting two spatially separated areas, where agents are finding a sink that is dripping water. 
\texttt{Apartment} (\texttt{17DRP5sb8fy}) has one bedroom and two bathrooms, with a creaking bed and a counter with coin-dropping sound as targets. 
\texttt{Ranch} (\texttt{JeFG25nYj2p}) contains five bedrooms and two bathrooms, with a picture with a camera-shutter sound and a table with silverware-dropping sounds as targets.
\texttt{Maze} (\texttt{B6ByNegPMKs}) is the largest scene with the most complex layout, agents need to find a table with silverware dropping, a dragging chair, and a drawer with a pulling sound while navigating through the scene within the episode limit.

Each dataset entry corresponds to a task in an episode. At the beginning of each episode, agents' positions and orientations are randomly initialized. Agents move on the navigable mesh grids to find all targets. They must stop within a specified distance of a target to mark it as found. Each target sound is assigned to an eligible object with the matching semantic category; sounds from multiple targets are mixed and removed once the corresponding target is found. An episode ends when all targets are found or all agents stop simultaneously. We set the horizon to $H=$70, 150, 500, 1000, 1500 for five scenes. Bird's-eye-view visualizations and dataset statistics are provided in Appendix~\ref{app:expset}.

Since most objects in \texttt{Matterport3D} are large and visually distinctive, an agent can effortlessly localize them without requiring collaboration. However, real-world visual perception is often constrained (e.g., darkness, fog, or blind spots). To make the benchmark more challenging, we restrict vision to depth maps with a sensing range of $0$--$5\,\mathrm{m}$, a resolution of $16\times16$ pixels, and an HFoV of $10^\circ$. Details about agent configurations and model architectures are provided in Appendix~\ref{subapp:architecture}.

\subsection{Baselines}

We consider the \textbf{\textit{Single-Agent}} baseline, where a large monolithic model takes all available modalities as input \cite{chen2020soundspaces}. For a fair comparison, we use the same episode horizon, and the agent's initial position is randomly selected from existing initial positions in our collaborative navigation dataset. 

We further compare CRONA with three homogeneous collaboration baselines, where all agents receive the same input modalities. Several recent studies have explored Vision-Language-Action (VLA) models for collaborative navigation. Hao et al.~\cite{hao2025conav} propose the CoNav framework in which one agent has access to a bird's-eye view, while Wang et al.~\cite{wang2026conavbench} put forward VLA-based collaborative navigation, CoNavBench, with inter-agent communication. Both settings involve centralized information and differ substantially from ours in environments, agent observability, architectures, and language information. To enable an informative comparison under our task setting, we implement a fully decentralized VLA collaboration baseline as a representative in our scenes, denoted as \textbf{\textit{VLA-Collab}}.
Although audio-language-action (ALA) collaboration has been less explored in navigation, we nevertheless include \textbf{\textit{ALA-Collab}} as the audio counterpart to VLA-Collab. Both VLA-Collab and ALA-Collab use restricted modality inputs. So we include \textbf{\textit{AVLA-Collab}}, where all agents receive audio, vision, and language inputs, to represent homogeneous full-modality collaboration in our settings~\cite{zhang2025advancing,liu2025caml}.
For a fair comparison, all baselines use the same configurations and hyperparameters, and agents in collaborative baselines have the same number of parameters. 



\begin{figure}[t]
    \centering
    \captionsetup{font=footnotesize}
    \subcaptionsetup{font=scriptsize}

    \hspace{2mm}
    \begin{subfigure}[b]{0.3\textwidth}
        \centering
        \includegraphics[width=\textwidth]{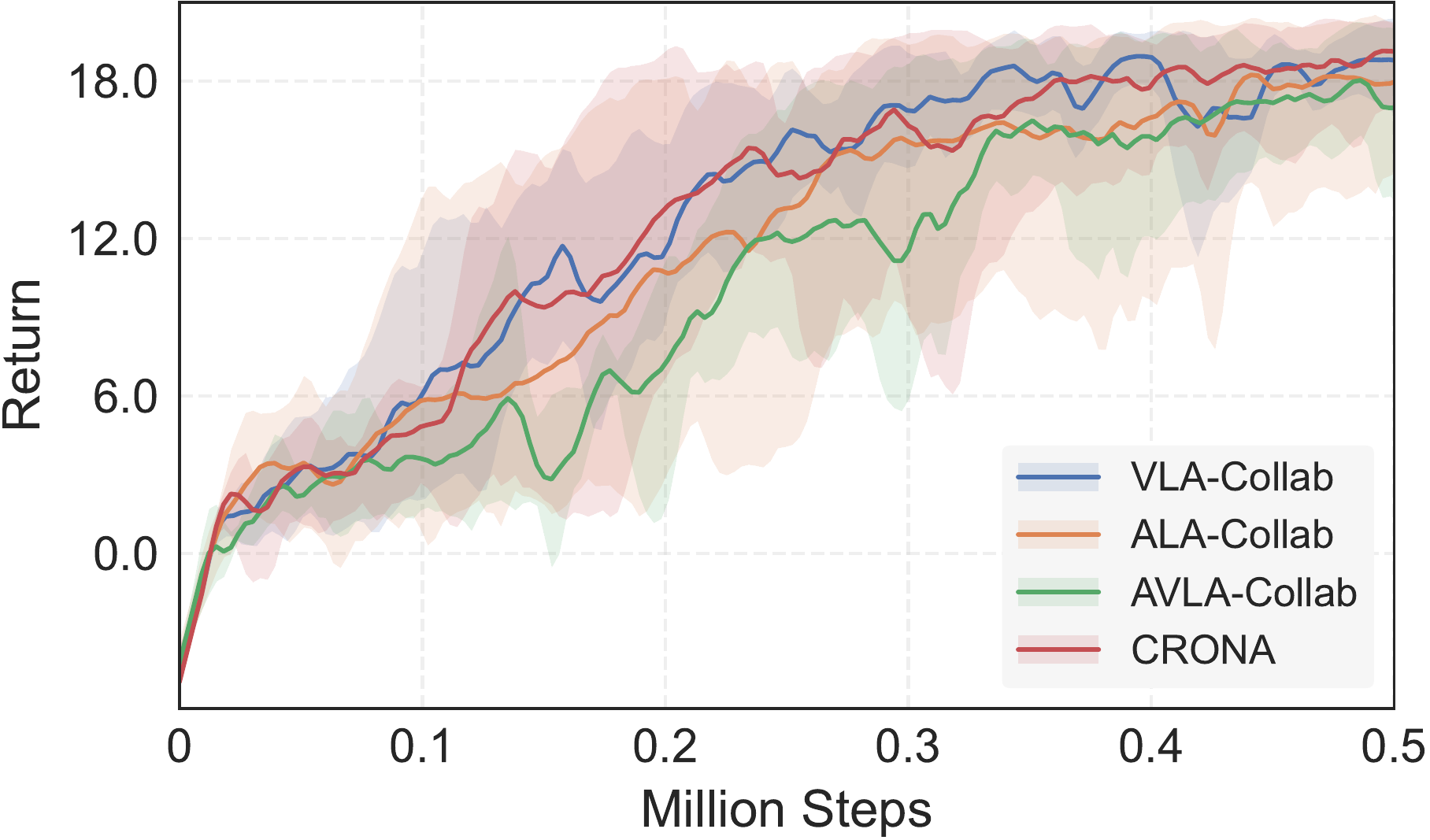}
        \caption{\;\texttt{Studio} $\mid$ \texttt{Picture}}
        \label{subfig:studio-train}
    \end{subfigure}
    \hfill
    \begin{subfigure}[b]{0.3\textwidth}
        \centering
        \includegraphics[width=\textwidth]{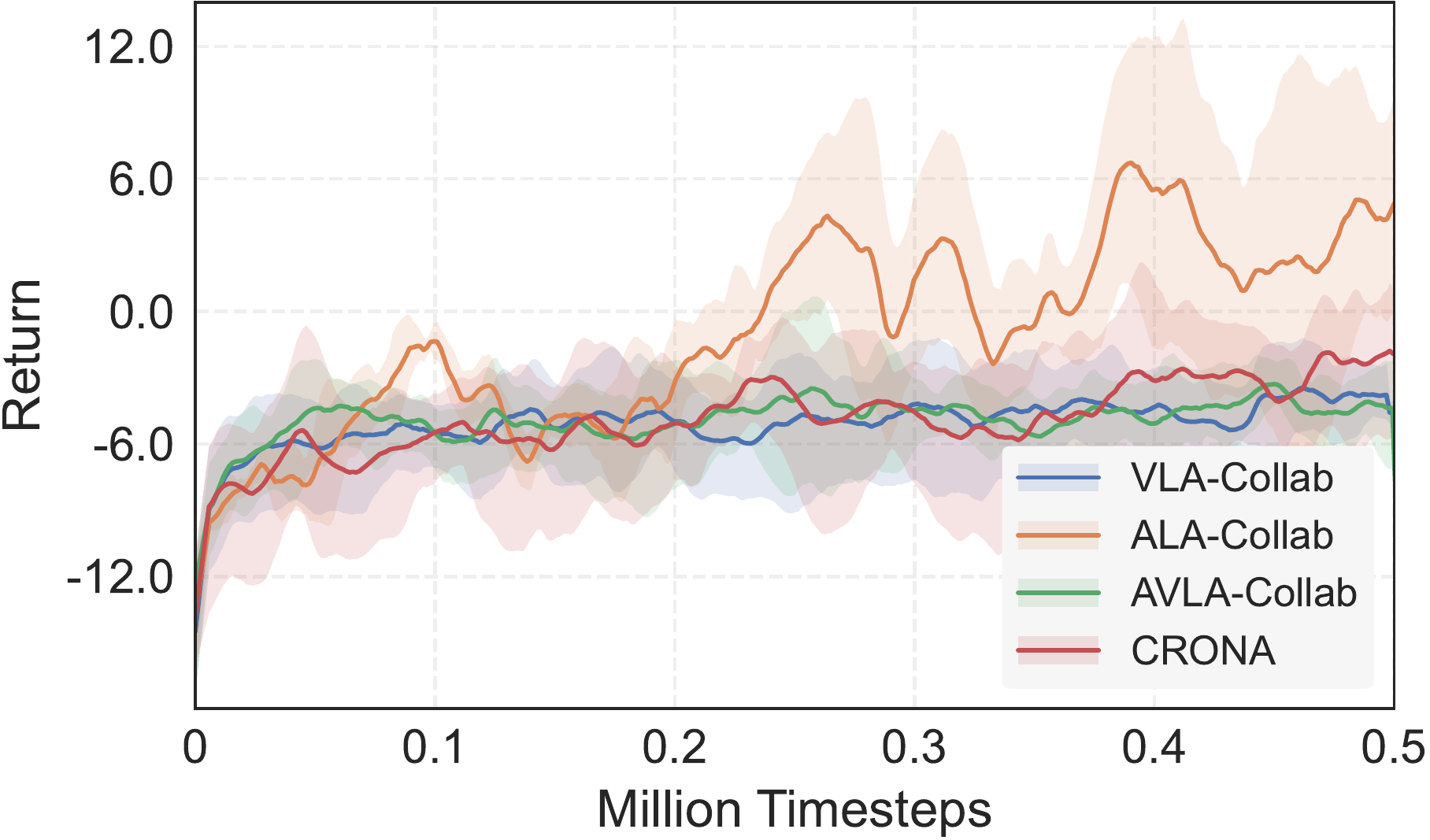}
        \caption{\;\texttt{Corridor} $\mid$ \texttt{Sink}}
        \label{subfig:corridor-train}
    \end{subfigure}
    \hfill
    \begin{subfigure}[b]{0.3\textwidth}
        \centering
        \includegraphics[width=\textwidth]{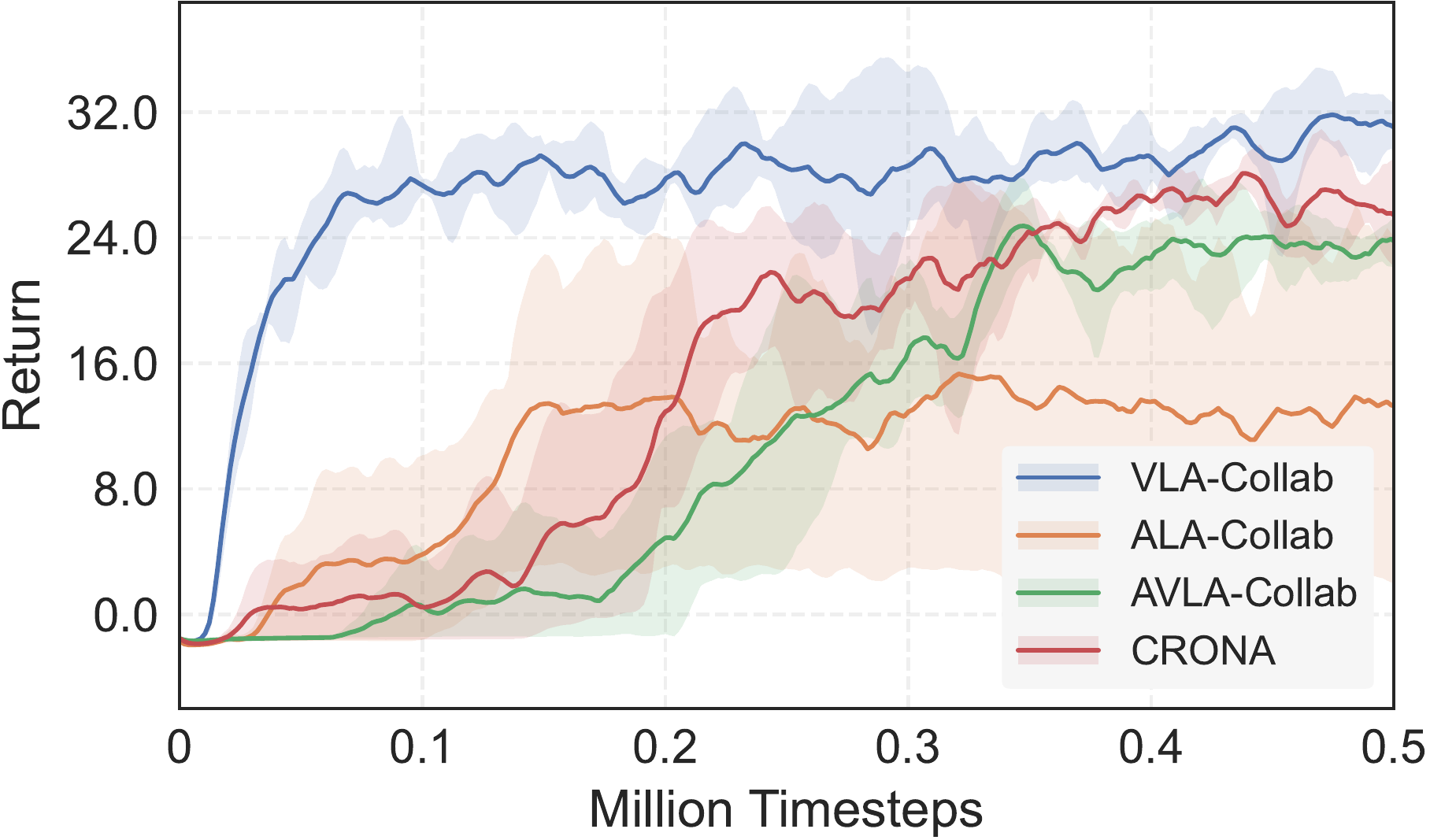}
        \caption{\,\texttt{Apartment} $\mid$ \texttt{Bed \& Counter}}
        \label{subfig:apt-train}
    \end{subfigure}
    \hspace{2mm}

    \vspace{1mm}
    \hspace{2mm}
    \begin{subfigure}[b]{0.3\textwidth}
        \centering
        \includegraphics[width=\textwidth]{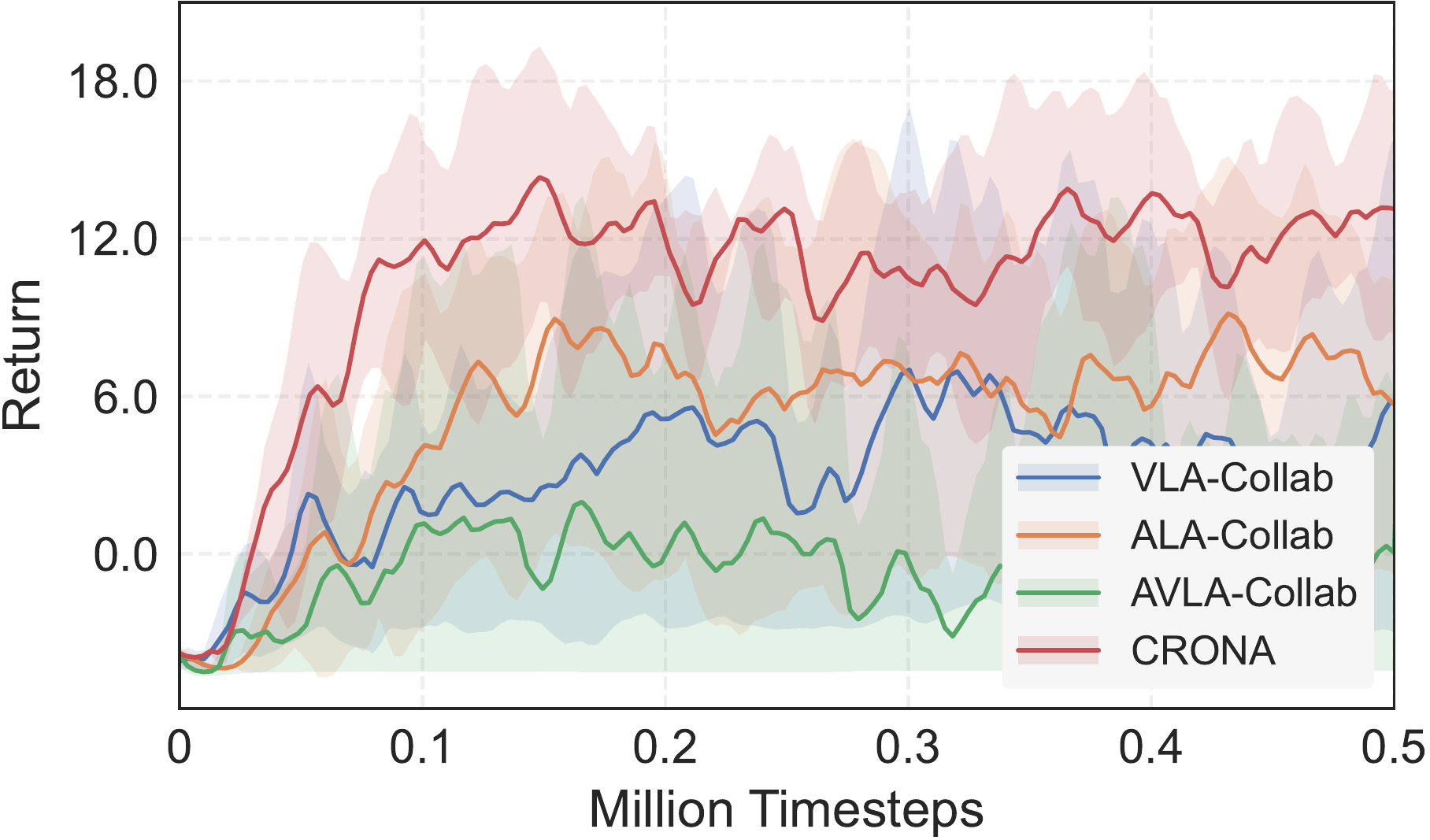}
        \caption{\;\;\texttt{Ranch} $\mid$ \texttt{Picture \& Table}}
        \label{subfig:ranch-train}
    \end{subfigure}
    \hfill
    \begin{subfigure}[b]{0.3\textwidth}
        \centering
        \includegraphics[width=\textwidth]{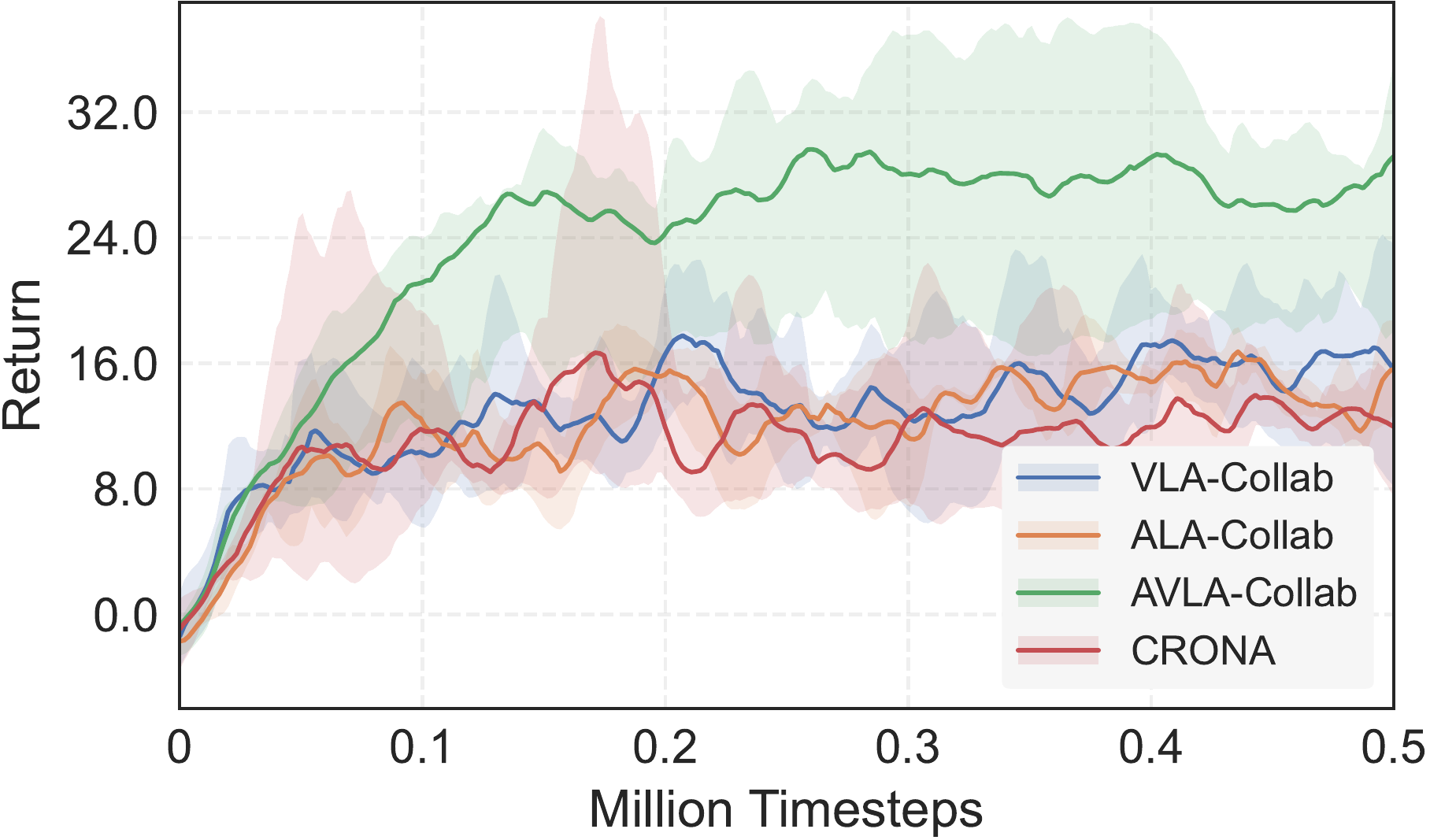}
        \caption{\texttt{Maze} $\mid$  \texttt{Drawer \& Table \& Chair}}
        \label{subfig:maze-train}
    \end{subfigure}
    \hfill
    \begin{subfigure}[b]{0.3\textwidth}
        \centering
        \includegraphics[width=\textwidth]{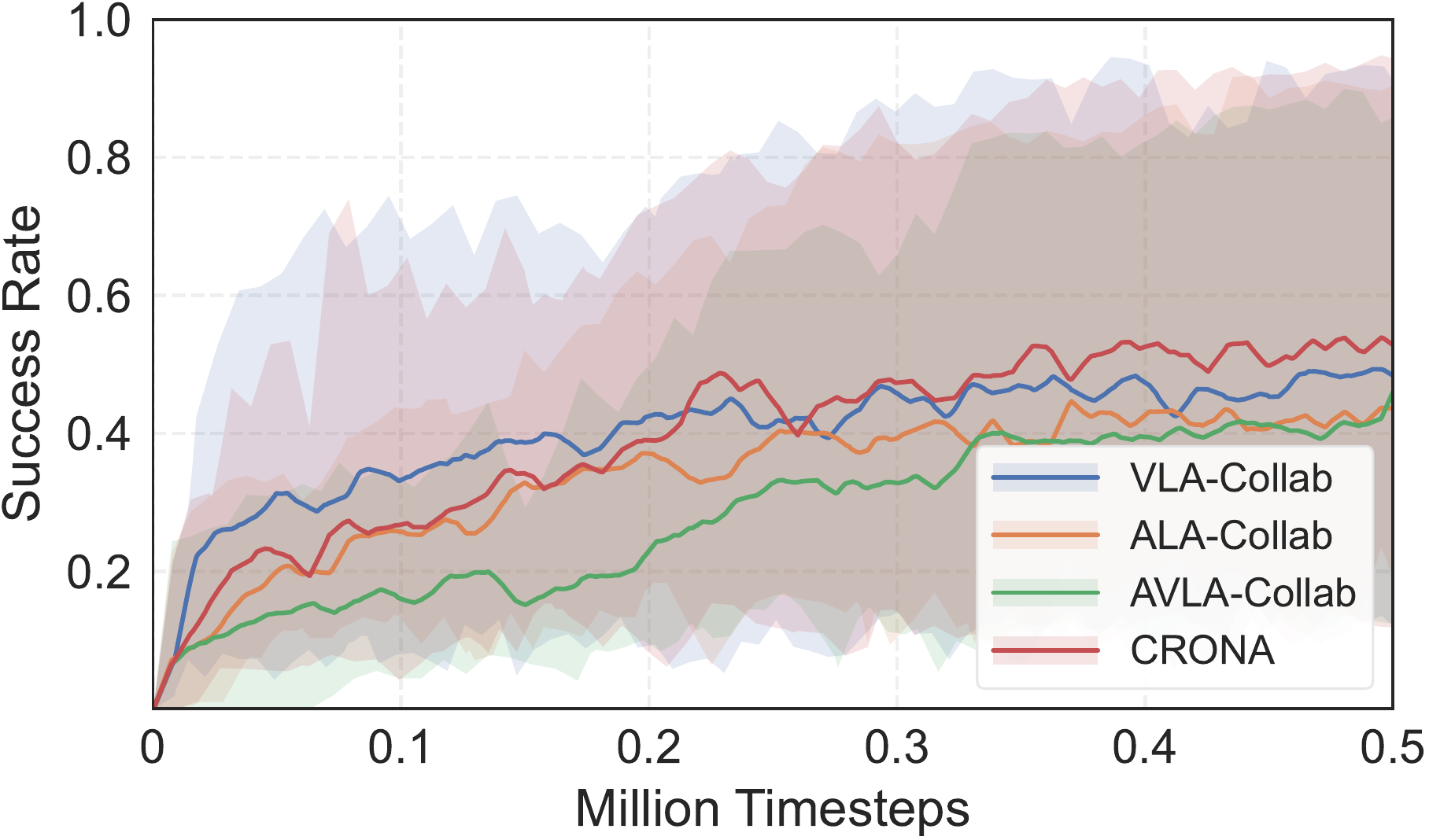}
        \caption{\;\;\texttt{Aggregated Success Rate}}
        \label{subfig:aggregated-train}
    \end{subfigure}
    \hspace{2mm}
    \caption{Evaluation of CRONA and collaborative navigation baselines across 5 \texttt{Matterport3D} scenes: (a)-(e) show the episode return; (f) shows the aggregated results of success rate. The x-axis indicates the environment steps. Curves are smoothed by an exponential moving average. Shadows denote 90\% bootstrapped CI. Results are averaged over 5 runs.\vspace{-4mm}}
    \label{fig:training}
\end{figure}

\subsection{Results}

Figure~\ref{fig:training} shows the evaluation during training, averaged over five runs. Table~\ref{tab:scene_results_combined} provides a detailed comparison between CRONA and the baselines on task completion and navigation efficiency. The effectiveness of collaborative navigation is domain-dependent. We group them into 5 patterns: no clear dominance, vision dominance, audio dominance, cross-modal, and multi-modal dominance.

\textbf{No Clear Dominance }
As shown in Figure~\ref{subfig:studio-fig} and the \texttt{Studio} columns of Table~\ref {tab:scene_results_combined}, all collaborative navigation methods perform well on \texttt{Studio}, achieving an average success rate of 90.80 $\pm$ 4.93\%. CRONA achieves the highest success rate, at 95.72\%. All collaborative methods substantially outperform the single-agent baseline. This is because decentralized agents can cover a larger exploration area and reduce the impact of premature stopping near the target. These results demonstrate the advantage of collaborative navigation: with proper training, even fully decentralized agents can coordinate effectively without communication.

\textbf{Audio Dominance }
As shown in Figure~\ref{subfig:corridor-train} and the \texttt{Corridor} columns of Table~\ref{tab:scene_results_combined}, audio cues dominate policy learning in this task. Most methods perform poorly, whereas ALA-Collab achieves the best performance with a 25.31\% success rate. This pattern is mainly due to the corridor geometry. Agents initialized near the middle of the corridor receive few informative visual cues and must rely on weak acoustic signals to infer the sound-source direction. As a result, VLA-Collab and AVLA-Collab perform worst among the collaborative baselines, with success rates around 14\%. Moreover, incorrect early decisions require long U-turns to recover, as indicated by more than 87.36\% timeouts. This audio-dominant pattern suggests that vision is not always the most reliable cue: certain targets can be localized more effectively via audio. This suggests the potential of cross-modal collaboration.

\textbf{Vision Dominance }
Figure~\ref{subfig:apt-train} and the \texttt{Apartment} columns demonstrate a vision-dominant regime. The collaboration between two vision agents achieves the best performance, reaching a success rate of 78.96\%, since the targets are large and visually salient. Audio observations are less reliable in this setting, where two audio agents achieve only 38.23\% success, mainly because mixed audio from two distinct sound sources can disrupt auxiliary belief prediction and lead to unstable policy updates. Audio agents struggle to identify the precise stopping location, as reflected by the high early-stop failure rate of 21.45\%, compared with 4.96\%-11.98\% for the other methods. 
Notably, CRONA outperforms AVLA-Collab by 5.14\% in this environment. \textbf{This suggests that weak or unreliable modalities can hurt multi-modal policies: with limited model capacity, noisy inputs may divert representational capacity away from useful cues.} We find that lower target distances and higher detection rates are associated with higher success rates, indicating that task completion is primarily governed by localization quality rather than by a single bottleneck object.

\textbf{Cross-Modal Dominance }
In \texttt{Ranch}, CRONA achieves the strongest performance, with a 64.62\% success rate (Figure~\ref{subfig:ranch-train}). We attribute this to effective collaboration between agents with complementary modalities. The audio agent localizes the picture using clean, transient camera-shutter sounds, while the vision agent identifies the table based on its large profile in an open, unobstructed dining room (Appendix~\ref{app:expset}). Interestingly, VLA-Collab and ALA-Collab achieve reasonable performance of around 40\%, but AVLA-Collab performs even worse, with only 18.93\% success. This is because the monolithic multi-modal with limited capacity struggles to align and effectively exploit different modalities (discussed in Section~\ref{subsec:ablation}). \textbf{We hypothesize that cross-modal collaboration is particularly effective when different targets have clean, modality-specific cues. It is also parameter-efficient, as each agent only needs to model its own sensory input rather than jointly aligning and reasoning over rich multi-modal observations.}

Although the success rate is generally consistent with the detection rate and average distance to targets as in other domains, steps, and timeout rate do not align with task success in this scene (Table~\ref{tab:scene_results_combined}b). Single-Agent, VLA-Collab, and ALA-Collab often terminate early or stop exploring, resulting in fewer steps but lower success. In contrast, AVLA-Collab and CRONA take more exploration steps and achieve higher success. This suggests that in harder, time-constrained tasks, inputs with heterogeneous modalities can induce more diverse behaviors and thereby promote broader exploration.

\textbf{Multi-Modal Dominance }
In the most complex scene, \texttt{Maze}, collaboration benefits from larger model capacity and access to all available sensory inputs. As shown in Figure~\ref{subfig:maze-train} and \texttt{Maze} columns in Table~\ref{tab:scene_results_combined}, AVLA-Collab achieves the best performance in \texttt{Maze}, with a 26.16\% success rate. This result is mainly consistent with the observation in \texttt{Ranch}: navigation in complex scenes requires complementary information from multiple modalities. CRONA performs only moderately worse than the homogeneous collaboration baselines, suggesting that cross-modal collaboration can still exploit partial, modality-specific inputs effectively. \textbf{This finding, together with its best overall performance (Figure~\ref{subfig:aggregated-train}), indicates that CRONA provides a robust and efficient alternative to multi-modal collaboration.} The success rates for all scenes are shown in Appendix~\ref{app:addres}.

\begin{table*}[t]
\centering
\scriptsize
\captionsetup{font=footnotesize, skip=2pt}
\caption{Comparison between CRONA and baselines across five scenes: \texttt{Studio}, \texttt{Apartment}, \texttt{Ranch}, \texttt{Corridor}, and \texttt{Maze}. \textbf{Dist}, \textbf{Detect}, and \textbf{Succ} denote the average distance from each agent to its nearest target object (m), target detection rate (\%), and task success rate (\%), respectively. \textbf{Steps} and \textbf{Timeout} denote the average number of steps used per episode and the episode timeout rate (\%), respectively. \underline{\textbf{Underlined bolds}} denote the best performance across baselines on each domain.}
\label{tab:scene_results_combined}
\newcolumntype{C}[1]{>{\centering\arraybackslash}m{#1}}

\vspace{1mm}
\textbf{(a) Task performance comparison.}
\vspace{0.5mm}

\resizebox{\textwidth}{!}{
\begin{tabular}{C{2cm} *{13}{C{0.5cm}}}
\toprule
\multirow{2}{*}{\textbf{Method}}
& \multicolumn{2}{c}{\texttt{\textbf{Studio}}}
& \multicolumn{2}{c}{\texttt{\textbf{Corridor}}}
& \multicolumn{3}{c}{\texttt{\textbf{Apartment}}}
& \multicolumn{3}{c}{\texttt{\textbf{Ranch}}}
& \multicolumn{3}{c}{\texttt{\textbf{Maze}}} \\
\cmidrule(lr){2-3} \cmidrule(lr){4-5} \cmidrule(lr){6-8} \cmidrule(lr){9-11} \cmidrule(lr){12-14}
& \textbf{Dist} & \textbf{Succ}
& \textbf{Dist} & \textbf{Succ}
& \textbf{Dist} & \textbf{Detect} & \textbf{Succ}
& \textbf{Dist} & \textbf{Detect} & \textbf{Succ}
& \textbf{Dist} & \textbf{Detect} & \textbf{Succ} \\
\midrule
\textbf{\textit{Single-Agent}}
& 3.24 & 32.66
& 11.95 & 5.71
& 8.58 & 0.84 & 31.55
& 8.68 & 0.74 & 12.34
& 7.29 & 0.18 & 0.00 \\
\textbf{\textit{VLA-Collab}}
& 1.49 & 93.65
& 9.28 & 14.54
& \underline{\textbf{2.32}} & \underline{\textbf{1.78}} & \underline{\textbf{78.96}}
& 5.75 & 0.89 & 38.97
& 6.89 & 1.06 & 18.96 \\
\textbf{\textit{ALA-Collab}}
& 3.05 & 88.17
& \underline{\textbf{8.64}} & \underline{\textbf{25.31}}
& 4.34 & 1.47 & 38.23
& 5.33 & 1.28 & 42.15
& 6.81 & 1.17 & 19.63 \\
\textbf{\textit{AVLA-Collab}}
& 2.91 & 85.87
& 9.75 & 14.29
& 3.93 & 1.61 & 63.38
& 6.87 & 0.78 & 18.93
& \underline{\textbf{6.77}} & \underline{\textbf{1.46}} & \underline{\textbf{26.16}} \\
\rowcolor{black!5}
\textbf{CRONA}
& \underline{\textbf{1.45}} & \underline{\textbf{95.72}}
& 9.11 & 21.50
& 3.64 & 1.69 & 68.52
& \underline{\textbf{5.02}} & \underline{\textbf{1.58}} & \underline{\textbf{64.62}}
& 7.06 & 0.93 & 12.13 \\
\bottomrule
\end{tabular}
}

\vspace{2mm}
\textbf{(b) Navigation efficiency comparison.}
\vspace{0.5mm}

\resizebox{\textwidth}{!}{
\begin{tabular}{C{1.9cm} *{10}{C{0.8cm}}}
\toprule
\multirow{2}{*}{\textbf{Method}}
& \multicolumn{2}{c}{\texttt{\textbf{Studio}}}
& \multicolumn{2}{c}{\texttt{\textbf{Corridor}}}
& \multicolumn{2}{c}{\texttt{\textbf{Apartment}}}
& \multicolumn{2}{c}{\texttt{\textbf{Ranch}}}
& \multicolumn{2}{c}{\texttt{\textbf{Maze}}} \\
\cmidrule(lr){2-3} \cmidrule(lr){4-5} \cmidrule(lr){6-7} \cmidrule(lr){8-9} \cmidrule(lr){10-11}
& \textbf{Steps} & \textbf{Timeout}
& \textbf{Steps} & \textbf{Timeout}
& \textbf{Steps} & \textbf{Timeout}
& \textbf{Steps} & \textbf{Timeout}
& \textbf{Steps} & \textbf{Timeout} \\
\midrule
\textbf{\textit{Single-Agent}}
& 23.40 & 1.38
& 146.58 & 95.86
& 434.60 & 56.47
& \underline{\textbf{260.11}} & \underline{\textbf{18.94}}
& \underline{\textbf{15.32}} & \underline{\textbf{0.00}} \\
\textbf{\textit{VLA-Collab}}
& 19.47 & 0.71
& 118.92 & 87.94
& \underline{\textbf{289.76}} & \underline{\textbf{16.08}}
& 318.67 & 22.41
& 129.13 & 0.79 \\
\textbf{\textit{ALA-Collab}}
& 20.18 & 0.85
& \underline{\textbf{95.66}} & \underline{\textbf{74.68}}
& 342.59 & 40.32
& 490.26 & 32.87
& 156.18 & 0.88 \\
\textbf{\textit{AVLA-Collab}}
& 21.59 & 0.92
& 116.34 & 87.36
& 308.27 & 28.96
& 396.28 & 24.88
& 624.50 & 20.45 \\
\rowcolor{black!5}
\textbf{CRONA}
& \underline{\textbf{16.08}} & \underline{\textbf{0.65}}
& 135.51 & 88.55
& 293.51 & 24.87
& 606.53 & 36.90
& 293.69 & 7.14 \\
\bottomrule
\end{tabular}
}

\vspace{-4mm}
\end{table*}

\subsection{Ablation Study} \label{subsec:ablation}

Table~\ref{tab:ablation} analyzes the effects of model capacity, input-signal quality, and framework components. 

We vary the embedding size and compare the homogeneous multi-modal baseline AVLA-Collab with CRONA. AVLA-Collab is highly sensitive to representation capacity. With a small embedding size, agents must compress visual, acoustic, and language information into a limited latent space, leading to poor collaboration performance (0.06\% success at embedding size 60). Increasing the embedding size adds only modest overhead (roughly 1 MiB for every additional 40 dimensions), but improves success rate by up to 29.61\%. With sufficient capacity, AVLA-Collab can even outperform CRONA at the same embedding size. This suggests that full-modality agents can benefit from rich inputs once capacity is no longer the bottleneck. CRONA is more stable across embedding sizes, since each agent processes fewer modalities and faces a simpler representation-learning problem.

We vary image resolution to evaluate robustness to visual signal quality. The homogeneous vision-based methods, VLA-Collab and AVLA-Collab, degrade substantially at low resolution, achieving only 12.76\% and 15.43\% success, respectively. CRONA is more robust, maintaining 42.76\%-65.48\% success across different resolutions. This robustness comes from modality specialization. Even with poor visual observations, the audio-based agent may take over and help to maintain the performance.

Finally, we ablate the auxiliary beliefs and state input of the centralized critic in Table~\ref{tab:ablation}c. Removing the category belief only slightly reduces performance, by 0.72\% for AVLA-Collab and 2.04\% for CRONA. In contrast, the location belief has a much larger effect. Once it is removed, either alone or together with the category belief, the success rate drops by about half for both methods. We find state information crucial for centralized training, where both methods almost fail (less than 0.2\% success rate) to learn without it. Overall, each component contributes to performance, with location belief and state information playing the most important roles.

\begin{table}[t]
\centering
\scriptsize
\setlength{\tabcolsep}{4pt}
\captionsetup{font=footnotesize}
\caption{Ablation studies on \texttt{Ranch}. 
(a) compares \textit{\textbf{AVLA-Collab}} and \textbf{CRONA} across embedding sizes, reporting model size (MiB) and task success rate (\%). 
(b) compares \textit{\textbf{VLA-Collab}}, \textit{\textbf{AVLA-Collab}}, and \textbf{CRONA} across visual resolutions, reporting task success rate (\%) and steps per episode. 
(c) compares \textit{\textbf{AVLA-Collab}} with \textbf{CRONA} and ablates key CRONA components.
$^{\dagger}$ denotes the pivot setting used in Table~\ref{tab:scene_results_combined}. 
Subscripted arrows show absolute changes relative to the corresponding $^{\dagger}$ pivot entry, where 
\agentred{$\uparrow$} denotes an increase and \blueagent{$\downarrow$} denotes a decrease. 
\protect\underline{\protect\textbf{Underlined bolds}} mark the best performance under each setting. \vspace{1mm}}
\label{tab:ablation}
\newcolumntype{C}[1]{>{\centering\arraybackslash}m{#1}}

\vspace{1mm}
\textbf{(a) Embedding-size ablation.}
\vspace{0.5mm}

\resizebox{\linewidth}{!}{
\begin{tabular}{C{2cm} C{1.3cm} C{1.3cm} C{1.0cm} C{1.0cm} C{1.3cm} C{1.3cm} C{1.3cm} C{1.3cm}}
\toprule
\multirow{2}{*}{\textbf{Method}}
& \multicolumn{2}{c}{\textbf{60}}
& \multicolumn{2}{c}{\textbf{100}$^{\dagger}$}
& \multicolumn{2}{c}{\textbf{140}}
& \multicolumn{2}{c}{\textbf{180}} \\
\cmidrule(lr){2-3} \cmidrule(lr){4-5} \cmidrule(lr){6-7} \cmidrule(lr){8-9}
& \textbf{Size} & \textbf{Succ}
& \textbf{Size} & \textbf{Succ}
& \textbf{Size} & \textbf{Succ}
& \textbf{Size} & \textbf{Succ} \\
\midrule
\textbf{\textit{AVLA-Collab}} 
& 36.95${}_{\;\blueagent{\downarrow\,\text{0.94}}}$ & 0.06${}_{\;\blueagent{\downarrow\,\text{18.87}}}$
& 37.89$^{\dagger}$ & 18.93$^{\dagger}$ 
& 38.83${}_{\;\agentred{\uparrow\,\text{0.94}}}$ & 43.72${}_{\;\agentred{\uparrow\,\text{24.79}}}$
& 39.76${}_{\;\agentred{\uparrow\,\text{1.87}}}$ & \underline{\textbf{73.33}}${}_{\;\agentred{\uparrow\,\text{54.40}}}$ \\
\rowcolor{black!5}
\textbf{CRONA}       
& \underline{\textbf{27.11}}${}_{\;\blueagent{\downarrow\,\text{0.93}}}$ & \underline{\textbf{11.38}}${}_{\;\blueagent{\downarrow\,\text{53.24}}}$
& \underline{\textbf{28.04}}$^{\dagger}$ & \underline{\textbf{64.62}}$^{\dagger}$ 
& \underline{\textbf{28.98}}${}_{\;\agentred{\uparrow\,\text{0.94}}}$ & \underline{\textbf{65.54}}${}_{\;\agentred{\uparrow\,\text{0.92}}}$
& \underline{\textbf{29.92}}${}_{\;\agentred{\uparrow\,\text{1.88}}}$ & 68.75${}_{\;\agentred{\uparrow\,\text{4.13}}}$ \\
\bottomrule
\end{tabular}
}

\vspace{2mm}
\textbf{(b) Resolution ablation.}
\vspace{0.5mm}

\resizebox{\linewidth}{!}{
\begin{tabular}{C{2cm} C{1.3cm} C{1.3cm} C{1.3cm} C{1.3cm} C{1.0cm} C{1.0cm} C{1.3cm} C{1.3cm}}
\toprule
\multirow{2}{*}{\textbf{Method}}
& \multicolumn{2}{c}{\textbf{4 $\times$ 4}}
& \multicolumn{2}{c}{\textbf{8 $\times$ 8}}
& \multicolumn{2}{c}{\textbf{16 $\times$ 16}$^{\dagger}$}
& \multicolumn{2}{c}{\textbf{32 $\times$ 32}} \\
\cmidrule(lr){2-3} \cmidrule(lr){4-5} \cmidrule(lr){6-7} \cmidrule(lr){8-9}
& \textbf{Succ} & \textbf{Steps}
& \textbf{Succ} & \textbf{Steps}
& \textbf{Succ} & \textbf{Steps}
& \textbf{Succ} & \textbf{Steps} \\
\midrule
\textbf{\textit{VLA-Collab}} 
& 12.76${}_{\;\blueagent{\downarrow\,\text{26.21}}}$ & \underline{\textbf{317.68}}${}_{\;\blueagent{\downarrow\,\text{0.99}}}$
& 16.51${}_{\;\blueagent{\downarrow\,\text{22.46}}}$ & 343.08${}_{\;\agentred{\uparrow\,\text{24.41}}}$
& 38.97$^{\dagger}$ & \underline{\textbf{318.67}}$^{\dagger}$
& 63.53${}_{\;\agentred{\uparrow\,\text{24.56}}}$ & 581.70${}_{\;\agentred{\uparrow\,\text{263.03}}}$ \\
\textbf{\textit{AVLA-Collab}} 
& 15.43${}_{\;\blueagent{\downarrow\,\text{3.50}}}$ & 320.76${}_{\;\blueagent{\downarrow\,\text{75.52}}}$
& 18.25${}_{\;\blueagent{\downarrow\,\text{0.68}}}$ & \underline{\textbf{322.65}}${}_{\;\blueagent{\downarrow\,\text{73.63}}}$
& 18.93$^{\dagger}$ & 396.28$^{\dagger}$ 
& 19.21${}_{\;\agentred{\uparrow\,\text{0.28}}}$ & \underline{\textbf{388.29}}${}_{\;\blueagent{\downarrow\,\text{7.99}}}$ \\
\rowcolor{black!5}
\textbf{CRONA}   
& \underline{\textbf{42.76}}${}_{\;\blueagent{\downarrow\,\text{21.86}}}$ & 346.19${}_{\;\blueagent{\downarrow\,\text{260.34}}}$
& \underline{\textbf{62.04}}${}_{\;\blueagent{\downarrow\,\text{2.58}}}$ & 573.81${}_{\;\blueagent{\downarrow\,\text{32.72}}}$
& \underline{\textbf{64.62}}$^{\dagger}$ & 606.53$^{\dagger}$ 
& \underline{\textbf{65.48}}${}_{\;\agentred{\uparrow\,\text{0.86}}}$ & 615.92${}_{\;\agentred{\uparrow\,\text{9.39}}}$ \\
\bottomrule
\end{tabular}
}

\vspace{2mm}
\textbf{(c) Component ablation.}
\vspace{0.5mm}

\resizebox{\linewidth}{!}{
\begin{tabular}{C{2.3cm} C{2.3cm} C{2.3cm} C{2.3cm} C{2.3cm} C{1cm}}
\toprule
\vspace{1mm} \textbf{Method} \vspace{0.5mm} 
& \vspace{1mm} \textbf{w/o Category Belief} \vspace{0.5mm} 
& \vspace{1mm} \textbf{w/o Location Belief}\vspace{0.5mm}  
& \vspace{1mm} \textbf{w/o Any Belief} \vspace{0.5mm} 
& \vspace{1mm} \textbf{Critic w/o State} \vspace{0.5mm}  
& \vspace{1mm} \textbf{Full}$^{\dagger}$ \vspace{0.5mm}   \\
\midrule
\textbf{\textit{AVLA-Collab}}
& 18.21${}_{\;\blueagent{\downarrow\,\text{0.72}}}$
& 8.78${}_{\;\blueagent{\downarrow\,\text{10.15}}}$
& 8.75${}_{\;\blueagent{\downarrow\,\text{10.18}}}$
& 0.06${}_{\;\blueagent{\downarrow\,\text{18.87}}}$
& \underline{\textbf{18.93}}$^{\dagger}$ \\
\rowcolor{black!5}
\textbf{CRONA}
& 62.58${}_{\;\blueagent{\downarrow\,\text{2.04}}}$
& 26.16${}_{\;\blueagent{\downarrow\,\text{38.46}}}$
& 31.40${}_{\;\blueagent{\downarrow\,\text{33.22}}}$
& 0.13${}_{\;\blueagent{\downarrow\,\text{64.49}}}$
& \underline{\textbf{64.62}}$^{\dagger}$ \\
\bottomrule
\end{tabular}
}

\vspace{-4mm}
\end{table}

\section{Conclusion}

We propose CRONA, a decentralized MARL framework for cross-modal navigation. By assigning complementary sensory modalities to different agents, CRONA reduces the burden of learning dense multi-modal representations within each agent, while retaining fully decentralized execution at test time. Experiments show that homogeneous collaboration with limited modalities may suffice for short-range navigation, while heterogeneous collaboration with complementary modalities generally performs better. In more complex scenes, richer multi-modal inputs and sufficient model capacity are also important for navigation. Overall, cross-modal collaboration is a robust and efficient alternative to multi-modal collaboration, especially when targets exhibit clean, modality-specific cues.

\textbf{Limitations}
This work has several limitations that suggest directions for future exploration. First, we focus on two common modalities, vision and audio, and extending CRONA to other sensory inputs, such as point clouds, LiDAR, or tactile signals, requires further study. Second, we use location and category beliefs as a proof of concept for auxiliary belief learning. Developing control-relevant belief representations for broader modalities and task structures is an important direction. Finally, due to constraints in the environment configuration, our current tasks are instantiated in 2D navigation settings. Extending cross-modal collaboration to full 3D embodied environments would further test its generality and practical applicability.

\clearpage
\newpage
\bibliographystyle{unsrt}
\bibliography{reference}

@inproceedings{fried2018speaker,
  title={Speaker-follower models for vision-and-language navigation},
  author={Fried, Daniel and Hu, Ronghang and Cirik, Volkan and Rohrbach, Anna and Andreas, Jacob and Morency, Louis-Philippe and Berg-Kirkpatrick, Taylor and Saenko, Kate and Klein, Dan and Darrell, Trevor},
  booktitle={Advances in neural information processing systems},
  volume={31},
  year={2018}
}

@inproceedings{hao2020towards,
  title={Towards learning a generic agent for vision-and-language navigation via pre-training},
  author={Hao, Weituo and Li, Chunyuan and Li, Xiujun and Carin, Lawrence and Gao, Jianfeng},
  booktitle={Proceedings of the IEEE/CVF conference on computer vision and pattern recognition},
  pages={13137--13146},
  year={2020}
}

@inproceedings{hong2021vlnbert,
  title={Vln bert: A recurrent vision-and-language bert for navigation},
  author={Hong, Yicong and Wu, Qi and Qi, Yuankai and Rodriguez-Opazo, Cristian and Gould, Stephen},
  booktitle={Proceedings of the IEEE/CVF conference on Computer Vision and Pattern Recognition},
  pages={1643--1653},
  year={2021}
}

@article{chen2021hamt,
  title={History aware multimodal transformer for vision-and-language navigation},
  author={Chen, Shizhe and Guhur, Pierre-Louis and Schmid, Cordelia and Laptev, Ivan},
  journal={Advances in neural information processing systems},
  volume={34},
  pages={5834--5847},
  year={2021}
}

@inproceedings{chen2020soundspaces,
  title={Soundspaces: Audio-visual navigation in 3d environments},
  author={Chen, Changan and Jain, Unnat and Schissler, Carl and Gari, Sebastia Vicenc Amengual and Al-Halah, Ziad and Ithapu, Vamsi Krishna and Robinson, Philip and Grauman, Kristen},
  booktitle={European conference on computer vision},
  pages={17--36},
  year={2020},
  organization={Springer}
}

@inproceedings{huang2022modality,
  title={Modality competition: What makes joint training of multi-modal network fail in deep learning?(provably)},
  author={Huang, Yu and Lin, Junyang and Zhou, Chang and Yang, Hongxia and Huang, Longbo},
  booktitle={International conference on machine learning},
  pages={9226--9259},
  year={2022},
  organization={PMLR}
}

@inproceedings{anderson2018vision,
  title={Vision-and-language navigation: Interpreting visually-grounded navigation instructions in real environments},
  author={Anderson, Peter and Wu, Qi and Teney, Damien and Bruce, Jake and Johnson, Mark and S{\"u}nderhauf, Niko and Reid, Ian and Gould, Stephen and Van Den Hengel, Anton},
  booktitle={Proceedings of the IEEE conference on computer vision and pattern recognition},
  pages={3674--3683},
  year={2018}
}

@inproceedings{wang2020makes,
  title={What makes training multi-modal classification networks hard?},
  author={Wang, Weiyao and Tran, Du and Feiszli, Matt},
  booktitle={Proceedings of the IEEE/CVF conference on computer vision and pattern recognition},
  pages={12695--12705},
  year={2020}
}

@inproceedings{wu2022characterizing,
  title={Characterizing and overcoming the greedy nature of learning in multi-modal deep neural networks},
  author={Wu, Nan and Jastrzebski, Stanislaw and Cho, Kyunghyun and Geras, Krzysztof J},
  booktitle={International Conference on Machine Learning},
  pages={24043--24055},
  year={2022},
  organization={PMLR}
}

@inproceedings{chen2021semantic,
  title={Semantic audio-visual navigation},
  author={Chen, Changan and Al-Halah, Ziad and Grauman, Kristen},
  booktitle={Proceedings of the IEEE/CVF Conference on Computer Vision and Pattern Recognition},
  pages={15516--15525},
  year={2021}
}

@inproceedings{huang2023visual,
  title={Visual language maps for robot navigation},
  author={Huang, Chenguang and Mees, Oier and Zeng, Andy and Burgard, Wolfram},
  booktitle={2023 IEEE International Conference on Robotics and Automation (ICRA)},
  pages={10608--10615},
  year={2023},
  organization={IEEE}
}

@inproceedings{krantz2020beyond,
  title={Beyond the nav-graph: Vision-and-language navigation in continuous environments},
  author={Krantz, Jacob and Wijmans, Erik and Majumdar, Arjun and Batra, Dhruv and Lee, Stefan},
  booktitle={European Conference on Computer Vision},
  pages={104--120},
  year={2020},
  organization={Springer}
}

@inproceedings{qi2021road,
  title={The road to know-where: An object-and-room informed sequential bert for indoor vision-language navigation},
  author={Qi, Yuankai and Pan, Zizheng and Hong, Yicong and Yang, Ming-Hsuan and Van Den Hengel, Anton and Wu, Qi},
  booktitle={Proceedings of the IEEE/CVF International Conference on Computer Vision},
  pages={1655--1664},
  year={2021}
}

@inproceedings{majumdar2020improving,
  title={Improving vision-and-language navigation with image-text pairs from the web},
  author={Majumdar, Arjun and Shrivastava, Ayush and Lee, Stefan and Anderson, Peter and Parikh, Devi and Batra, Dhruv},
  booktitle={European Conference on Computer Vision},
  pages={259--274},
  year={2020},
  organization={Springer}
}

@inproceedings{yu2023l3mvn,
  title={L3mvn: Leveraging large language models for visual target navigation},
  author={Yu, Bangguo and Kasaei, Hamidreza and Cao, Ming},
  booktitle={2023 IEEE/RSJ International Conference on Intelligent Robots and Systems (IROS)},
  pages={3554--3560},
  year={2023},
  organization={IEEE}
}

@article{paul2022avlen,
  title={Avlen: Audio-visual-language embodied navigation in 3d environments},
  author={Paul, Sudipta and Roy-Chowdhury, Amit and Cherian, Anoop},
  journal={Advances in Neural Information Processing Systems},
  volume={35},
  pages={6236--6249},
  year={2022}
}

@article{chen2020learning,
  title={Learning to set waypoints for audio-visual navigation},
  author={Chen, Changan and Majumder, Sagnik and Al-Halah, Ziad and Gao, Ruohan and Ramakrishnan, Santhosh Kumar and Grauman, Kristen},
  journal={arXiv preprint arXiv:2008.09622},
  year={2020}
}

@article{bruce2006safe,
  title={Safe multirobot navigation within dynamics constraints},
  author={Bruce, James R and Veloso, Manuela M},
  journal={Proceedings of the IEEE},
  volume={94},
  number={7},
  pages={1398--1411},
  year={2006},
  publisher={IEEE}
}

@book{decpomdp,
    author = {Oliehoek, Frans A. and Amato, Christopher},
    title = {A Concise Introduction to Decentralized POMDPs},
    year = {2016},
    publisher = {Springer},
    doi = {10.1007/978-3-319-28929-8}
}

@inproceedings{van2009centralized,
  title={Centralized path planning for multiple robots: Optimal decoupling into sequential plans.},
  author={van Den Berg, Jur and Snoeyink, Jack and Lin, Ming C and Manocha, Dinesh},
  booktitle={Robotics: Science and systems},
  volume={2},
  pages={2--3},
  year={2009}
}

@article{janssen2016cloud,
  title={Cloud based centralized task control for human domain multi-robot operations},
  author={Janssen, Rob and van de Molengraft, Ren{\'e} and Bruyninckx, Herman and Steinbuch, Maarten},
  journal={Intelligent Service Robotics},
  volume={9},
  number={1},
  pages={63--77},
  year={2016},
  publisher={Springer}
}

@book{ marl-book,
  author = {Stefano V. Albrecht and Filippos Christianos and Lukas Sch\"afer},
  title = {Multi-Agent Reinforcement Learning: Foundations and Modern Approaches},
  publisher = {MIT Press},
  year = {2024},
  url = {https://www.marl-book.com}
}

@article{amato2024introduction,
  title={An introduction to centralized training for decentralized execution in cooperative multi-agent reinforcement learning},
  author={Amato, Christopher},
  journal={arXiv preprint arXiv:2409.03052},
  year={2024}
}

@article{yuan2023survey,
  title={A survey of progress on cooperative multi-agent reinforcement learning in open environment},
  author={Yuan, Lei and Zhang, Ziqian and Li, Lihe and Guan, Cong and Yu, Yang},
  journal={arXiv preprint arXiv:2312.01058},
  year={2023}
}

@article{zhang2021multi,
  title={Multi-agent reinforcement learning: A selective overview of theories and algorithms},
  author={Zhang, Kaiqing and Yang, Zhuoran and Ba{\c{s}}ar, Tamer},
  journal={Handbook of reinforcement learning and control},
  pages={321--384},
  year={2021},
  publisher={Springer}
}

@inproceedings{tan1993multi,
  title={Multi-agent reinforcement learning: Independent vs. cooperative agents},
  author={Tan, Ming and others},
  booktitle={Proceedings of the tenth international conference on machine learning},
  pages={330--337},
  year={1993}
}

@article{peshkin2001learning,
  title={Learning to cooperate via policy search},
  author={Peshkin, Leonid and Kim, Kee-Eung and Meuleau, Nicolas and Kaelbling, Leslie Pack},
  journal={arXiv preprint cs/0105032},
  year={2001}
}

@article{lyu2021contrasting,
  title={Contrasting centralized and decentralized critics in multi-agent reinforcement learning},
  author={Lyu, Xueguang and Xiao, Yuchen and Daley, Brett and Amato, Christopher},
  journal={arXiv preprint arXiv:2102.04402},
  year={2021}
}

@article{lyu2023centralized,
  title={On centralized critics in multi-agent reinforcement learning},
  author={Lyu, Xueguang and Baisero, Andrea and Xiao, Yuchen and Daley, Brett and Amato, Christopher},
  journal={Journal of Artificial Intelligence Research},
  volume={77},
  pages={295--354},
  year={2023}
}

@article{MADDPG,
  title={Multi-agent actor-critic for mixed cooperative-competitive environments},
  author={Lowe, Ryan and Wu, Yi I and Tamar, Aviv and Harb, Jean and Pieter Abbeel, OpenAI and Mordatch, Igor},
  journal={Advances in neural information processing systems},
  volume={30},
  year={2017}
}

@inproceedings{MAPPO,
 title = {The Surprising Effectiveness of PPO in Cooperative Multi-Agent Games},
 author = {Yu, Chao and Velu, Akash and Vinitsky, Eugene and Gao, Jiaxuan and Wang, Yu and Bayen, Alexandre and Wu, Yi},
 booktitle = {Advances in Neural Information Processing Systems},
 pages = {24611--24624},
 publisher = {Curran Associates, Inc.},
 volume = {35},
 year = {2022}
}

@inproceedings{COMA,
  title = {Counterfactual Multi-Agent Policy Gradients},
  author = {Foerster, Jakob and Farquhar, Gregory and Afouras, Triantafyllos and Nardelli, Nantas and Whiteson, Shimon},
  year = {2018},
  month = apr,
  booktitle = {Proceedings of the AAAI Conference on Artificial Intelligence},
}

@article{claus1998dynamics,
  title={The dynamics of reinforcement learning in cooperative multiagent systems},
  author={Claus, Caroline and Boutilier, Craig},
  journal={AAAI/IAAI},
  volume={1998},
  number={746-752},
  pages={2},
  year={1998}
}

@inproceedings{tuyls2003selection,
  title={A selection-mutation model for q-learning in multi-agent systems},
  author={Tuyls, Karl and Verbeeck, Katja and Lenaerts, Tom},
  booktitle={Proceedings of the second international joint conference on Autonomous agents and multiagent systems},
  pages={693--700},
  year={2003}
}

@inproceedings{wunder2010classes,
  title={Classes of multiagent q-learning dynamics with epsilon-greedy exploration},
  author={Wunder, Michael and Littman, Michael L and Babes, Monica},
  booktitle={Proceedings of the 27th International Conference on Machine Learning (ICML-10)},
  pages={1167--1174},
  year={2010}
}

@inproceedings{zhang2025advancing,
  title={Advancing Audio-Visual Navigation Through Multi-Agent Collaboration in 3D Environments},
  author={Zhang, Hailong and Yu, Yinfeng and Wang, Liejun and Sun, Fuchun and Zheng, Wendong},
  booktitle={International Conference on Neural Information Processing},
  pages={502--516},
  year={2025},
  organization={Springer}
}

@article{hao2025conav,
  title={Conav: Collaborative cross-modal reasoning for embodied navigation},
  author={Hao, Haihong and Han, Mingfei and Li, Changlin and Li, Zhihui and Chang, Xiaojun},
  journal={arXiv preprint arXiv:2505.16663},
  year={2025}
}

@article{liu2025caml,
  title={Caml: Collaborative auxiliary modality learning for multi-agent systems},
  author={Liu, Rui and Shen, Yu and Gao, Peng and Tokekar, Pratap and Lin, Ming},
  journal={arXiv preprint arXiv:2502.17821},
  year={2025}
}

@article{huang2021decentralized,
  title={Decentralized autonomous navigation of a UAV network for road traffic monitoring},
  author={Huang, Hailong and Savkin, Andrey V and Huang, Chao},
  journal={IEEE Transactions on Aerospace and Electronic Systems},
  volume={57},
  number={4},
  pages={2558--2564},
  year={2021},
  publisher={IEEE}
}

@article{qin2021fully,
  title={Fully decentralized cooperative navigation for spacecraft constellations},
  author={Qin, Tong and Macdonald, Malcolm and Qiao, Dong},
  journal={IEEE Transactions on Aerospace and Electronic Systems},
  volume={57},
  number={4},
  pages={2383--2394},
  year={2021},
  publisher={IEEE}
}

@article{azzam2023swarm,
  title={Swarm cooperative navigation using centralized training and decentralized execution},
  author={Azzam, Rana and Boiko, Igor and Zweiri, Yahya},
  journal={Drones},
  volume={7},
  number={3},
  pages={193},
  year={2023},
  publisher={MDPI}
}

@inproceedings{wang2024multi,
  title={Multi-robot cooperative socially-aware navigation using multi-agent reinforcement learning},
  author={Wang, Weizheng and Mao, Le and Wang, Ruiqi and Min, Byung-Cheol},
  booktitle={2024 IEEE International Conference on Robotics and Automation (ICRA)},
  pages={12353--12360},
  year={2024},
  organization={IEEE}
}

@article{Matterport3D,
  title={Matterport3D: Learning from RGB-D Data in Indoor Environments},
  author={Chang, Angel and Dai, Angela and Funkhouser, Thomas and Halber, Maciej and Niessner, Matthias and Savva, Manolis and Song, Shuran and Zeng, Andy and Zhang, Yinda},
  journal={International Conference on 3D Vision (3DV)},
  year={2017}
}

@inproceedings{szot2021habitat,
  title     =     {Habitat 2.0: Training Home Assistants to Rearrange their Habitat},
  author    =     {Andrew Szot and Alex Clegg and Eric Undersander and Erik Wijmans and Yili Zhao and John Turner and Noah Maestre and Mustafa Mukadam and Devendra Chaplot and Oleksandr Maksymets and Aaron Gokaslan and Vladimir Vondrus and Sameer Dharur and Franziska Meier and Wojciech Galuba and Angel Chang and Zsolt Kira and Vladlen Koltun and Jitendra Malik and Manolis Savva and Dhruv Batra},
  booktitle =     {Advances in Neural Information Processing Systems (NeurIPS)},
  year      =     {2021}
}

@inproceedings{habitat19iccv,
  title     =     {Habitat: {A} {P}latform for {E}mbodied {AI} {R}esearch},
  author    =     {{Manolis Savva*} and {Abhishek Kadian*} and {Oleksandr Maksymets*} and Yili Zhao and Erik Wijmans and Bhavana Jain and Julian Straub and Jia Liu and Vladlen Koltun and Jitendra Malik and Devi Parikh and Dhruv Batra},
  booktitle =     {Proceedings of the IEEE/CVF International Conference on Computer Vision (ICCV)},
  year      =     {2019}
}

@article{shridhar2020alfworld,
  title={Alfworld: Aligning text and embodied environments for interactive learning},
  author={Shridhar, Mohit and Yuan, Xingdi and C{\^o}t{\'e}, Marc-Alexandre and Bisk, Yonatan and Trischler, Adam and Hausknecht, Matthew},
  journal={arXiv preprint arXiv:2010.03768},
  year={2020}
}

@article{duan2022survey,
  title={A survey of embodied ai: From simulators to research tasks},
  author={Duan, Jiafei and Yu, Samson and Tan, Hui Li and Zhu, Hongyuan and Tan, Cheston},
  journal={IEEE Transactions on Emerging Topics in Computational Intelligence},
  volume={6},
  number={2},
  pages={230--244},
  year={2022},
  publisher={IEEE}
}

@article{kolve2017ai2,
  title={Ai2-thor: An interactive 3d environment for visual ai},
  author={Kolve, Eric and Mottaghi, Roozbeh and Han, Winson and VanderBilt, Eli and Weihs, Luca and Herrasti, Alvaro and Deitke, Matt and Ehsani, Kiana and Gordon, Daniel and Zhu, Yuke and others},
  journal={arXiv preprint arXiv:1712.05474},
  year={2017}
}

@inproceedings{shridhar2020alfred,
  title={Alfred: A benchmark for interpreting grounded instructions for everyday tasks},
  author={Shridhar, Mohit and Thomason, Jesse and Gordon, Daniel and Bisk, Yonatan and Han, Winson and Mottaghi, Roozbeh and Zettlemoyer, Luke and Fox, Dieter},
  booktitle={Proceedings of the IEEE/CVF conference on computer vision and pattern recognition},
  pages={10740--10749},
  year={2020}
}

@inproceedings{chen2017multi,
  title={Multi-view 3d object detection network for autonomous driving},
  author={Chen, Xiaozhi and Ma, Huimin and Wan, Ji and Li, Bo and Xia, Tian},
  booktitle={Proceedings of the IEEE conference on Computer Vision and Pattern Recognition},
  pages={1907--1915},
  year={2017}
}

@article{brohan2022rt,
  title={Rt-1: Robotics transformer for real-world control at scale},
  author={Brohan, Anthony and Brown, Noah and Carbajal, Justice and Chebotar, Yevgen and Dabis, Joseph and Finn, Chelsea and Gopalakrishnan, Keerthana and Hausman, Karol and Herzog, Alex and Hsu, Jasmine and others},
  journal={arXiv preprint arXiv:2212.06817},
  year={2022}
}

@inproceedings{ku2018joint,
  title={Joint 3d proposal generation and object detection from view aggregation},
  author={Ku, Jason and Mozifian, Melissa and Lee, Jungwook and Harakeh, Ali and Waslander, Steven L},
  booktitle={2018 IEEE/RSJ international conference on intelligent robots and systems (IROS)},
  pages={1--8},
  year={2018},
  organization={IEEE}
}

@inproceedings{qi2018frustum,
  title={Frustum pointnets for 3d object detection from rgb-d data},
  author={Qi, Charles R and Liu, Wei and Wu, Chenxia and Su, Hao and Guibas, Leonidas J},
  booktitle={Proceedings of the IEEE conference on computer vision and pattern recognition},
  pages={918--927},
  year={2018}
}

@inproceedings{shridhar2022cliport,
  title={Cliport: What and where pathways for robotic manipulation},
  author={Shridhar, Mohit and Manuelli, Lucas and Fox, Dieter},
  booktitle={Conference on robot learning},
  pages={894--906},
  year={2022},
  organization={PMLR}
}

@article{driess2023palm,
  title={Palm-e: An embodied multimodal language model},
  author={Driess, Danny and Xia, Fei and Sajjadi, Mehdi SM and Lynch, Corey and Chowdhery, Aakanksha and Ichter, Brian and Wahid, Ayzaan and Tompson, Jonathan and Vuong, Quan and Yu, Tianhe and others},
  journal={arXiv preprint arXiv:2303.03378},
  year={2023}
}

@inproceedings{das2017visual,
  title={Visual dialog},
  author={Das, Abhishek and Kottur, Satwik and Gupta, Khushi and Singh, Avi and Yadav, Deshraj and Moura, Jos{\'e} MF and Parikh, Devi and Batra, Dhruv},
  booktitle={Proceedings of the IEEE conference on computer vision and pattern recognition},
  pages={326--335},
  year={2017}
}

@inproceedings{jia2021scaling,
  title={Scaling up visual and vision-language representation learning with noisy text supervision},
  author={Jia, Chao and Yang, Yinfei and Xia, Ye and Chen, Yi-Ting and Parekh, Zarana and Pham, Hieu and Le, Quoc and Sung, Yun-Hsuan and Li, Zhen and Duerig, Tom},
  booktitle={International conference on machine learning},
  pages={4904--4916},
  year={2021},
  organization={PMLR}
}

@inproceedings{tsai2019multimodal,
  title={Multimodal transformer for unaligned multimodal language sequences},
  author={Tsai, Yao-Hung Hubert and Bai, Shaojie and Liang, Paul Pu and Kolter, J Zico and Morency, Louis-Philippe and Salakhutdinov, Ruslan},
  booktitle={Proceedings of the 57th annual meeting of the association for computational linguistics},
  pages={6558--6569},
  year={2019}
}

@inproceedings{li2023blip,
  title={Blip-2: Bootstrapping language-image pre-training with frozen image encoders and large language models},
  author={Li, Junnan and Li, Dongxu and Savarese, Silvio and Hoi, Steven},
  booktitle={International conference on machine learning},
  pages={19730--19742},
  year={2023},
  organization={PMLR}
}

@inproceedings{peng2022balanced,
  title={Balanced multimodal learning via on-the-fly gradient modulation},
  author={Peng, Xiaokang and Wei, Yake and Deng, Andong and Wang, Dong and Hu, Di},
  booktitle={Proceedings of the IEEE/CVF conference on computer vision and pattern recognition},
  pages={8238--8247},
  year={2022}
}

@inproceedings{kim2021vilt,
  title={Vilt: Vision-and-language transformer without convolution or region supervision},
  author={Kim, Wonjae and Son, Bokyung and Kim, Ildoo},
  booktitle={International conference on machine learning},
  pages={5583--5594},
  year={2021},
  organization={PMLR}
}

@article{yao2024minicpm,
  title={Minicpm-v: A gpt-4v level mllm on your phone},
  author={Yao, Yuan and Yu, Tianyu and Zhang, Ao and Wang, Chongyi and Cui, Junbo and Zhu, Hongji and Cai, Tianchi and Li, Haoyu and Zhao, Weilin and He, Zhihui and others},
  journal={arXiv preprint arXiv:2408.01800},
  year={2024}
}

@article{hu2023semantic,
  title={Semantic collaborative learning for cross-modal moment localization},
  author={Hu, Yupeng and Wang, Kun and Liu, Meng and Tang, Haoyu and Nie, Liqiang},
  journal={ACM Transactions on Information Systems},
  volume={42},
  number={2},
  pages={1--26},
  year={2023},
  publisher={ACM New York, NY, USA}
}

@inproceedings{velagapudi2010decentralized,
  title={Decentralized prioritized planning in large multirobot teams},
  author={Velagapudi, Prasanna and Sycara, Katia and Scerri, Paul},
  booktitle={2010 IEEE/RSJ International Conference on Intelligent Robots and Systems},
  pages={4603--4609},
  year={2010},
  organization={IEEE}
}

@inproceedings{luna2011efficient,
  title={Efficient and complete centralized multi-robot path planning},
  author={Luna, Ryan and Bekris, Kostas E},
  booktitle={2011 IEEE/RSJ International Conference on Intelligent Robots and Systems},
  pages={3268--3275},
  year={2011},
  organization={IEEE}
}

@inproceedings{iqbal2019actor,
  title={Actor-attention-critic for multi-agent reinforcement learning},
  author={Iqbal, Shariq and Sha, Fei},
  booktitle={International conference on machine learning},
  pages={2961--2970},
  year={2019},
  organization={PMLR}
}

@article{oliehoek2008optimal,
  title={Optimal and approximate Q-value functions for decentralized POMDPs},
  author={Oliehoek, Frans A and Spaan, Matthijs TJ and Vlassis, Nikos},
  journal={Journal of Artificial Intelligence Research},
  volume={32},
  pages={289--353},
  year={2008}
}

@article{xiao2022asynchronous,
  title={Asynchronous actor-critic for multi-agent reinforcement learning},
  author={Xiao, Yuchen and Tan, Weihao and Amato, Christopher},
  journal={Advances in Neural Information Processing Systems},
  volume={35},
  pages={4385--4400},
  year={2022}
}

@inproceedings{xiao2020learning,
  title={Learning multi-robot decentralized macro-action-based policies via a centralized q-net},
  author={Xiao, Yuchen and Hoffman, Joshua and Xia, Tian and Amato, Christopher},
  booktitle={2020 IEEE International conference on robotics and automation (ICRA)},
  pages={10695--10701},
  year={2020},
  organization={IEEE}
}

@article{xue2023multi,
  title={Multi-agent deep reinforcement learning for UAVs navigation in unknown complex environment},
  author={Xue, Yuntao and Chen, Weisheng},
  journal={IEEE Transactions on Intelligent Vehicles},
  volume={9},
  number={1},
  pages={2290--2303},
  year={2023},
  publisher={IEEE}
}

@inproceedings{he2016deep,
  title={Deep residual learning for image recognition},
  author={He, Kaiming and Zhang, Xiangyu and Ren, Shaoqing and Sun, Jian},
  booktitle={Proceedings of the IEEE conference on computer vision and pattern recognition},
  pages={770--778},
  year={2016}
}

@article{mcfee2015librosa,
  title={librosa: Audio and music signal analysis in python.},
  author={McFee, Brian and Raffel, Colin and Liang, Dawen and Ellis, Daniel PW and McVicar, Matt and Battenberg, Eric and Nieto, Oriol and others},
  journal={SciPy},
  volume={2015},
  number={18-24},
  pages={7},
  year={2015}
}

@inproceedings{wang2026conavbench,
  title={CoNavBench: Collaborative Long-Horizon Vision-Language Navigation Benchmark},
  author={Wang, Tianhang and Li, Xinhai and Lu, Fan and Gong, Tianshi and Dong, Jiankun and Xue, Weiyi and Qu, Sanqing and Bai, Chenjia and Chen, Guang},
  booktitle={The Fourteenth International Conference on Learning Representations},
  year={2026}
}

@article{yuksel2026gram,
  title={GRAM: Spatial general-purpose audio representations for real-world environments},
  author={Yuksel, Goksenin and van Gerven, Marcel and van der Heijden, Kiki},
  journal={arXiv preprint arXiv:2602.03307},
  year={2026}
}

@inproceedings{cheng2018mobile,
  title={Mobile robot navigation based on lidar},
  author={Cheng, Yi and Wang, Gong Ye},
  booktitle={2018 Chinese control and decision conference (CCDC)},
  pages={1243--1246},
  year={2018},
  organization={IEEE}
}

@article{yao2023radar,
  title={Radar-camera fusion for object detection and semantic segmentation in autonomous driving: A comprehensive review},
  author={Yao, Shanliang and Guan, Runwei and Huang, Xiaoyu and Li, Zhuoxiao and Sha, Xiangyu and Yue, Yong and Lim, Eng Gee and Seo, Hyungjoon and Man, Ka Lok and Zhu, Xiaohui and others},
  journal={IEEE Transactions on Intelligent Vehicles},
  volume={9},
  number={1},
  pages={2094--2128},
  year={2023},
  publisher={IEEE}
}

@article{wang2021collaborative,
  title={Collaborative visual navigation},
  author={Wang, Haiyang and Wang, Wenguan and Zhu, Xizhou and Dai, Jifeng and Wang, Liwei},
  journal={arXiv preprint arXiv:2107.01151},
  year={2021}
}

@article{parker2002alliance,
  title={ALLIANCE: An architecture for fault tolerant multirobot cooperation},
  author={Parker, Lynne E},
  journal={IEEE transactions on robotics and automation},
  volume={14},
  number={2},
  pages={220--240},
  year={2002},
  publisher={IEEE}
}

@inproceedings{simmons2000coordination,
  title={Coordination for multi-robot exploration and mapping},
  author={Simmons, Reid and Apfelbaum, David and Burgard, Wolfram and Fox, Dieter and Moors, Mark and Thrun, Sebastian and Younes, H{\aa}kan},
  booktitle={Aaai/Iaai},
  pages={852--858},
  year={2000}
}

@article{gu2021attention,
  title={Attention-based fault-tolerant approach for multi-agent reinforcement learning systems},
  author={Gu, Shanzhi and Geng, Mingyang and Lan, Long},
  journal={Entropy},
  volume={23},
  number={9},
  pages={1133},
  year={2021},
  publisher={MDPI}
}

@article{burgard2005coordinated,
  title={Coordinated multi-robot exploration},
  author={Burgard, Wolfram and Moors, Mark and Stachniss, Cyrill and Schneider, Frank E},
  journal={IEEE Transactions on robotics},
  volume={21},
  number={3},
  pages={376--386},
  year={2005},
  publisher={IEEE}
}

\newpage
\appendix

\section{Proofs} \label{app:proof}

\begin{proposition}
Assume that the belief predictor is accurate, such that the predicted belief
$\,\mathbf{b}$ is a correct history-induced latent representation of
$\,\mathbf{h}$. Then the state-history-belief value
$\,V^{\boldsymbol{\pi}}(\mathbf{h},\mathbf{b},s)$ provides an unbiased
estimate of $\,V^{\boldsymbol{\pi}}(\mathbf{h})$.
\end{proposition}

\begin{proof}
By Lemma 2 of \cite{lyu2023centralized}, the state-augmented action-value
function is unbiased with respect to the history-conditioned value, i.e.,
\[
Q^{\boldsymbol{\pi}}(\mathbf{h},\mathbf{a})
=
\mathbb{E}_{s\mid \mathbf{h}}
\left[
Q^{\boldsymbol{\pi}}(\mathbf{h},s,\mathbf{a})
\right].
\]
Since $\mathbf{b}$ is inferred from $\mathbf{h}$ and is assumed to be correct,
conditioning on $(\mathbf{h},\mathbf{b})$ does not introduce additional
information beyond the history. Therefore,
\[
p(s\mid \mathbf{h},\mathbf{b}) = p(s\mid \mathbf{h}),
\]
and the augmented critic satisfies
\[
Q^{\boldsymbol{\pi}}(\mathbf{h},\mathbf{b},s,\mathbf{a})
=
Q^{\boldsymbol{\pi}}(\mathbf{h},s,\mathbf{a}).
\]
Then,
\[
\begin{aligned}
V^{\boldsymbol{\pi}}(\mathbf{h})
&=
\mathbb{E}_{\mathbf{a}\sim \boldsymbol{\pi}(\cdot\mid \mathbf{h})}
\left[
Q^{\boldsymbol{\pi}}(\mathbf{h},\mathbf{a})
\right] \\
&=
\mathbb{E}_{\mathbf{a}\sim \boldsymbol{\pi}(\cdot\mid \mathbf{h})}
\left[
\mathbb{E}_{s\mid \mathbf{h}}
\left[
Q^{\boldsymbol{\pi}}(\mathbf{h},s,\mathbf{a})
\right]
\right] \\
&=
\mathbb{E}_{s\mid \mathbf{h}}
\left[
\mathbb{E}_{\mathbf{a}\sim \boldsymbol{\pi}(\cdot\mid \mathbf{h})}
\left[
Q^{\boldsymbol{\pi}}(\mathbf{h},\mathbf{b},s,\mathbf{a})
\right]
\right] \\
&=
\mathbb{E}_{s\mid \mathbf{h}}
\left[
V^{\boldsymbol{\pi}}(\mathbf{h},\mathbf{b},s)
\right].
\end{aligned}
\]
Thus, when $s$ is sampled from the posterior state distribution
$p(s\mid \mathbf{h})$, the estimator
$V^{\boldsymbol{\pi}}(\mathbf{h},\mathbf{b},s)$ has expectation
$V^{\boldsymbol{\pi}}(\mathbf{h})$. Therefore,
$V^{\boldsymbol{\pi}}(\mathbf{h},\mathbf{b},s)$ is an unbiased estimator of
$V^{\boldsymbol{\pi}}(\mathbf{h})$.
\end{proof}

\section{Collaborative Navigation Benchmark} \label{app:expset}

\subsection{Scenes} 
The bird's-eye-view of the scenes we used are shown in Figure~\ref{fig:illustraion_scenes}. Table~\ref{tab:scene_stats} shows the \texttt{Matterport3D} scene IDs, the number of navigable points, and the total navigable area of each scene. 

\begin{figure}[ht]
    \centering
    \begin{subfigure}[b]{0.24\textwidth}
        \centering
        \includegraphics[width=\textwidth]{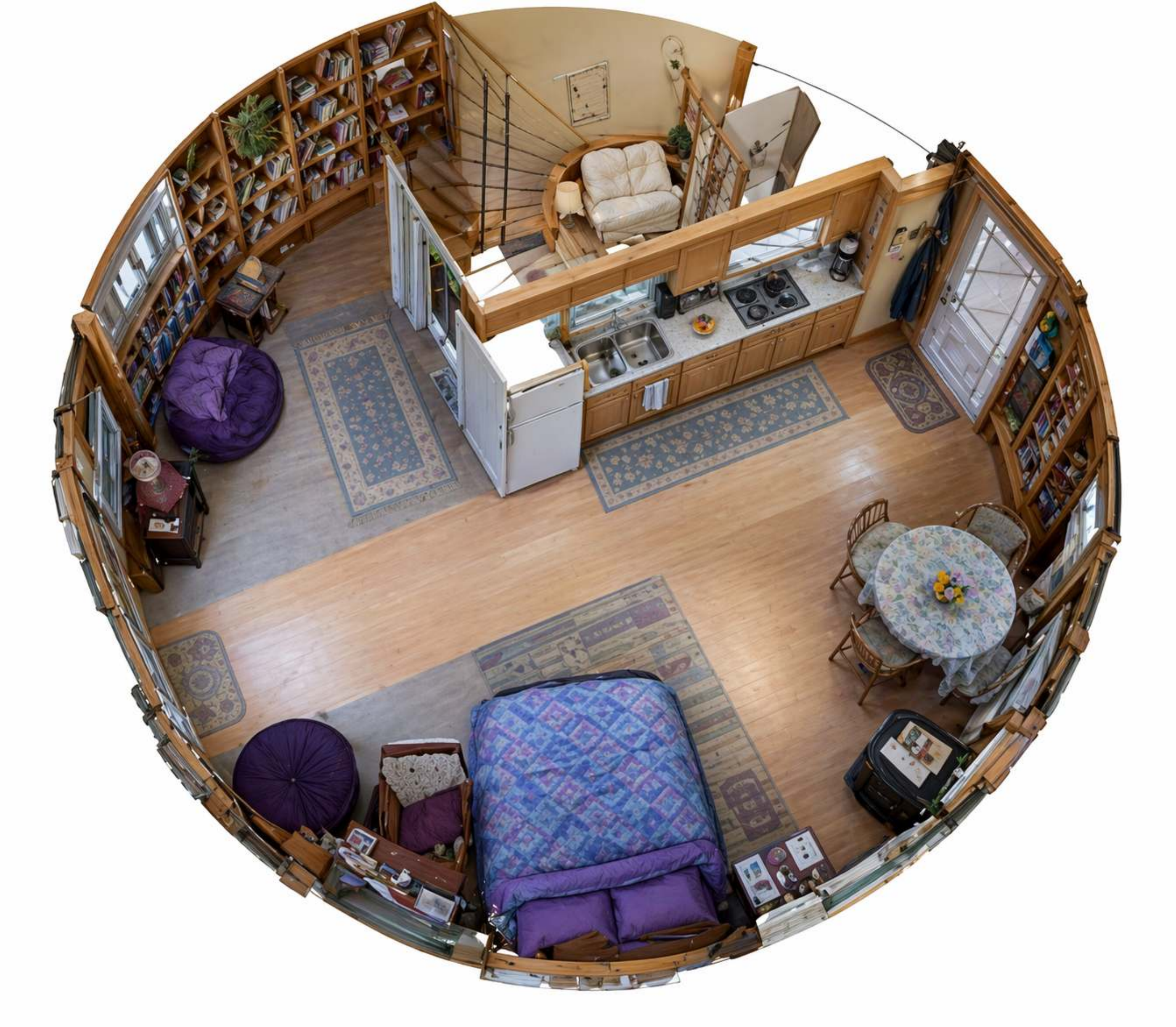}
        \caption{\footnotesize \;\texttt{Studio} $\mid$ \texttt{Picture}}
        \label{subfig:studio-fig}
    \end{subfigure}
    \hfill
    \begin{subfigure}[b]{0.29\textwidth}
        \centering
        \includegraphics[width=\textwidth]{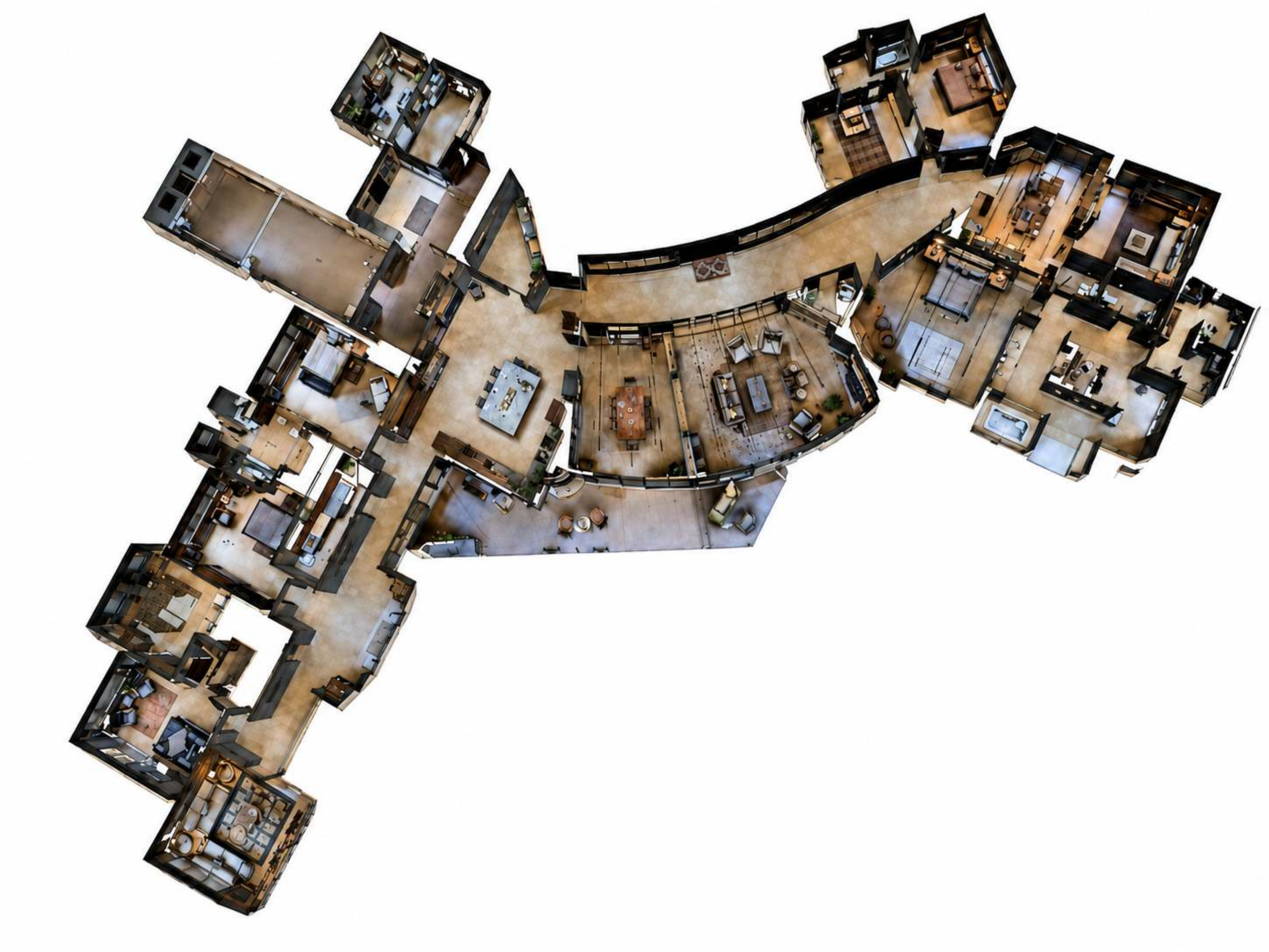}
        \caption{\footnotesize \;\texttt{Corridor} $\mid$ \texttt{Sink}}
        \label{subfig:apt-fig}
    \end{subfigure}
    \hfill
    \begin{subfigure}[b]{0.32\textwidth}
        \centering
        \includegraphics[width=\textwidth]{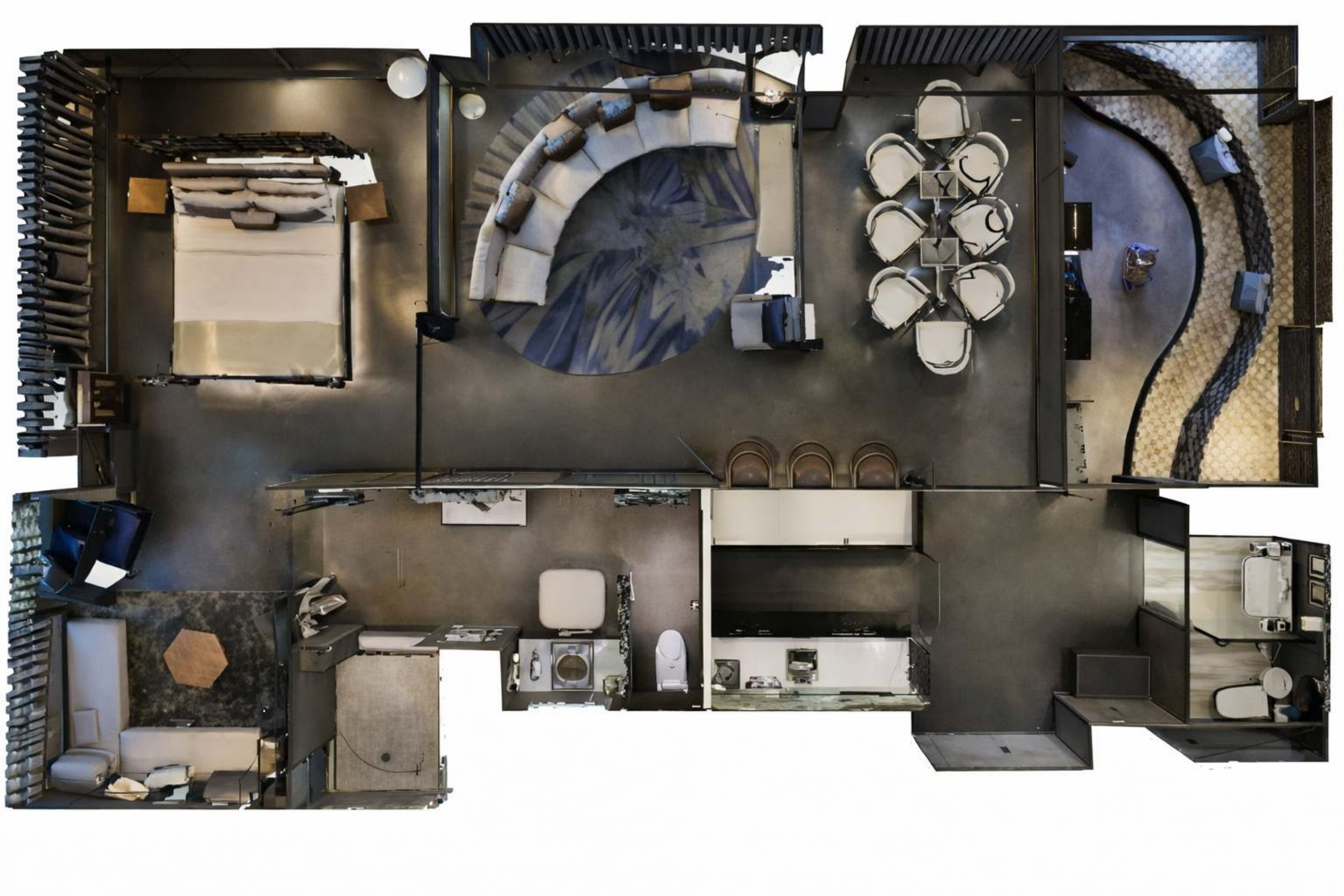}
        \caption{\footnotesize \,\texttt{Apartment} $\mid$ \texttt{Bed \& Counter}}
        \label{subfig:ranch-fig}
    \end{subfigure}

    \vspace{0.3cm}
    \hspace{1.5cm}
    \begin{subfigure}[b]{0.32\textwidth}
        \centering
        \includegraphics[width=\textwidth]{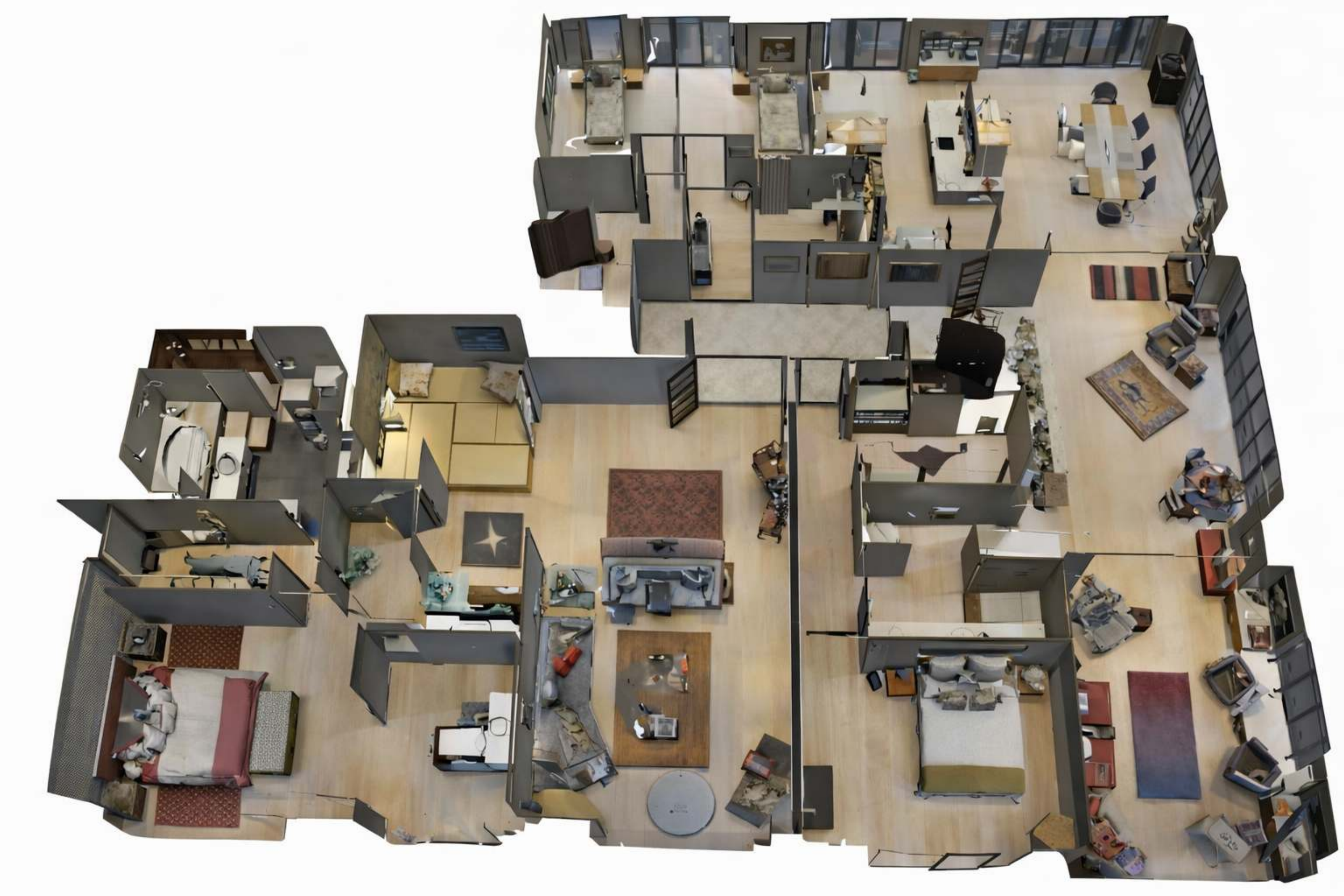}
        \caption{\footnotesize \;\;\texttt{Ranch} $\mid$ \texttt{Picture \& Table}}
        \label{subfig:corridor-fig}
    \end{subfigure}
    \hfill
    \begin{subfigure}[b]{0.34\textwidth}
        \centering
        \includegraphics[width=\textwidth]{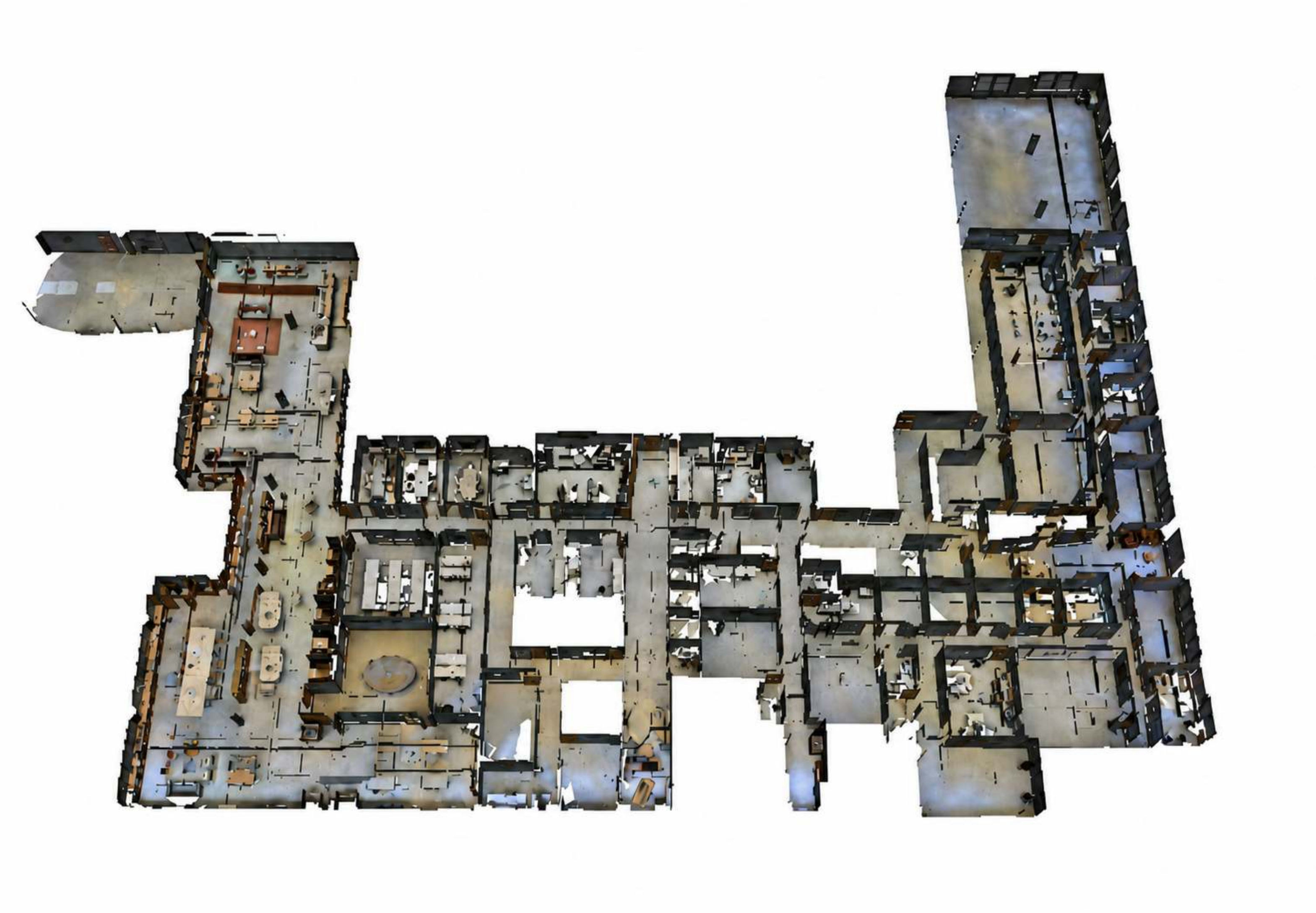}
        \caption{\footnotesize \texttt{Maze} $\mid$  \texttt{Drawer \& Table \& Chair}}
        \label{subfig:church-fig}
    \end{subfigure}
    \hspace{1.5cm}
    \caption{\vspace{2mm} Bird-eye's-views of \texttt{MatterPort3D} scenes.}
    \label{fig:illustraion_scenes}
\end{figure}

\subsection{Dataset Construction Details}
To construct datasets, we use base episodes in \cite{chen2020soundspaces} according to the desired \texttt{Matterport3D} scenes. We chose 1, 2, or 3 targets from different categories from the initial datasets. Episodes are filtered out if the initial distances between agents, or between agents and target objects are below or exceed predefined thresholds. Table~\ref{tab:dataset_root_params} summarizes the dataset construction parameters.

\begin{table}[htb]
\centering
\footnotesize
\captionsetup{font=footnotesize, skip=2pt}
\caption{Statistics of the five scenes used in our collaborative navigation dataset.\vspace{2mm}}
\label{tab:scene_stats}

\renewcommand{\arraystretch}{1.1}
\begin{tabular}{p{2.2cm} p{3.2cm} p{2.5cm} p{2.8cm}}
\toprule
\textbf{Scene} & \textbf{ID in \texttt{Matterport3D}} & \textbf{Navigable points} & \textbf{Navigable area (m$^2$)} \\
\midrule
\texttt{Studio}    & \texttt{GdvgFV5R1Z5} & 20   & 20.49   \\
\texttt{Corridor}  & \texttt{ac26ZMwG7aT} & 619  & 369.33  \\
\texttt{Apartment} & \texttt{17DRP5sb8fy} & 83   & 52.04   \\
\texttt{Ranch}     & \texttt{JeFG25nYj2p} & 193  & 166.22  \\
\texttt{Maze}     & \texttt{B6ByNegPMKs} & 1603 & 1348.31 \\
\bottomrule
\end{tabular}
\end{table}

Some \texttt{Matterport3D} scenes contain disconnected navigable regions. To create diverse but valid episodes, we filter the generated tasks using both minimum- and maximum-distance constraints. The minimum-distance constraint avoids overly easy episodes where agents start too close to the targets, while the maximum-distance constraint removes episodes in which some targets are unreachable. We apply these constraints in \texttt{Corridor} and \texttt{Maze}, where disconnected regions occur more frequently. As a result, although \texttt{Corridor} contains a high number of navigable areas, a horizon of 150 is enough for all  \texttt{Corridor} episodes.

\begin{table}[!htbp]
\centering
\footnotesize
\captionsetup{font=footnotesize, skip=2pt}
\caption{Dataset construction parameters for the five chosen \texttt{MatterPort3D} scenes. A single-object dataset keeps one target per episode, while the multi-object dataset combines multiple targets.\vspace{2mm}}
\label{tab:dataset_root_params}

\renewcommand{\arraystretch}{1.1}
\begin{tabular}{p{1.6cm} p{4.5cm} p{1.5cm} p{1.8cm} p{1.5cm}}
\toprule
\textbf{Scene} & \textbf{Targets} & \textbf{Target dist.} & \textbf{Start-goal dist.} & \textbf{Train eps.} \\
\midrule
\texttt{Studio}
& \texttt{picture}
& -
& $\geq 2.0 \,\mathrm{m}$
& 220 \\

\texttt{Corridor}
& \texttt{sink}
& -
& $2$-$5 \,\mathrm{m}$
& 218 \\

\texttt{Apartment}
& \texttt{bed}, \texttt{counter}
& $\geq 2.0 \,\mathrm{m}$
& $\geq 4.0 \,\mathrm{m}$
& 230 \\

\texttt{Ranch}
& \texttt{picture}, \texttt{table}
& $\geq 2.0 \,\mathrm{m}$
& $\geq 4.0 \,\mathrm{m}$
& 228 \\

\texttt{Maze}
& \texttt{chair}, \texttt{table}, \texttt{chest\_of\_drawers}
& $3$-$10 \,\mathrm{m}$ 
& $3$-$10 \,\mathrm{m}$
& 252 \\
\bottomrule
\end{tabular}
\end{table}

Two agents are initialized at different starting positions. For multi-object datasets, an episode is considered successful if each target is reached within a distance threshold of $1\,\mathrm{m}$. For two-object datasets, agents must start at least $1.5\,\mathrm{m}$ apart. For \texttt{Corridor}, the initial distance lie between $2.0$ and $5.0\,\mathrm{m}$. For \texttt{Maze}, the initial distance lie between $3.0$ and $10.0\,\mathrm{m}$. We split each dataset into training and validation sets with $3{:}1$.

\subsection{Acoustic Simulation} 
We use shared material configurations across all evaluated scenes to simulate the physical acoustics of each space. Each material contains frequency-dependent absorption, scattering, and transmission coefficients that simulate the sound propagation. The configuration here only affects acoustic rendering and does not change navigable points or the navigable area. Table~\ref{tab:acoustic_materials} shows mappings between semantics and materials, with the corresponding coefficient ranges.

\begin{table}[!htbp]
\centering
\footnotesize
\captionsetup{font=scriptsize, skip=2pt}
\caption{Representative acoustic material configurations for audio rendering.\vspace{2mm}}
\label{tab:acoustic_materials}

\renewcommand{\arraystretch}{1.2}
\setlength{\tabcolsep}{3pt}
\begin{tabular}{p{2.3cm} p{3.8cm} p{1.7cm} p{1.7cm} p{1.9cm}}
\toprule
\textbf{Acoustic material} & \textbf{Example semantic labels} & \textbf{Absorption} & \textbf{Scattering} & \textbf{Transmission} \\
\midrule
Acoustic Tile 
& \texttt{ceiling} 
& 0.50-0.70 
& 0.10-0.30 
& 0.002-0.050 \\

Gypsum Board 
& \texttt{wall} 
& 0.04-0.29 
& 0.10-0.15 
& 0.001-0.035 \\

Carpet 
& \texttt{floor}, \texttt{mat} 
& 0.01-0.65 
& 0.10-0.45 
& 0.001-0.008 \\

Glass 
& \texttt{window}, \texttt{mirror}, \texttt{tv\_monitor} 
& 0.05-0.35 
& 0.05-0.05 
& 0.022-0.125 \\

Foliage 
& \texttt{plant}, \texttt{indoor-plant} 
& 0.03-0.31 
& 0.20-0.80 
& 0.30-0.90 \\

Steel 
& \texttt{sink}, \texttt{microwave}, \texttt{railing} 
& 0.02-0.10 
& 0.10-0.10 
& 0.056-0.250 \\

wood, Thick 
& \texttt{chair}, \texttt{table}, \texttt{counter} 
& 0.05-0.19 
& 0.10-0.15 
& 0.001-0.035 \\

Wood Floor 
& \texttt{cabinet}, \texttt{stair} 
& 0.06-0.15 
& 0.10-0.15 
& 0.002-0.071 \\

Curtain 
& \texttt{bed}, \texttt{blanket}, \texttt{cushion}, \texttt{sofa} 
& 0.07-0.75 
& 0.10-0.50 
& 0.045-0.420 \\

Default 
& \texttt{default} 
& 0.10-0.10 
& 0.50-0.50 
& 0.000-0.000 \\
\bottomrule
\end{tabular}
\end{table}

\newpage
We choose the following sounds from \cite{chen2020soundspaces} for constructing our dataset: 
{\setstretch{1.5}
\begin{itemize}
    \item Dragging Chair \, \includegraphics[height=1em,width=6em]{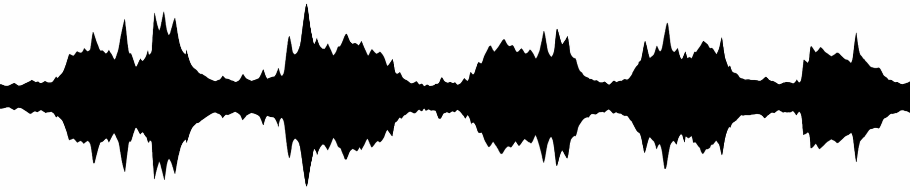}
    \item Table with Silverware Dropping \, \includegraphics[height=1em,width=6em]{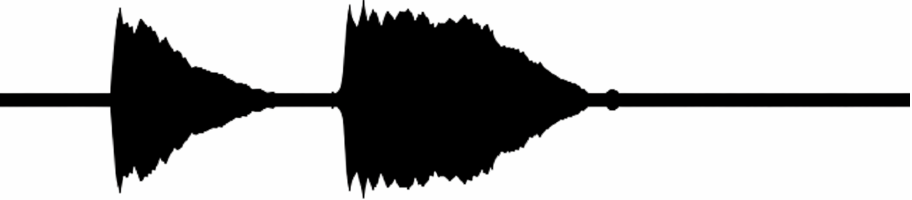}
    \item Picture with Camera Shutter \, \includegraphics[height=1.5em,width=6em]{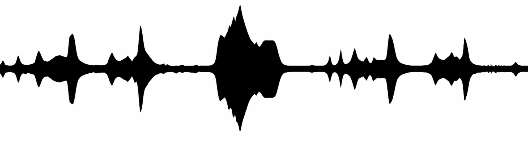}
    \item Sink with Dripping Water \, \includegraphics[height=0.9em,width=6em]{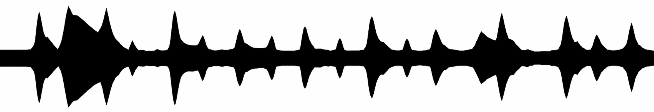}
    \item Counter with Coin Drop\, \includegraphics[height=0.9em,width=6em]{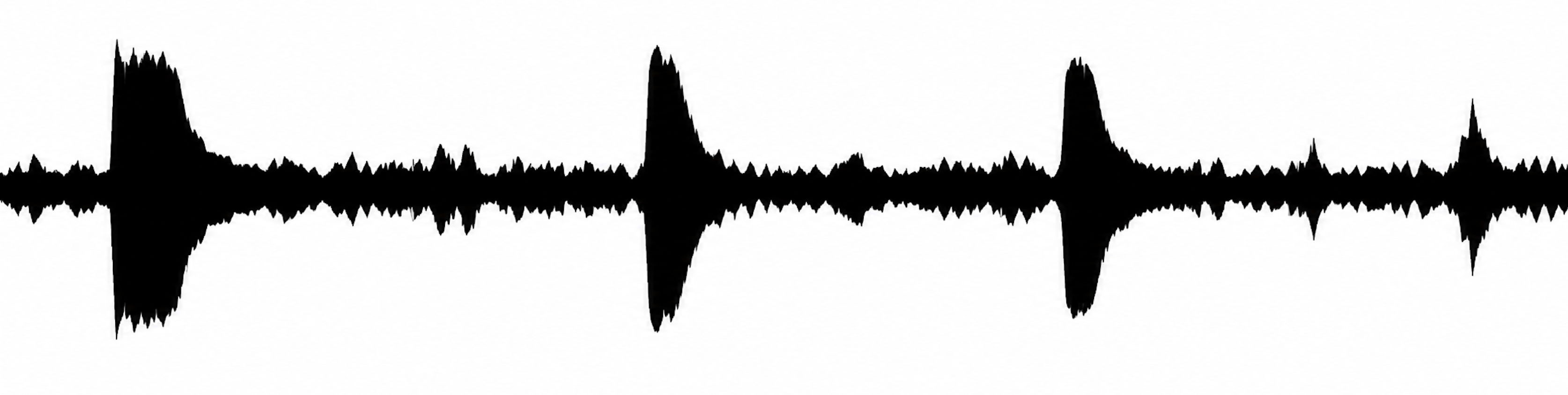}
    \item Pulling Chest of Drawers \, \includegraphics[height=0.9em,width=6em]{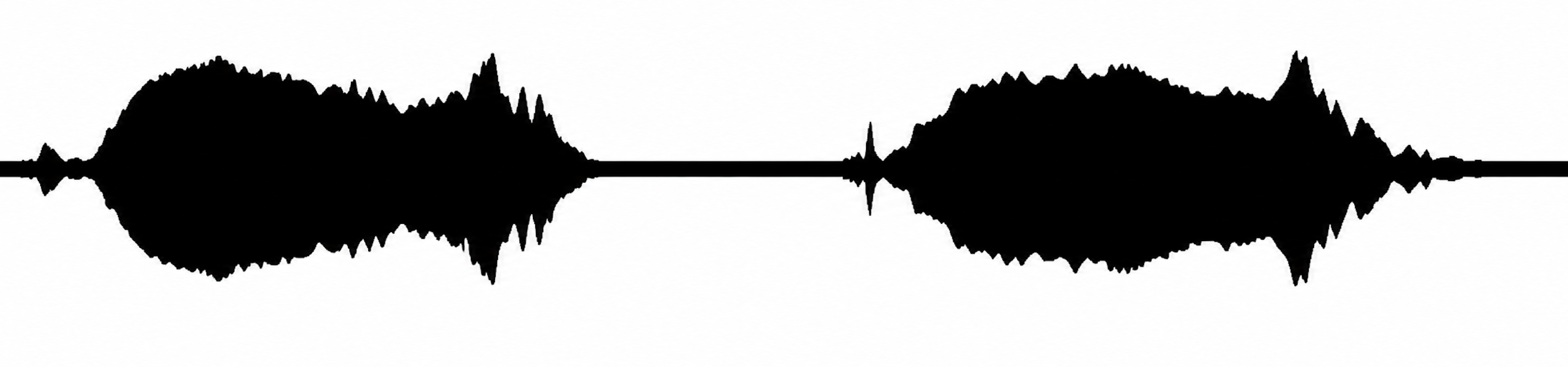}
    \item Creaking Bed \, \includegraphics[height=0.9em,width=6em]{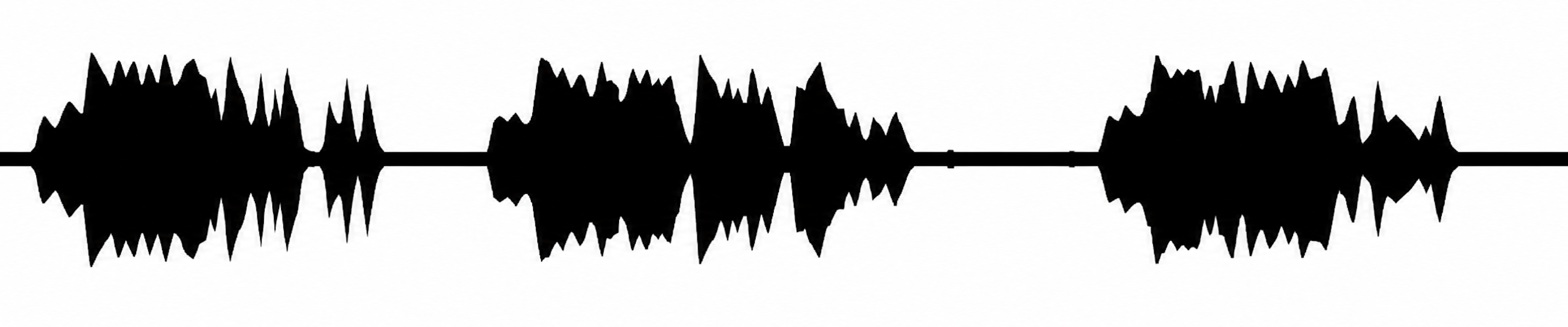}
\end{itemize}
}
Figure~\ref{fig:source_spectrograms} shows the corresponding source spectrograms.

\begin{figure}[!htbp]
    \centering
    \captionsetup{font=footnotesize}
    \subcaptionsetup{font=scriptsize}

    \begin{subfigure}[b]{0.23\textwidth}
        \centering
        \includegraphics[width=\textwidth]{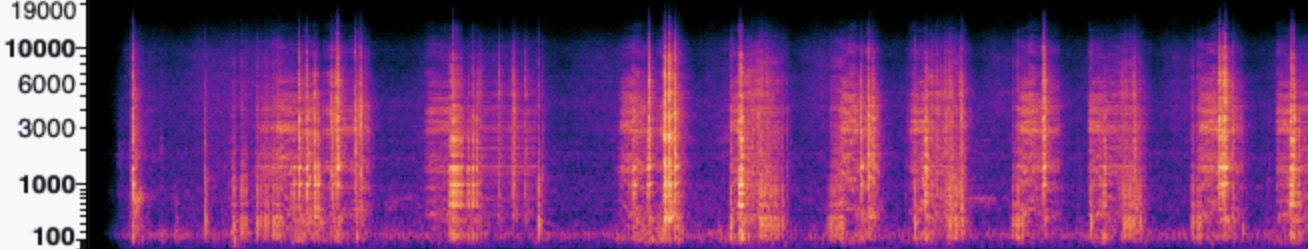}
        \caption{\texttt{Dragging Chair}}
        \label{subfig:spec-chair}
    \end{subfigure}
    \hfill
    \begin{subfigure}[b]{0.23\textwidth}
        \centering
        \includegraphics[width=\textwidth]{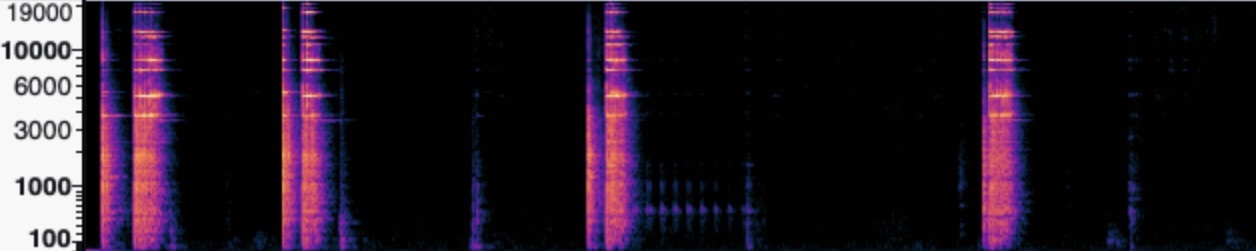}
        \caption{\texttt{Table with Silverware}}
        \label{subfig:spec-table}
    \end{subfigure}
    \hfill
    \begin{subfigure}[b]{0.23\textwidth}
        \centering
        \includegraphics[width=\textwidth]{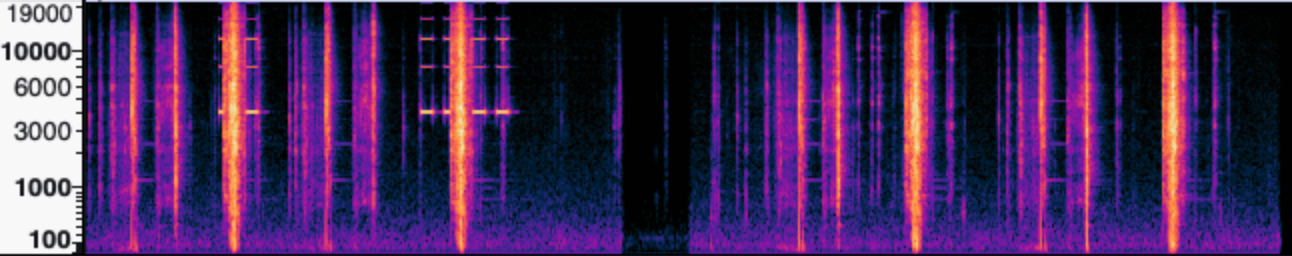}
        \caption{\texttt{Picture Shutter}}
        \label{subfig:spec-picture}
    \end{subfigure}
    \hfill
    \begin{subfigure}[b]{0.23\textwidth}
        \centering
        \includegraphics[width=\textwidth]{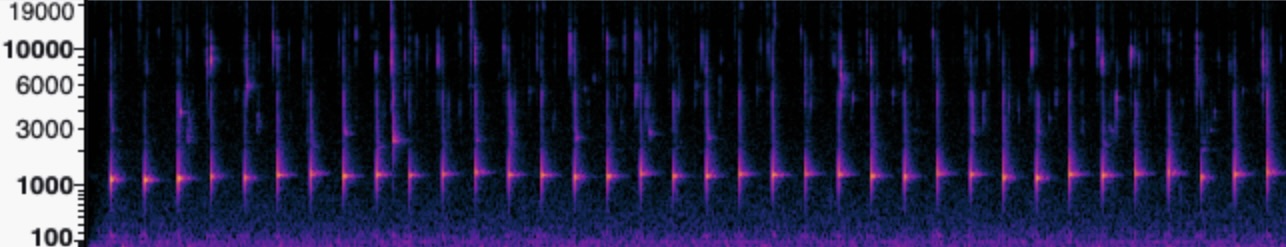}
        \caption{\texttt{Sink Dripping}}
        \label{subfig:spec-sink}
    \end{subfigure}

    \vspace{2mm}
    \hspace{12mm}
    \begin{subfigure}[b]{0.23\textwidth}
        \centering
        \includegraphics[width=\textwidth]{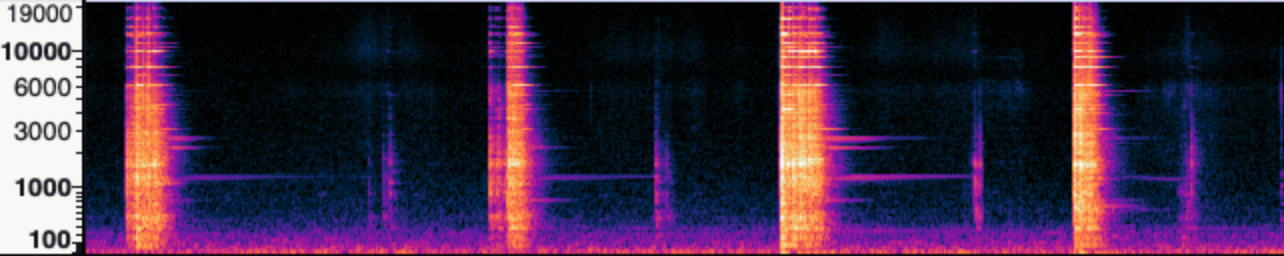}
        \caption{\texttt{Coin Drop on Counter}}
        \label{subfig:spec-counter}
    \end{subfigure}
    \hfill
    \begin{subfigure}[b]{0.23\textwidth}
        \centering
        \includegraphics[width=\textwidth]{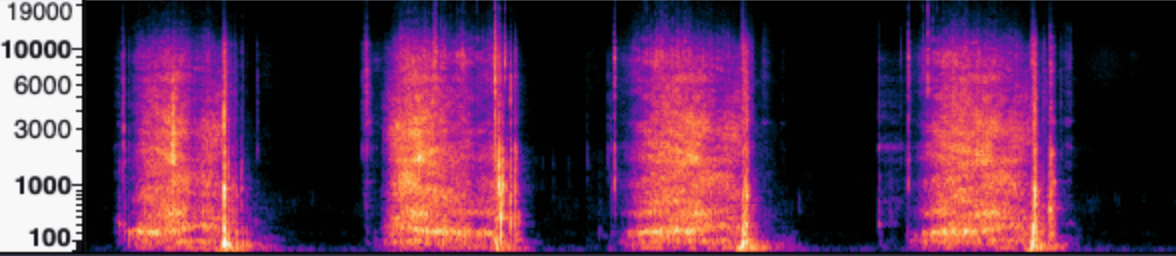}
        \caption{\texttt{Chest of Drawers}}
        \label{subfig:spec-chest}
    \end{subfigure}
    \hfill
    \begin{subfigure}[b]{0.23\textwidth}
        \centering
        \includegraphics[width=\textwidth]{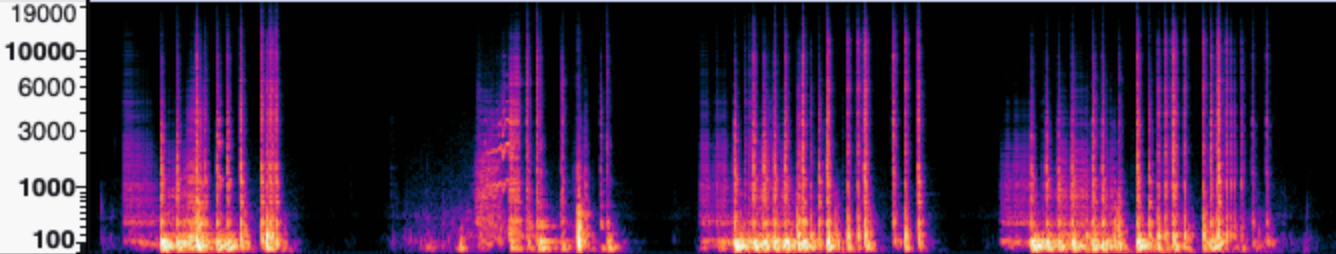}
        \caption{\texttt{Creaking Bed}}
        \label{subfig:spec-bed}
    \end{subfigure}
    \hspace{12mm}

    \caption{Spectrograms of selected sounds above.}
    \label{fig:source_spectrograms}
\end{figure}

Agents learn more effectively from short, clean, well-isolated sounds. Such sounds usually exhibit a sharp attack, little or no sustain, and a rapid decay with minimal trailing energy. For example, \texttt{Sink} and \texttt{Table} clearly follow this structure, producing consistent acoustic patterns that provide reliable cues for target localization.

\begin{figure}[htb]
    \centering
    \begin{subfigure}[b]{0.24\textwidth}
        \centering
        \includegraphics[width=\textwidth]{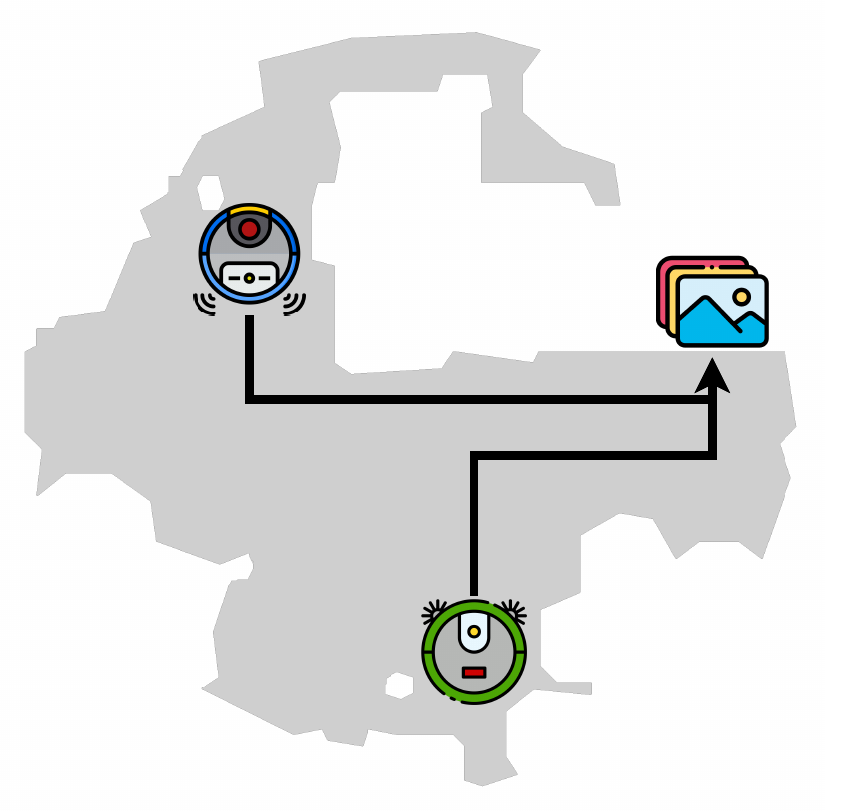}
        \caption{\footnotesize \;\texttt{Studio} $\mid$ \texttt{Picture}}
        \label{subfig:studio-nav}
    \end{subfigure}
    \hfill
    \begin{subfigure}[b]{0.25\textwidth}
        \centering
        \includegraphics[width=\textwidth]{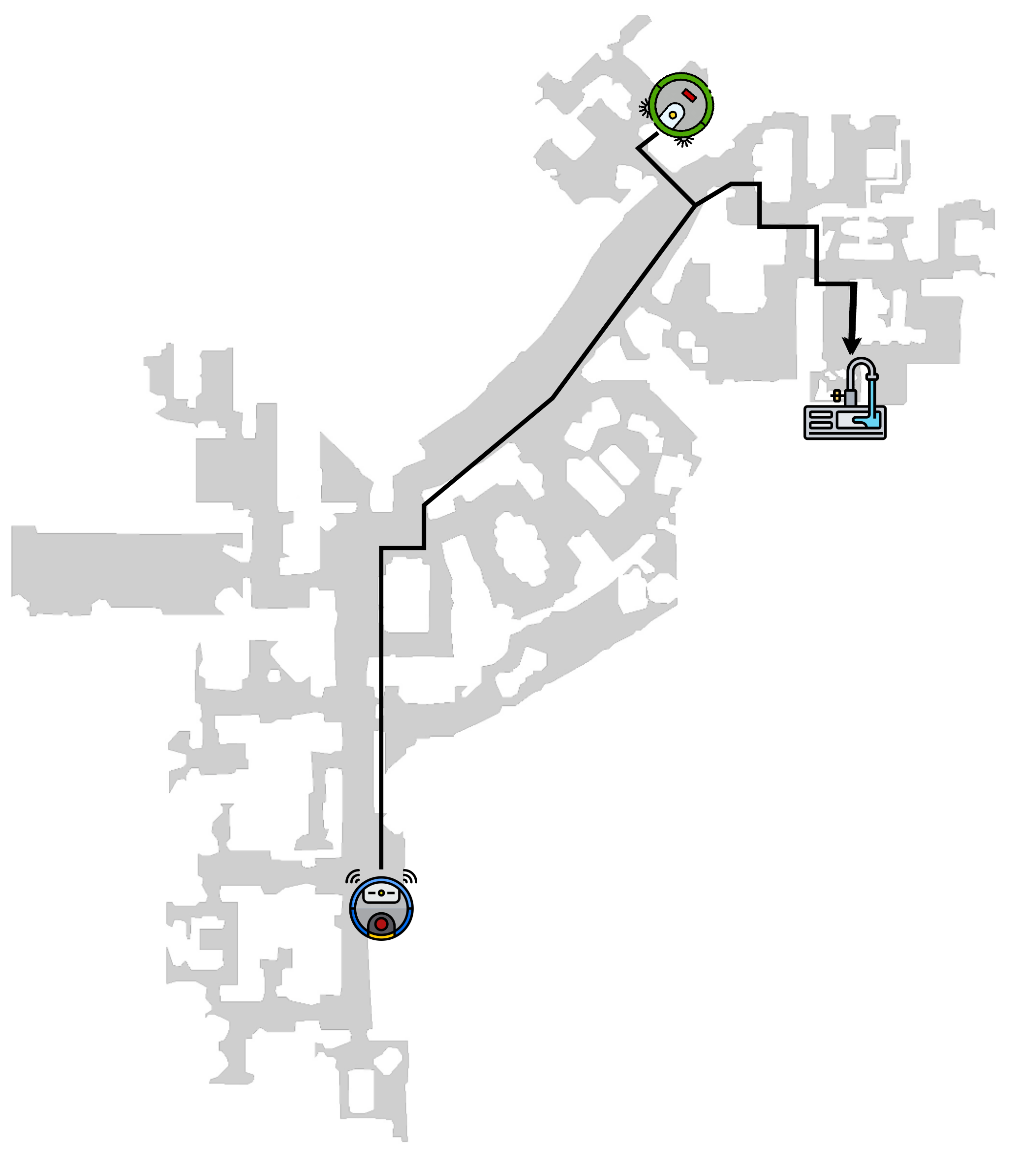}
        \caption{\footnotesize \;\texttt{Corridor} $\mid$ \texttt{Sink}}
        \label{subfig:apt-nav}
    \end{subfigure}
    \hfill
    \begin{subfigure}[b]{0.32\textwidth}
        \centering
        \includegraphics[width=\textwidth]{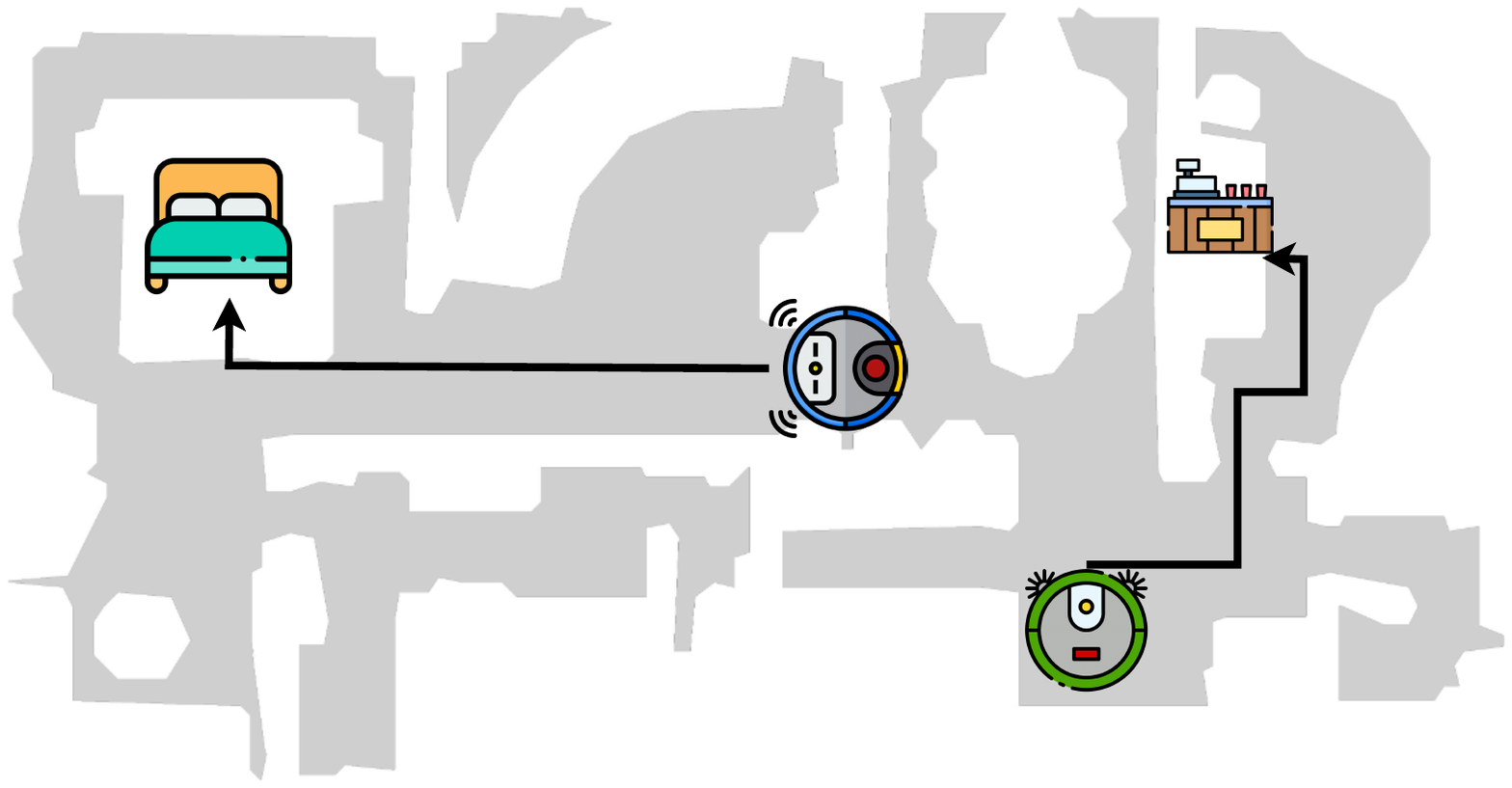}
        \caption{\footnotesize \,\texttt{Apartment} $\mid$ \texttt{Bed \& Counter}}
        \label{subfig:ranch-nav}
    \end{subfigure}

    \vspace{0.3cm}
    \hspace{1.5cm}
    \begin{subfigure}[b]{0.28\textwidth}
        \centering
        \includegraphics[width=\textwidth]{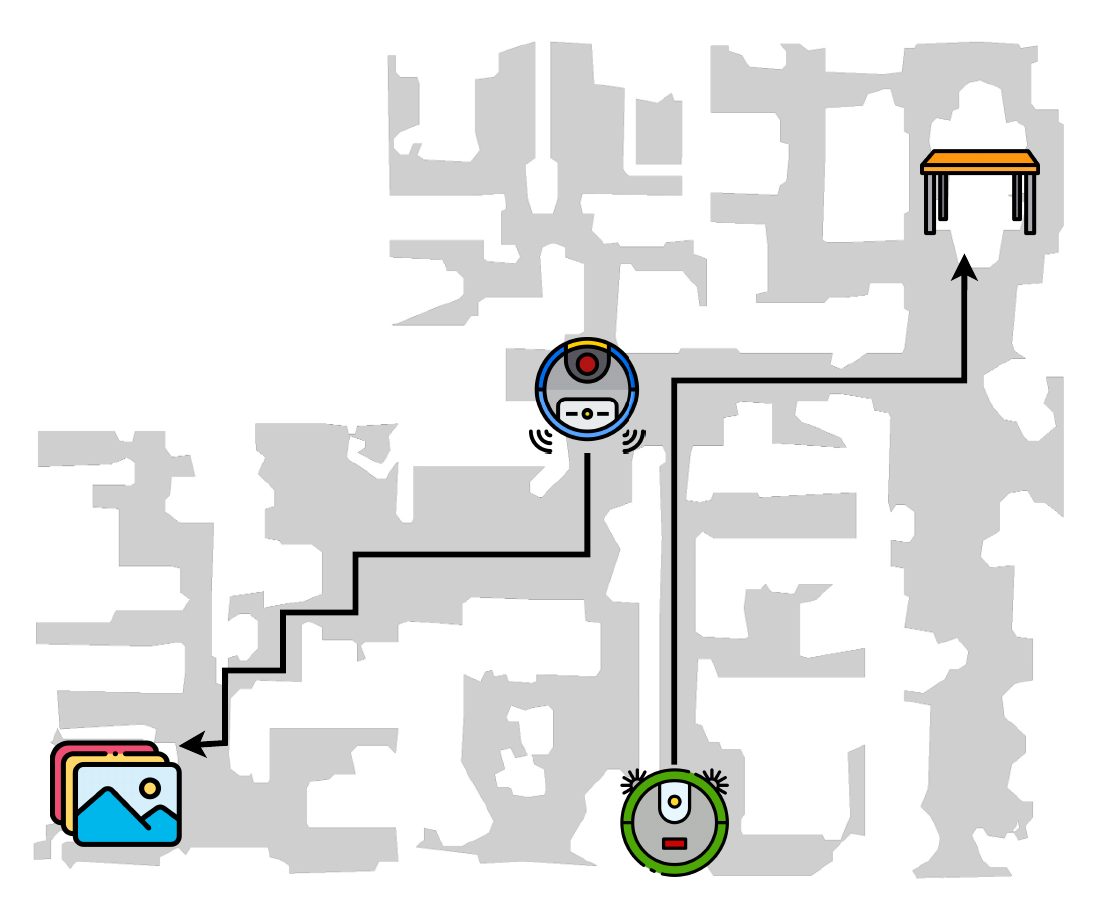}
        \caption{\footnotesize \;\;\texttt{Ranch} $\mid$ \texttt{Picture \& Table}}
        \label{subfig:corridor-nav}
    \end{subfigure}
    \hfill
    \begin{subfigure}[b]{0.32\textwidth}
        \centering
        \includegraphics[width=\textwidth]{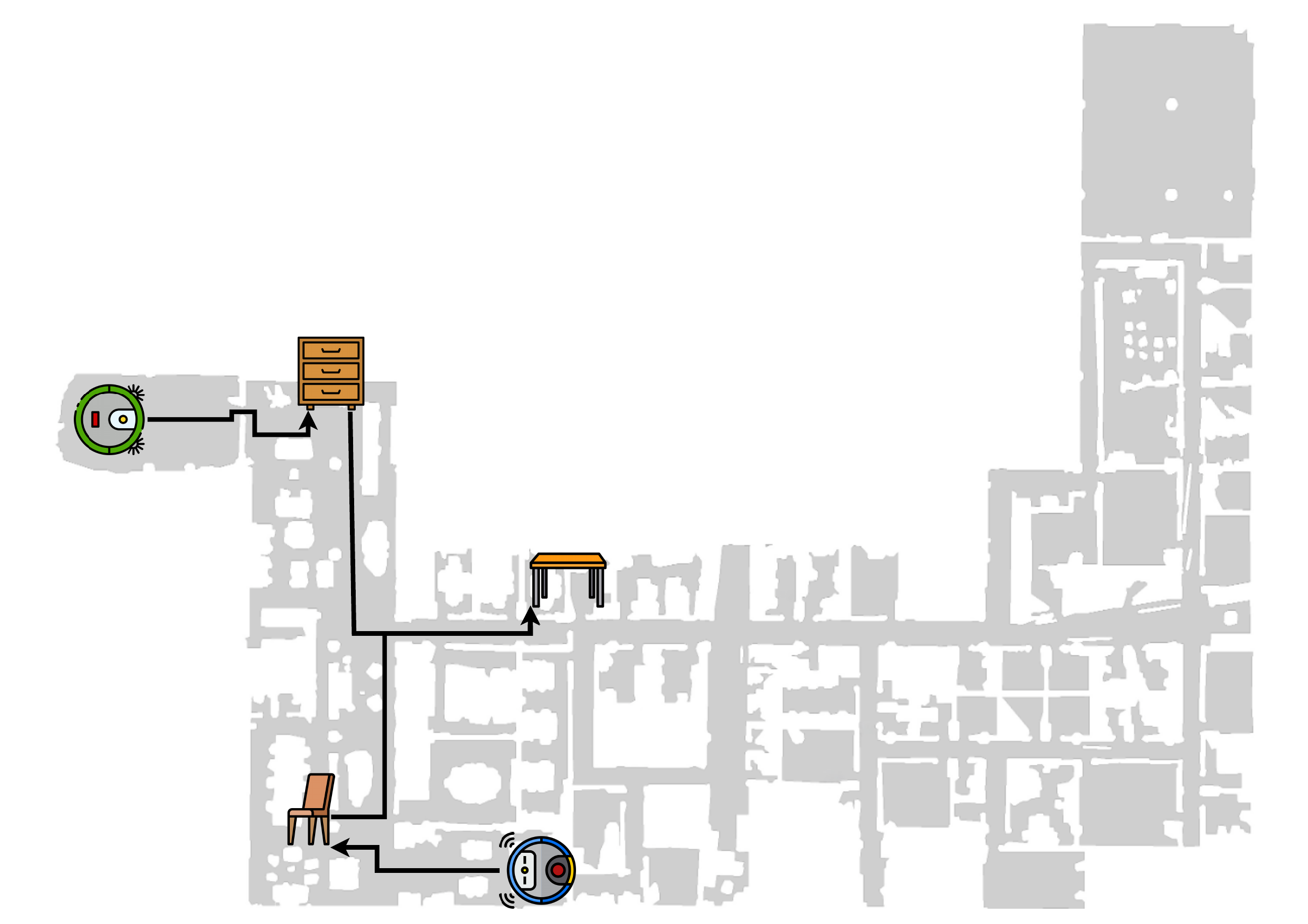}
        \caption{\footnotesize \texttt{Maze} $\mid$  \texttt{Drawer \& Table \& Chair}}
        \label{subfig:church-nav}
    \end{subfigure}
    \hspace{1.5cm} 
    \caption{Illustration of example episodes.} 
    \label{fig:illustraion_episodes}
\end{figure}

\subsection{Episodes} For each episode, objects will be selected based on the corresponding categories. Each episode strictly follows Table~\ref{tab:dataset_root_params}, the demonstrations of example episodes are in Figure~\ref{fig:illustraion_episodes}.

\FloatBarrier

\section{Experimental Settings}

\subsection{Hyperparameters} \label{subapp:hyperparameters}

We use the same hyperparameters across all scenes.
Table~\ref{tab:shared_hyperparameters} summarizes the training configurations.

\begin{table}[!htbp]
\centering
\footnotesize
\captionsetup{font=footnotesize, skip=2pt}
\caption{Shared training hyperparameters used across all evaluated scenes.\vspace{2mm}}
\label{tab:shared_hyperparameters}

\renewcommand{\arraystretch}{1.08}
\begin{tabular}{p{3.5cm} p{1.3cm} p{3.5cm} p{1.3cm}}
\toprule
\textbf{Hyperparameter} & \textbf{Value} & \textbf{Hyperparameter} & \textbf{Value} \\
\midrule
Optimizer & Adam
& Actor learning rate & 0.00025 \\

Critic learning rate & 0.0002
& Adam epsilon & 1e-5 \\

PPO epochs & 2
& PPO mini-batches & 1 \\

Rollout steps per update & 150
& PPO policy ratio $\rho$ clip & 0.2 \\

PPO value clip & 0.25
& Policy-value loss coefficient $\mu$ & 0.67 \\

Belief smooth coefficient $\alpha$ & 0.5
& Entropy coefficient $\beta$ & 0.05 \\

Discount factor $\gamma$ & 0.99
& GAE $\lambda$ & 0.95 \\

Max gradient norm & 0.2
& History cache size (steps) & 150 \\

Vision encoder hidden size & 128
& Audio encoder hidden size & 128 \\

Transformer hidden size & 256
& Language embedding size & 384 \\

Language encoder hidden size & 24
& Normalized advantage & False \\

Number of updates & 8000
& Max depth for depth sensor & 3 m \\
\bottomrule
\end{tabular}
\end{table}

\FloatBarrier

\subsection{Architecture} \label{subapp:architecture}

\paragraph{Agent Configuration}
Since many objects in \texttt{Matterport3D} scenes are large and visually salient, a single agent can effortlessly localize them without requiring collaboration. To increase the difficulty of tasks, we restrict visual observations to depth maps with a sensing range of $0$--$5\,\mathrm{m}$, a resolution of $16 \times 16$ pixels, and a horizontal field of view (HFoV) of $10^\circ$. We use \texttt{sentence-transformers/all-MiniLM-L6-v2} as the instruction encoder, a compact ResNet-18 as the visual encoder \cite{he2016deep}, and a plain CNN with 3 convolutional layers as the audio encoder. At each time step, each agent combines its previous observations and actions with the current observation using an MLP, then encodes them using $8$-head transformers to obtain a history representation.

\paragraph{Audio Encoder}
\begin{itemize}
    \item \textbf{Input:} Binaural spectrogram.
    \item \textbf{Layers:}
    \begin{itemize}
        \item $\mathrm{Conv2d}(2, 32, 5 \times 5, \mathrm{stride}=2) + \mathrm{ReLU}$
        \item $\mathrm{Conv2d}(32, 64, 3 \times 3, \mathrm{stride}=2) + \mathrm{ReLU}$
        \item $\mathrm{Conv2d}(64, 64, 3 \times 3, \mathrm{stride}=1)$
        \item Flatten
    \end{itemize}
\end{itemize}

\paragraph{Vision Encoder}
\begin{itemize}
    \item \textbf{Input:} Single-channel depth observation.

    \item \textbf{Layers:}
  \begin{itemize}
      \item $\mathrm{ResizeCenterCrop}(64 \times 64)$
      \item $\mathrm{Conv2d}(1, 16, 7 \times 7, \mathrm{stride}=1,
  \mathrm{padding}=3) + \mathrm{GroupNorm} + \mathrm{ReLU}$
      \item $2 \times \mathrm{ResidualBlock}(16 \rightarrow 16, \mathrm{stride}
  =1)$
      \item $2 \times \mathrm{ResidualBlock}(16 \rightarrow 32, \mathrm{stride}
  =2)$
      \item $2 \times \mathrm{ResidualBlock}(32 \rightarrow 64, \mathrm{stride}
  =2)$
      \item $2 \times \mathrm{ResidualBlock}(64 \rightarrow 128, \mathrm{stride}
  =2)$
      \item Flatten
      \item $\mathrm{Linear}(128 \times 8 \times 8, 64)$
  \end{itemize}

\end{itemize}

\paragraph{Auxiliary Belief Predictor}
  \begin{itemize}
      \item \textbf{Input:} Binaural spectrogram.

      \item \textbf{Layers:}
  \begin{itemize}
      \item $\mathrm{Conv2d}(2, 16, 7 \times 7, \mathrm{stride}=1,
  \mathrm{padding}=3) + \mathrm{GroupNorm} + \mathrm{ReLU}$
      \item $2 \times \mathrm{ResidualBlock}(16 \rightarrow 16, \mathrm{stride}
  =1)$
      \item $2 \times \mathrm{ResidualBlock}(16 \rightarrow 32, \mathrm{stride}
  =2)$
      \item $2 \times \mathrm{ResidualBlock}(32 \rightarrow 64, \mathrm{stride}
  =2)$
      \item $2 \times \mathrm{ResidualBlock}(64 \rightarrow 128, \mathrm{stride}
  =2)$
      \item Flatten
      \item $\mathrm{Linear}(4608, 2)$
  \end{itemize}

  \end{itemize}

\paragraph{History Encoder}
\begin{itemize}
    \item \textbf{Input:} Each agent's recent observation-action history.
    \item \textbf{Layers:}
  \begin{itemize}
      \item Previous action encoding: $\mathrm{Linear}(|\mathcal{A}|, 16)$
      \item Relative pose encoding: $\mathrm{Linear}(5, 16)$
      \item Feature fusion:
      $\mathrm{Linear}(d_{\mathrm{in}}, d_h) + \mathrm{ReLU} + \mathrm{Linear}
  (d_h, d_h)$
      \item Transformer encoder: $1$ layer with $8$ attention heads
      \item Transformer decoder: $1$ layer with $8$ attention heads
      \item Feed-forward dimension: $d_h$
      \item Activation: ReLU
  \end{itemize}
\end{itemize}

\paragraph{Language Encoder}
  \begin{itemize}
      \item \textbf{Input:} Tokenized target category or language instruction.

      \item \textbf{Layers:}
      \begin{itemize}
          \item WordPiece tokenization with truncation
          \item Token embeddings + position embeddings + segment embeddings
          \item $6 \times$ Transformer encoder layers
          \item Each layer uses $12$-head self-attention
          \item Feed-forward network:
          $\mathrm{Linear}(384, 1536) + \mathrm{GELU} + \mathrm{Linear}(1536,
  384)$
          \item Mean pooling over token embeddings using the attention mask
          \item $L_2$ normalization
      \end{itemize}
  \end{itemize}

\section{Additional Results} \label{app:addres} 

\subsection{Success Rates}
We provide the success rates for each scene in Figure~\ref{fig:additional_results}. In most scenes, the success curves align well with the return curves in Figure~\ref{fig:training}, leading to the same dominance patterns. This is expected because task success contributes the largest portion of the episode return. However, \texttt{Maze} exhibits a noticeable discrepancy: although AVLA-Collab achieves a much higher return than the other methods, its success rate does not improve to the same extent. This is because agents sometimes stop prematurely after reaching easier targets, which helps avoid large penalties from the distance-progress term but prevents them from completing all targets. Overall, the reward still provides useful optimization signals: it guides agents to navigate toward targets, reduce their distance to the goals, and stop within the target vicinity.

\begin{figure}[!htbp]
    \centering
    \captionsetup{font=footnotesize}
    \subcaptionsetup{font=footnotesize}
    
    \begin{subfigure}[b]{0.32\textwidth}
        \centering
        \includegraphics[width=\textwidth]{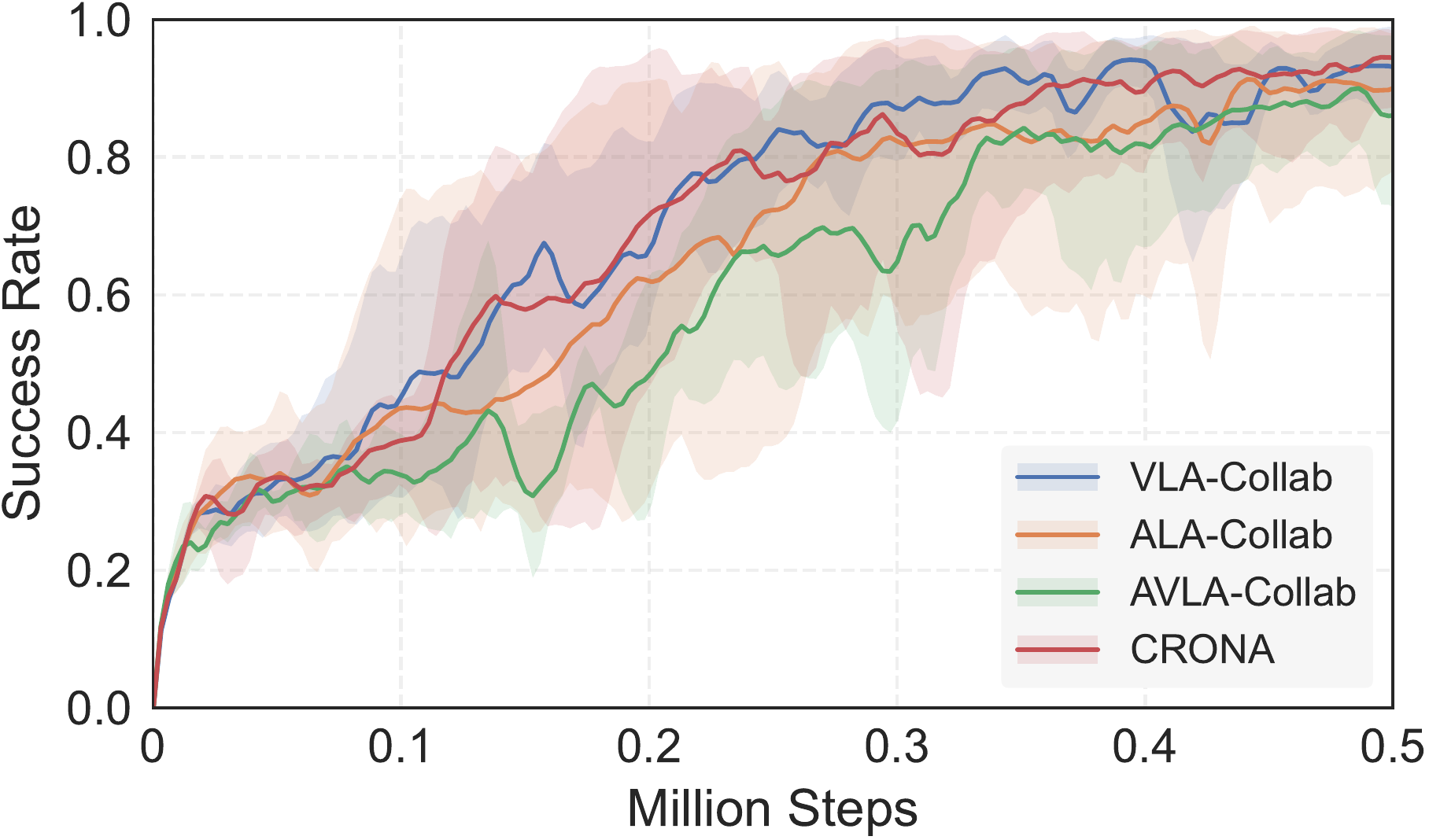}
        \caption{\;\texttt{Studio} $\mid$ \texttt{Picture}}
        \label{subfig:add-studio}
    \end{subfigure}
    \hfill
    \begin{subfigure}[b]{0.32\textwidth}
        \centering
        \includegraphics[width=\textwidth]{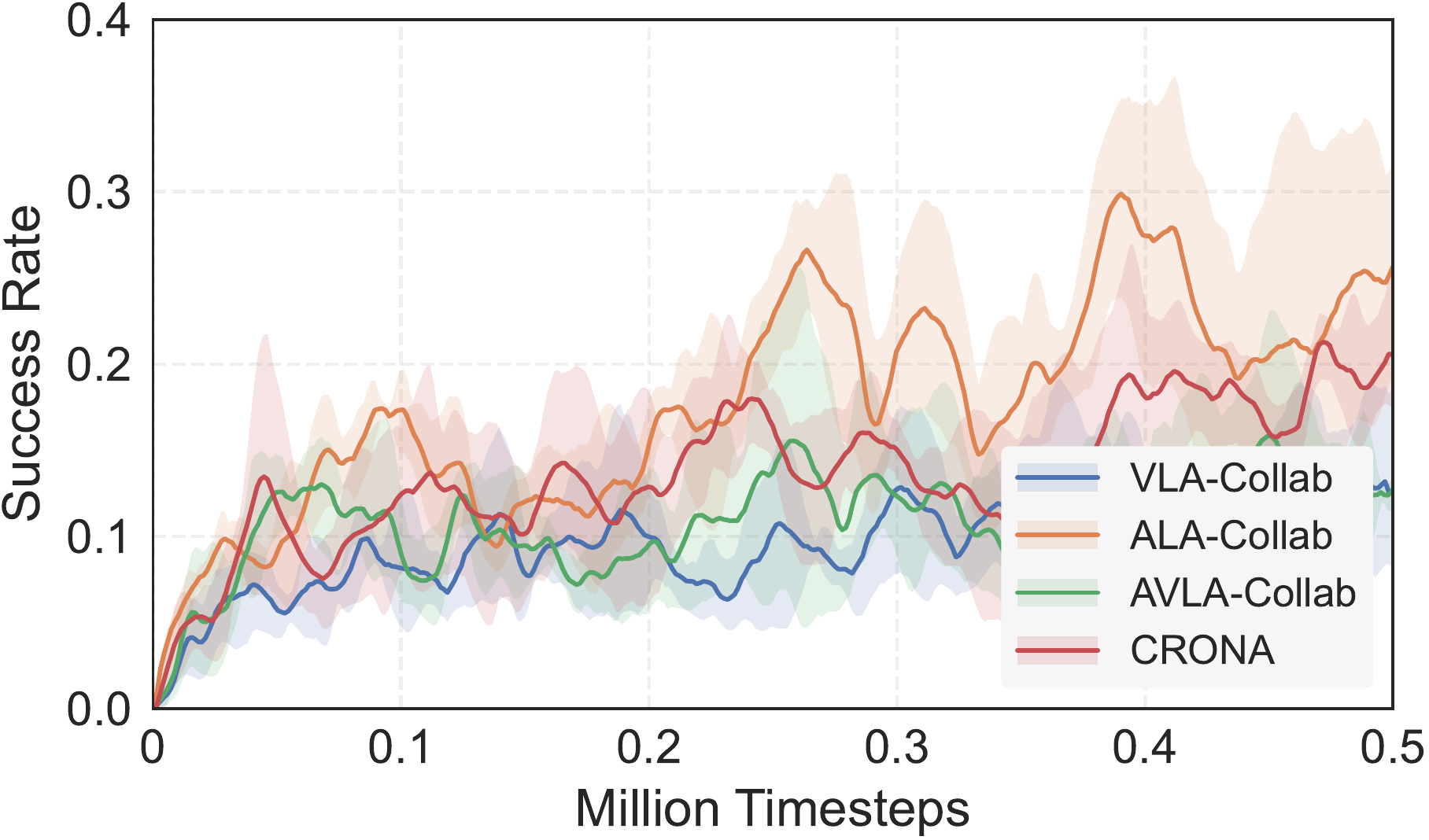}
        \caption{\;\texttt{Corridor} $\mid$ \texttt{Sink}}
        \label{subfig:add-corridor}
    \end{subfigure}
    \hfill
    \begin{subfigure}[b]{0.32\textwidth}
        \centering
        \includegraphics[width=\textwidth]{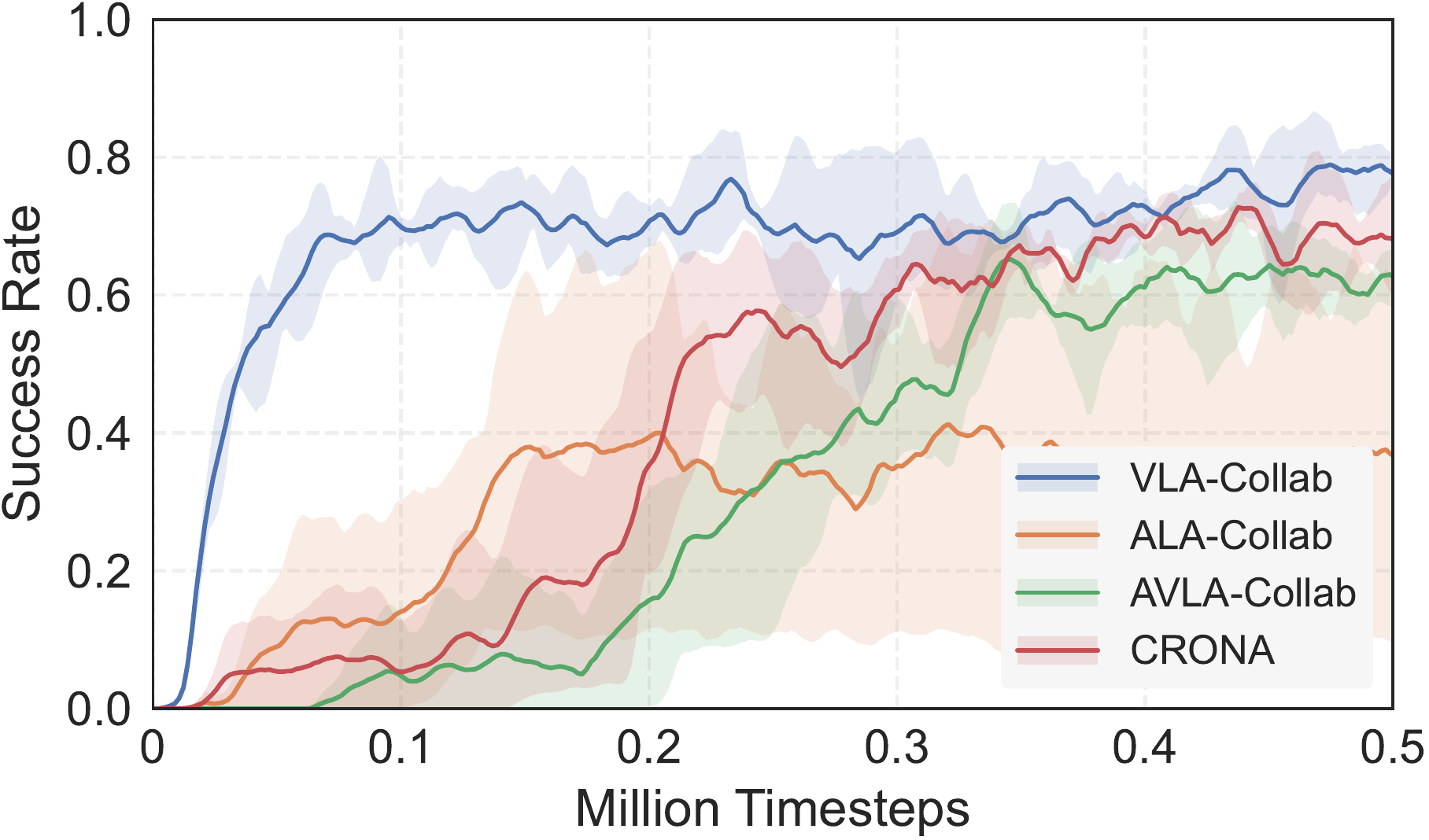}
        \caption{\,\texttt{Apartment} $\mid$ \texttt{Bed \& Counter}}
        \label{subfig:add-apt}
    \end{subfigure}

    \vspace{4mm}
    \hspace{2cm}
    \begin{subfigure}[b]{0.32\textwidth}
        \centering
        \includegraphics[width=\textwidth]{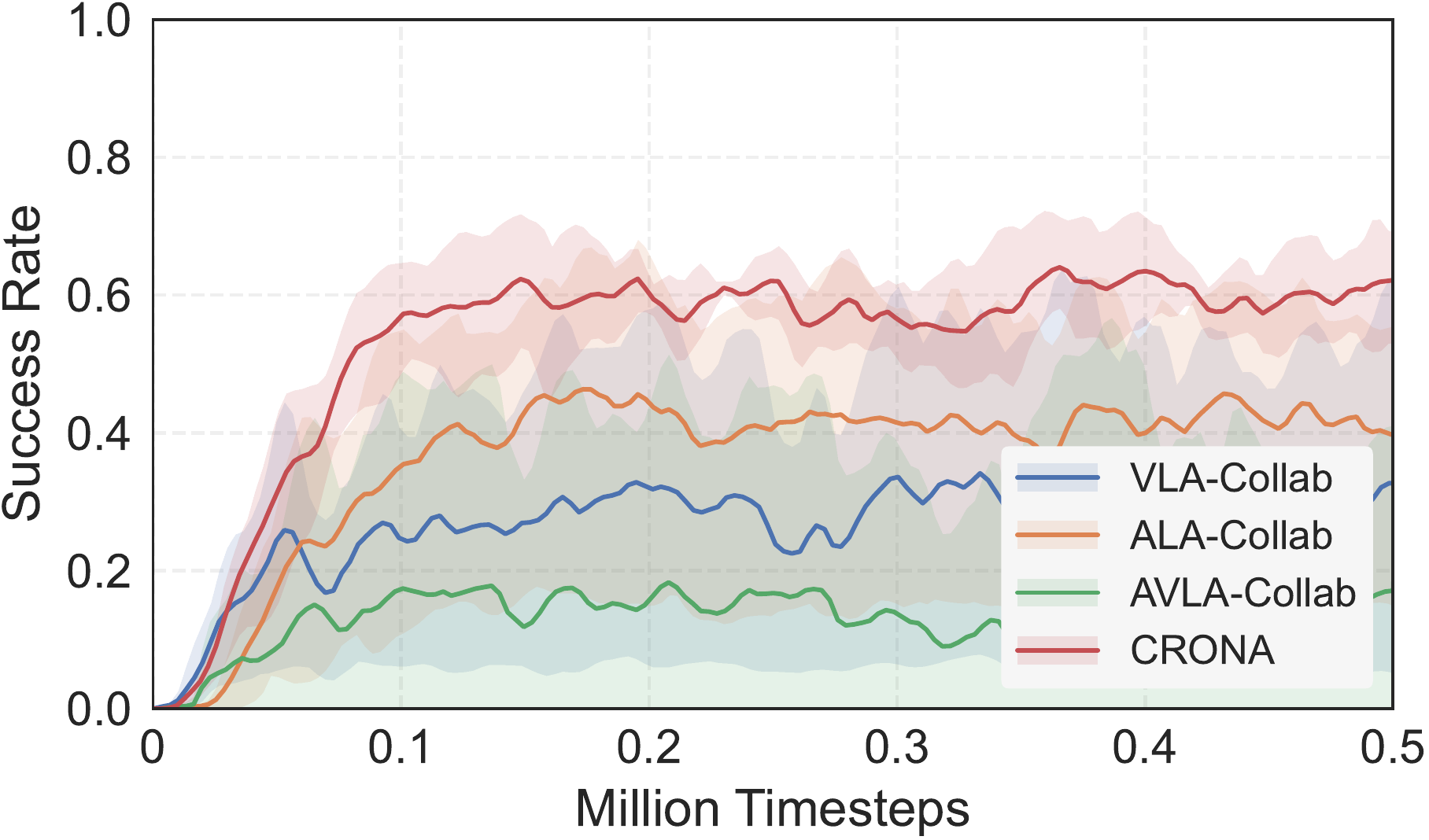}
        \caption{\;\;\texttt{Ranch} $\mid$ \texttt{Picture \& Table}}
        \label{subfig:add-ranch}
    \end{subfigure}
    \hfill
    \begin{subfigure}[b]{0.32\textwidth}
        \centering
        \includegraphics[width=\textwidth]{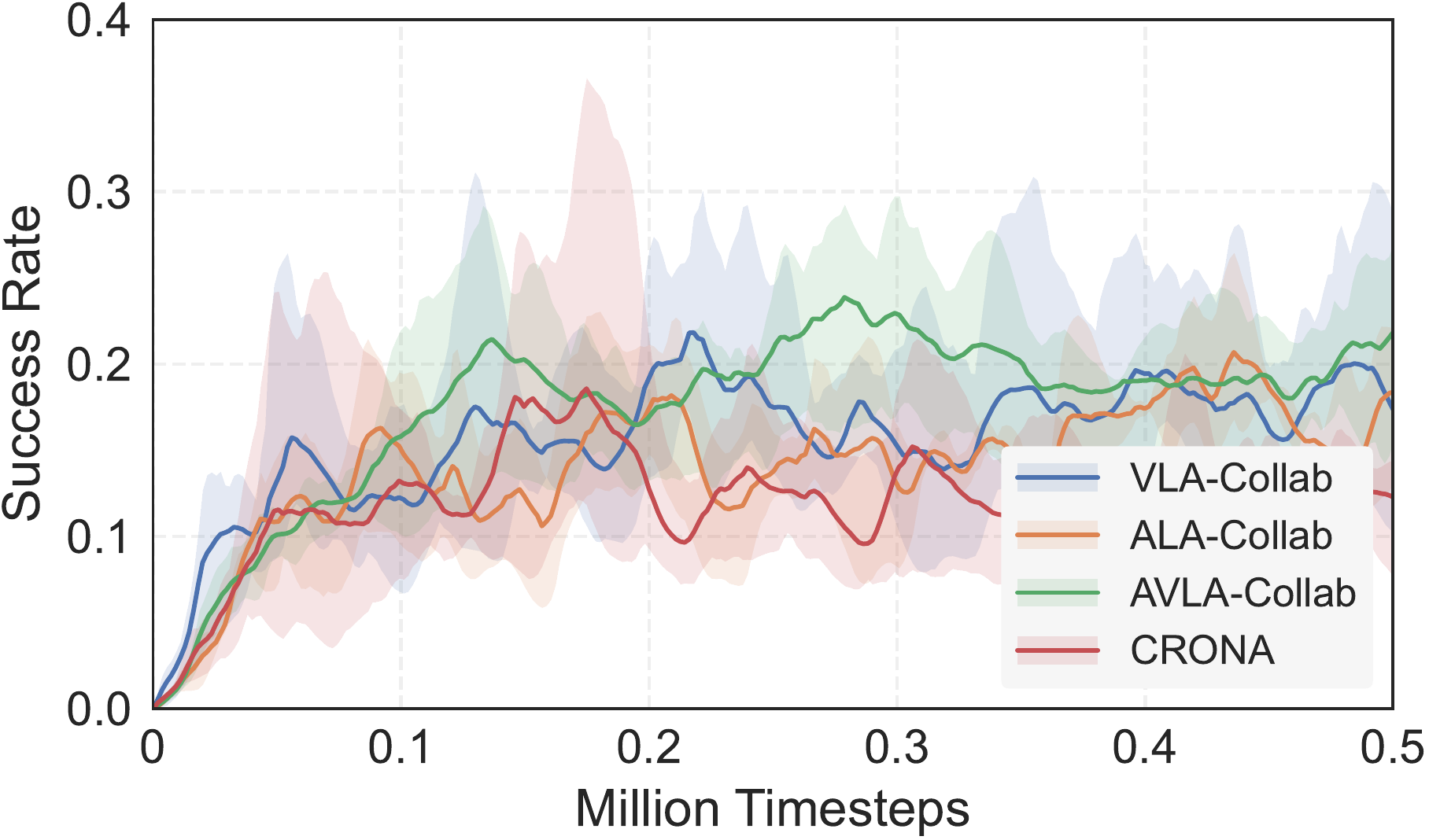}
        \caption{\texttt{Maze} $\mid$ \texttt{Drawer \& Table \& Chair}}
        \label{subfig:add-maze}
    \end{subfigure}
    \hfill
    \hspace{1cm}
    \caption{Additional evaluation of CRONA and collaborative navigation baselines across 5 \texttt{Matterport3D} scenes: (a)-(e) show the success rate of each scene. The x-axis indicates the environment steps. Curves are smoothed by an exponential moving average. Shadows denote 90\% bootstrapped CI. Results are averaged over 5 runs.}
    \label{fig:additional_results}
\end{figure}

\section{Insturction Design} \label{app:prompt}

We use three categories of prompt templates. For each data entry, the goal is specified in natural language. Audio-based agents sample prompts from the audio-specific templates, vision-based agents sample from the vision-specific templates, and agents with both modalities may also use the general templates. The target object name and its corresponding sound category are substituted into the selected template. The prompt templates used in our benchmark are listed below.

\begin{lstlisting}[numbers=none, escapeinside={(*}{*)}]

(*\textbf{General}*)
Please help me locate the ... and ...
Please help me find where the ... and ... are.
Your task is to find the ... and ... for me.
Show me the locations of the ... and ...
Please find the ... and ... in the environment.
Search for the ... and report where they are.
Navigate to the ... and the ...
Find both target objects: the ... and the ...

(*\textbf{Audio-Based Agent}*)
The environment contains sounds from ... and ... Please locate both sound-emitting objects.
I cannot find the ... and ... anymore. They sound like ... and ... Can you help me locate them?
Can you hear where the ... is?
Listen for the sound of ... and use it to find the ...
Find the objects that are making ... and ... sounds.
Use the audio cues to locate the ... and ...
Follow the sounds associated with ... and ... to find the target objects.
The ... produces a ... sound. Please locate it using the sound cue.

(*\textbf{Vision-Based Agent}*)
Based on the visual observation, find the ... and ...
Look for the ... and ... in the scene.
Use visual cues to locate the ... and ...
Search the environment for the visible ... and ...
Find the ... by observing its shape and appearance.
Watch for visual evidence of the ... and ...
Inspect the scene and locate the ... and ...
Navigate toward the visually observed ... and ...

\end{lstlisting}

\section{Reward Design} \label{app:reward-design}

All agents share a joint team reward. The reward design differs slightly between single-object and multi-object tasks.

\paragraph{Single-object tasks}
For single-object two-agent runs, each agent receives a per-step reward $r_i = r^{\mathrm{slack}} + r_i^{\mathrm{dist}} + r_i^{\mathrm{succ}}$,
where $r^{\mathrm{slack}}=-0.02$ is the per-agent time penalty, $r_i^{\mathrm{dist}}$ is the reduction in distance from agent $i$ to the target between consecutive steps, and $r_i^{\mathrm{succ}}=20$ if agent $i$ calls stop within the success distance of the target. The team reward returned to PPO is the sum of the two agent rewards.
Hence, the effective per-step slack penalty is $-0.04$ for two-agent episodes. No additional stop penalty is applied in the single-object multi-agent environment. An episode is successful if any agent calls stop near the target.

\paragraph{Multi-object tasks}
For multi-object runs, the team reward is defined as $r^{\mathrm{team}} =
r^{\mathrm{slack}}
+ \sum_i r_i^{\mathrm{dist}}
+ r^{\mathrm{stop}}
+ r^{\mathrm{goal}}$,
where $r^{\mathrm{slack}}=-0.02$ is a team-level time penalty, $\sum_i r_i^{\mathrm{dist}}$ is the sum of distance-progress rewards over agents, and $r^{\mathrm{stop}}=-0.2 \cdot n_{\mathrm{stop}}$ penalizes agents that call stop. When a new target is found, the goal reward is scaled by task progress, $r^{\mathrm{goal}} = \frac{N_{\mathrm{found}}}{N_{\mathrm{total}}} \cdot 20$,
where $N_{\mathrm{found}}$ is the number of targets found after the current discovery and $N_{\mathrm{total}}$ is the total number of targets in the episode. So in a two-target task, the first discovered target gives a reward of $10$, and the second discovered target gives a reward of $20$. If both agents call stop before all goals are found, the episode terminates after one such step, with no additional both-stop termination penalty.

\paragraph{Maze reward adjustment}
We adjusted the reward scale for \texttt{Maze} to encourage long-horizon exploration, as it's the largest and most complex scene with the most targets in our benchmark. Comparably, agents require more time to explore before receiving valuable visual or acoustic evidence, since acoustic cues are relatively weak and ambiguous when agents are far or occluded from the source. To reduce the cost of exploration and provide denser directional guidance toward target regions, we use a slack penalty of $-0.002$, a success reward scale of $3.0$, a distance reward scale of $2.0$, and a progressive distance reward scale of $1.5$.

\section{Compute Resources} \label{app:compute}

Experiments were run on a cluster and local workstations. The runtime of each training depends on the scene size, method, model size, and variable hardware. For a run trained to 500k environment steps, \texttt{Studio} took roughly $8$-$10$ hours on 5090, \texttt{Apartment} and \texttt{Ranch} took roughly $20$-$25$ hours on A100, \texttt{Corridor} and \texttt{Maze} took roughly $30$-$48$ hours on A100.

\begin{itemize}
    \item \textbf{Workstation:}
    \begin{itemize}
        \item GPU: 1$\times$ NVIDIA GeForce RTX 5090, 32 GB VRAM
        \item CPU: AMD Ryzen 9 9950X, 16 cores / 32 threads
        \item System memory: 123 GiB
    \end{itemize}
    
    \item \textbf{Cloud Cluster:}
    \begin{itemize}
        \item GPU: 2$\times$ NVIDIA A100-SXM4, 160 GB total VRAM
        \item CPU: AMD EPYC 7513, 32 cores
        \item System memory: 354 GB
    \end{itemize}

    \item \textbf{Software environment:}
    \begin{itemize}
        \item CUDA: 12.8
        \item PyTorch: 2.8.0
    \end{itemize}
\end{itemize} 

\section{Broader Impacts}

This work studies how agents with different sensory modalities contribute to collaborative navigation. We build a collaborative navigation benchmark with simulated vision and audio observations in realistic indoor environments. We identify several modality-dominance patterns and analyze when and why each pattern emerges. Our findings suggest that incorporating more modalities does not always lead to better performance; instead, the usefulness of each modality depends on the scene structure, target properties, and sensory reliability. We further propose a MARL framework for cross-modal collaborative navigation. This work opens the door to studying cross-modal collaboration in embodied multi-agent navigation with multi-agent reinforcement learning.

\end{document}